\documentclass{article}
\PassOptionsToPackage{table}{xcolor}
\usepackage{paper/iclr2027_conference,times}

\usepackage{amsmath,amsfonts,amssymb,amsthm,mathtools,bm}
\usepackage{algorithm}
\usepackage{algorithmic}
\usepackage{array}
\usepackage{textcomp}
\usepackage{verbatim}
\usepackage{graphicx}
\usepackage{microtype}
\usepackage{booktabs}
\usepackage{placeins}
\usepackage{wrapfig}
\usepackage{multirow}
\usepackage[table]{xcolor}
\usepackage{enumitem}
\usepackage{subcaption}
\usepackage{url}
\usepackage{hyperref}
\usepackage[capitalize,noabbrev]{cleveref}

\theoremstyle{plain}
\newtheorem{theorem}{Theorem}[section]

\theoremstyle{definition}

\theoremstyle{remark}

\title{STCFormer: Adaptive Spatio-Temporal Modeling with Dynamic Cluster Transformer for Station-based Weather Forecasting}

\author{
{\bfseries Rongwen Li$^{1}$, Haixin Xie$^{1}$, Mingyang Wang$^{1}$, Hongwu Liu$^{2}$,}\\
{\bfseries Kun Fang$^{3}$, Changjian Chen$^{1}$\thanks{Corresponding to changjianchen@hnu.edu.cn.}, Zhuo Tang$^{1}$, Kenli Li$^{1}$}\\[0.4em]
{\normalfont $^{1}$Hunan University}\\
{\normalfont $^{2}$China Meteorological Administration}\\
{\normalfont $^{3}$Hunan Provincial Meteorological Bureau}}
\iclrfinalcopy

\begin{document}
\raggedbottom
\maketitle
\pagestyle{plain}
\thispagestyle{plain}

\begin{abstract}
Station-based weather forecasting supports daily life and economic activity, yet accurate forecasts require modeling complex spatial dependencies among stations. Recent clustering-based selective modeling offers a promising alternative to dense inter-station interactions. However, a grouping shared across an observation window may obscure local changes in station relationships, while intra-cluster interactions alone may miss important global context. The theoretical advantages of selective interactions over dense connectivity also remain insufficiently understood. We therefore propose STCFormer, an adaptive spatio-temporal Transformer that dynamically groups stations according to their local evolution within each temporal patch. Its Cluster-Guided Attention Block combines fine-grained local attention within clusters and global attention over regional state summaries, allowing each station to access information beyond its own cluster. We further show that a derived Lipschitz upper bound for cluster-conditioned local attention is no larger than its fully connected counterpart, explaining a potential robustness benefit and motivating the design of InfoLoss. Experiments on three real-world weather datasets spanning eight temperature and wind forecasting tasks show that STCFormer achieves the lowest 24-hour mean squared error on all eight tasks and ranks first or second in 47 of 48 comparisons across metrics and forecasting horizons. Ablations and case studies further confirm the benefits of locally adaptive grouping and complementary local--global interactions. Our code can be obtained at https://github.com/hnu-vis/STCFormer.
\end{abstract}

\section{Introduction}
Station-based weather forecasting supports agriculture, transportation, and many other applications~\cite{lazo2009300,freebairn2002economic}. The widespread deployment of automatic weather stations has increased the demand for accurate station-level predictions, particularly for fine-grained urban forecasting~\cite{baklanov2018urban,rolnick2022tackling}. Shaped by local environments and atmospheric processes, station observations exhibit complex temporal variations and spatial relationships. Effectively capturing these spatio-temporal dependencies remains a key challenge in station-based weather forecasting.

Existing studies approach this challenge through Spatio-Temporal (ST) and Multivariate Time Series (MTS) forecasting. ST methods incorporate spatial structures, spatio-temporal priors, or physical processes~\cite{DBLP:conf/iclr/LiYS018,DBLP:journals/natmi/WuZLW23}. MTS methods typically treat stations as channels, modeling them independently or capturing their dependencies through channel mixing~\cite{DBLP:conf/iclr/NieNSK23,DBLP:conf/nips/ChenLFHFTLY24,DBLP:conf/kdd/QiuW0GH025}. Because geography and atmospheric processes induce heterogeneous dependencies, dense, fine-grained interactions between all station pairs are not always necessary. Recent clustering-based methods offer a promising strategy for selective dependency modeling by organizing related sequences into groups~\cite{DBLP:conf/nips/ChenLFHFTLY24,DBLP:conf/kdd/QiuW0GH025,DBLP:conf/icml/HuZLLLC0XP25}. Figure~\ref{fig:intro}(a) illustrates this grouping-based perspective.

\begin{figure}[t]
    \centering
    \includegraphics[width=\linewidth]{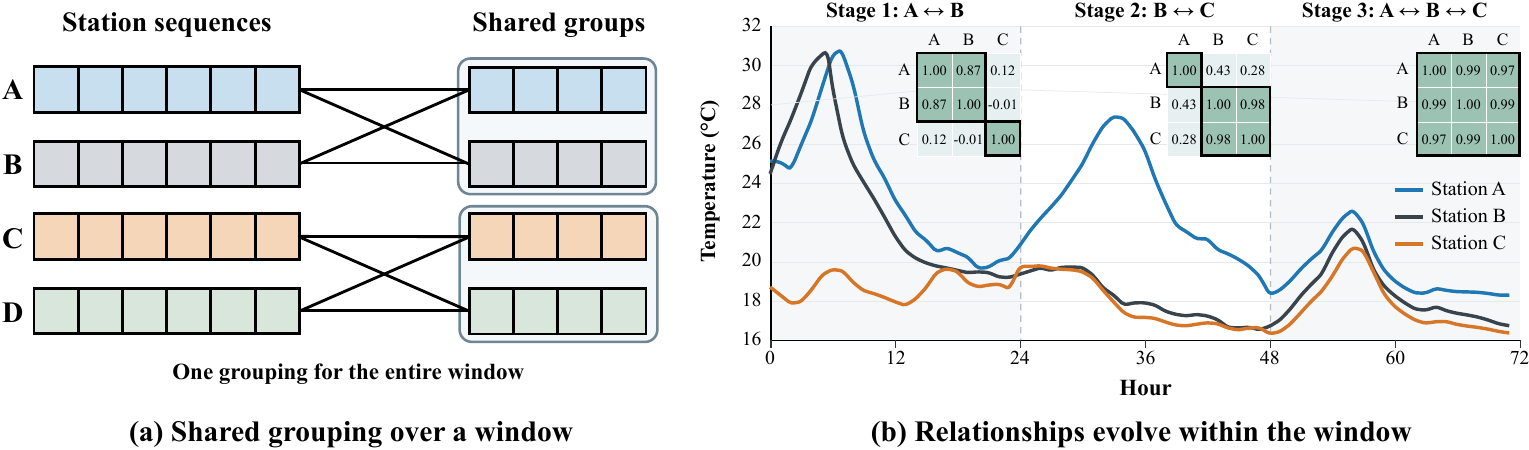}
    \caption{Motivation for locally adaptive station grouping. (a) A single grouping is shared across an observation window. (b) A real 72-hour case from the French temperature dataset shows the strongest relationship shifting from A--B to B--C before all three stations become highly correlated. Insets report Pearson correlations within each 24-hour segment.}
    \label{fig:intro}
\end{figure}

Despite the promising progress, two key challenges remain. \textbf{(1) Window-level grouping can obscure local relationship changes and restrict useful information exchange.} Groups estimated from complete sequences and shared across a window may overlook changing station dependencies as weather processes evolve. In the real French-temperature dataset case in Figure~\ref{fig:intro}(b), A and B are strongly correlated in the first 24-hour segment, B and C in the second, and all three stations in the third. A single grouping can therefore blend distinct local relationships. Moreover, when groups restrict interactions to their members, useful cross-region information may be lost: atmospheric circulation and propagating weather systems can connect stations with different local evolution patterns. Effective selective modeling should thus adapt to local relationship changes while retaining access to global context. \textbf{(2) Clustering-based selective interaction still lacks sufficient theoretical support.} Its effectiveness is primarily established empirically, leaving its advantages over dense interaction and potential robustness benefits insufficiently explained. A theoretical account is needed to complement empirical evidence and inform the design of selective interactions.

To address these challenges, we propose \textbf{STCFormer}, an adaptive spatio-temporal Transformer. Joint spatio-temporal patch embedding represents observations at a local temporal granularity. Within each temporal patch, efficient EM-style clustering uses evolution representations to dynamically organize stations into regional units. The Cluster-Guided Attention Block (CGAB) first performs hard-routed local attention within clusters, then aggregates the updated state features into regional content centers. Global station-to-region attention allows each station to access summaries from all regions, combining selective local interactions with global context. Theoretically, we show that the derived Lipschitz upper bound for fixed-partition local attention is no larger than the corresponding fully connected bound. We further connect this bound to station--cluster mutual information through the largest cluster size, motivating our mutual-information clustering regularizer, InfoLoss. Experiments on three real-world weather datasets demonstrate strong predictive performance and support the proposed design.

Our main contributions are as follows:
\begin{description}[leftmargin=0pt, labelsep=0.5em, itemsep=0pt, parsep=0pt, topsep=0pt]
    \item[(1)] We propose STCFormer, combining evolution-driven grouping at each temporal patch with cluster-guided local--global attention to capture changing station dependencies while preserving global context.
    \item[(2)] We establish a conditional sensitivity analysis of clustering-based attention, compare its derived Lipschitz upper bound with the dense counterpart, and connect the bound to mutual information to motivate InfoLoss.
    \item[(3)] Experiments on three real-world datasets covering eight forecasting tasks demonstrate competitive performance against 15 baselines, while ablations and case studies validate the benefits of locally adaptive grouping and complementary local--global interactions.
\end{description}

\section{STCFormer}
\label{sec:stcformer}

Let $\mathcal{X}_{H}\in\mathbb{R}^{T\times N\times C}$ denote $C$ meteorological variables observed at $N$ stations over the past $T$ time steps. Our goal is to predict the next $O$ steps, $\mathcal{Y}_{F}\in\mathbb{R}^{O\times N\times C}$. Figure~\ref{fig:stcformer_pipeline} gives an overview of STCFormer. The model combines locally adaptive station grouping with access to global spatial context. Spatio-Temporal Patch Embedding forms compact tokens from local observations and geographic coordinates. Each subsequent block constructs two temporal views: a state view provides contextualized features for prediction, while an evolution view guides grouping at each temporal patch. The Cluster-Guided Attention Block (CGAB) first models dependencies within the resulting groups, then enables each station to attend to content summarized from all regions. A direct prediction head maps the encoded features to future observations. We then analyze how grouping controls a local-attention sensitivity bound and use the resulting mutual-information connection to motivate InfoLoss for within-patch assignments.

\subsection{Spatio-Temporal Patch Embedding}
Patch embedding is an effective way to reduce sequence length in multivariate forecasting~\cite{DBLP:conf/iclr/NieNSK23,DBLP:conf/aaai/HuangSZCDZW24}, while conventional temporal patches encode each variable separately without explicitly incorporating neighboring stations. We construct joint spatio-temporal patches to retain local spatial context. Before patching, stations are grouped into latitude bands and ordered by longitude in alternating directions across adjacent bands. This snake-like ordering limits jumps at band boundaries and places geographically nearby stations at nearby indices~\cite{DBLP:journals/natmi/WuZLW23}. The temporal and station axes are padded to multiples of the patch lengths when necessary.

Let $L_t$ and $L_s$ be the temporal and spatial patch lengths, and let $T'$ and $N'$ be the resulting numbers of patches. The final feature or variable axis is implicit in all slices. The $(p,s)$-th patch is
\begin{align}
\label{eq:patch-seg}
    \mathcal{X}^{(\mathrm{pa})}_{p,s}
    =\mathcal{X}_{H}[(p-1)L_t:pL_t,(s-1)L_s:sL_s]
    \in\mathbb{R}^{L_t\times L_s\times C}.
\end{align}
A value projector encodes each flattened patch, while a lightweight geographic projector embeds the normalized latitude and longitude coordinates $\mathbf{c}_{s}$ within each spatial patch. Together with a learnable temporal-patch position embedding $\mathbf{e}^{(\mathrm{pos})}_{p}$, the input representation is
\begin{align}
\label{eq:patch-embed}
    \mathcal{X}_{p,s}
    =f_{\mathrm{val}}\!\left(\operatorname{vec}(\mathcal{X}^{(\mathrm{pa})}_{p,s})\right)+f_{\mathrm{geo}}(\mathbf{c}_{s})+\mathbf{e}^{(\mathrm{pos})}_{p}.\end{align}
The embedded tensor $\mathcal{X}\in\mathbb{R}^{T'\times N'\times D}$ reduces the token grid from $T\times N$ to $T'\times N'$ while retaining local spatio-temporal and geographic context for subsequent attention.

\begin{figure}
    \centering
    \captionsetup{skip=3pt}    \includegraphics[width=\linewidth]{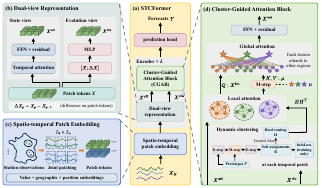}
    \caption{Overview of STCFormer. (a) The forecasting pipeline stacks CGAB encoders over patch embeddings and dual-view features. (b) Dual-view representation separates temporal state and local evolution. (c) Spatio-temporal patch embedding forms joint station--time tokens. (d) CGAB combines dynamic clustering, local attention within clusters, and global attention over regional summaries.}
    \label{fig:stcformer_pipeline}
\end{figure}

\subsection{Dual-View Temporal Representation}
We construct two temporal views to support predictive information exchange and station grouping according to local evolution. For spatial patch $s$, the temporal slice $\mathcal{X}_{:,s}\in\mathbb{R}^{T'\times D}$ contains its features across all temporal patches, where a colon denotes all entries along an axis. Specifically, the state view uses temporal self-attention to capture temporal dependencies across the observation window, followed by a feed-forward update:
\begin{align}
\label{eq:temporal-content}
    \mathcal{X}^{\mathrm{temp}}_{:,s}&=\operatorname{SelfAttention}(\mathcal{X}_{:,s}),\qquad
    \mathcal{X}^{\mathrm{att}}_{:,s}=\mathcal{X}^{\mathrm{temp}}_{:,s}+\operatorname{FFN}(\mathcal{X}^{\mathrm{temp}}_{:,s}).
\end{align}
The resulting state view $\mathcal{X}^{\mathrm{att}}$ retains temporal context, providing the feature content for subsequent local and global spatial attention. 
To make grouping responsive to changes within the observation window, we additionally construct an evolution view at each temporal patch. It incorporates local changes by concatenating each input token with its first-order temporal difference:
\begin{align}
\label{eq:temporal-clustering}
    \Delta\mathcal{X}_{p,s}&=\mathcal{X}_{p,s}-\mathcal{X}_{p-1,s},\qquad
    \mathcal{X}^{\mathrm{clu}}_{p,s}=\operatorname{MLP}_{\mathrm{clu}}\!\left([\mathcal{X}_{p,s};\Delta\mathcal{X}_{p,s}]\right).
\end{align}
Here the difference is defined for $p=2,\ldots,T'$, with $\Delta\mathcal{X}_{1,s}=\mathbf{0}$, and concatenation is along the feature axis. The patch-aligned views give CGAB temporally contextualized state features for attention and local evolution features for patch-specific grouping.

\subsection{Cluster-Guided Attention Block (CGAB)}
We propose the Cluster-Guided Attention Block to combine time-varying station interactions with global spatial context. CGAB comprises two modules: Dynamic Station Clustering and Cluster-Guided Local--Global Attention. For notational convenience, we take temporal patch $p$ as an example and write $\mathcal{X}^{\mathrm{clu}}_{p,:}=\mathbf{X}^{\mathrm{clu}}=[\mathbf{x}^{\mathrm{clu}}_1,\ldots,\mathbf{x}^{\mathrm{clu}}_{N'}]^{\top}$ and $\mathcal{X}^{\mathrm{att}}_{p,:}=\mathbf{X}^{\mathrm{att}}=[\mathbf{x}^{\mathrm{att}}_1,\ldots,\mathbf{x}^{\mathrm{att}}_{N'}]^{\top}$, both in $\mathbb{R}^{N'\times D}$. Each row corresponds to a spatial patch, referred to as a station token, and the time index is omitted below.

\subsubsection{Dynamic Station Clustering}
\label{sec:sc-section}
To capture relationship changes within an observation window, we group the $N'$ station tokens into $K$ clusters separately at each temporal patch using a differentiable EM-style procedure~\cite{NEURIPS2024_b48bb4ad,li2024one}. Each layer maintains learnable prototypes $\mathbf P=[\mathbf p_1,\ldots,\mathbf p_K]^\top\in\mathbb R^{K\times D}$. These shared prototypes provide a common reference across patches, while patch-specific updates allow a station to change groups as its local evolution changes. We perform one E--M update, followed by a final E-step to obtain the assignments used for spatial interaction.

\textbf{Initial E-step.} Using the prototypes as initial centroids, we compute provisional soft assignments from the evolution descriptors:
\begin{align}
\label{eq:E-step}
    \widetilde r_{s,k}
    =\operatorname{softmax}_{k}\!\left(
    \frac{\langle\mathbf x_s^{\mathrm{clu}},\mathbf p_k\rangle}{\tau}\right),
\end{align}
where $\tau>0$ is a learned temperature.
\textbf{M-step.} The assignments produce a data-dependent center for each cluster, which is blended with its learnable prototype through an EMA-style update:
\begin{align}
\label{eq:M-step}
    \bar{\mathbf c}_k
    &=\frac{\sum_{s=1}^{N'}\widetilde r_{s,k}\mathbf x_s^{\mathrm{clu}}}
    {\max(\sum_{s=1}^{N'}\widetilde r_{s,k},\epsilon)},\qquad
    \mathbf c_k=\eta\bar{\mathbf c}_k+(1-\eta)\mathbf p_k.
\end{align}
Here $\eta\in[0,1]$ balances the current patch's evolution information and the learned prototype, while $\epsilon>0$ prevents division by zero. This blend is computed within the current patch and does not carry centroids across time.
Finally, we repeat the E-step in Eq.~\eqref{eq:E-step} with the updated centers $\mathbf c_k$ replacing $\mathbf p_k$, obtaining the final soft assignments $\mathbf R=[r_{s,k}]\in[0,1]^{N'\times K}$ for InfoLoss, hard local routing, and soft regional aggregation. During training, Gumbel-Max samples the cluster index
\begin{align}
\label{eq:gumbel-max}
    \kappa_s=\underset{k}{\arg\max}\bigl(\log r_{s,k}+g_{s,k}\bigr),\qquad
    g_{s,k}\sim\operatorname{Gumbel}(0,1).
\end{align}
The hard assignment matrix $\mathbf H\in\{0,1\}^{N'\times K}$ has $H_{s,k}=1$ iff $k=\kappa_s$. During training, a straight-through Gumbel--Softmax estimator preserves hard routing in the forward pass while enabling backpropagation through the assignment probabilities~\cite{DBLP:journals/pami/HuijbenKPS23}. At inference, we omit the noise and set $\kappa_s=\arg\max_k r_{s,k}$. Recomputing this E--M--E procedure at each temporal patch lets groups track local changes within the observation window. Its hard assignments determine local attention neighborhoods, while the corresponding soft assignments carry graded membership into regional aggregation.

\subsubsection{Cluster-Guided Local--Global Attention}
Tokens are grouped according to similarity in their local evolution representations. We express local attention over the state features $\mathbf{X}^{\mathrm{att}}$ using $\mathbf{H}\mathbf{H}^{\top}$ as a mask; our experiments implement the same hard partition with efficient packed attention (Appendix~\ref{sec:b}):
\begin{align}
\label{eq:local}
    \mathbf{X}^{\mathrm{loc}}
    =\mathbf{X}^{\mathrm{att}}+
    \operatorname{Softmax}\!\left(
      \frac{(\mathbf{X}^{\mathrm{att}}\mathbf{W}^{\mathrm{loc}}_Q)
      (\mathbf{X}^{\mathrm{att}}\mathbf{W}^{\mathrm{loc}}_K)^{\top}}{\sqrt{D}}
      +\log(\mathbf{H}\mathbf{H}^{\top})\right)
      (\mathbf{X}^{\mathrm{att}}\mathbf{W}^{\mathrm{loc}}_V).
\end{align}
At each temporal patch, the mask restricts this attention step to station tokens assigned to the same region. The state view can therefore model fine-grained dependencies among stations with similar recent evolution.

Stations with different local evolution patterns may still provide useful predictive context. We therefore complement local attention with global station-to-region attention over all regional content centers. Specifically, the regional M-step aggregates the locally refined state features $\mathbf{X}^{\mathrm{loc}}$ with the soft assignments $\mathbf R$. Let $\mathbf D_R=\operatorname{Diag}(\max(\mathbf R^\top\mathbf1,\epsilon))$, where $\epsilon=10^{-6}$ and the maximum is elementwise. We then perform cross-attention with $\mathbf{X}^{\mathrm{loc}}$ as queries and $\boldsymbol{\mu}$ as keys and values:
\begin{align}
\label{eq:global}
    \boldsymbol{\mu}&=\mathbf{D}_{R}^{-1}\mathbf{R}^{\top}\mathbf{X}^{\mathrm{loc}},\qquad
    \mathbf{X}^{\mathrm{glo}}=\mathbf{X}^{\mathrm{loc}}+
    \operatorname{Softmax}\!\left(
      \frac{(\mathbf{X}^{\mathrm{loc}}\mathbf{W}^{\mathrm{glo}}_Q)
      (\boldsymbol{\mu}\mathbf{W}^{\mathrm{glo}}_K)^{\top}}{\sqrt{D}}
      \right)(\boldsymbol{\mu}\mathbf{W}^{\mathrm{glo}}_V).
\end{align}
The regional centers $\boldsymbol{\mu}$ summarize state content, whereas the updated clustering centers $\{\mathbf c_k\}_{k=1}^{K}$ represent evolution for grouping. Soft weights retain each token's degree of membership when forming these centers, even though local attention uses hard cluster boundaries. Each station queries all regional centers, accessing global context beyond its own group without all-to-all station attention. We then apply a feed-forward update for nonlinear transformation:
\begin{align}
    \mathbf{X}^{\mathrm{out}}
    =\mathbf{X}^{\mathrm{glo}}+\operatorname{FFN}(\mathbf{X}^{\mathrm{glo}}).\end{align}
Thus, local attention captures fine-grained dependencies within evolving clusters, while global attention shares information across clusters through compact regional summaries. This separation preserves selective local interactions without isolating stations from other regions. The following sensitivity analysis focuses on local attention; the equations above use single-head notation for clarity.

\paragraph{Direct prediction.} For spatial patch $s$, we flatten the encoded temporal features $\mathbf F_s\in\mathbb R^{T'\times D}$ and predict all $O L_s$ future values for one target variable: \begin{equation}\label{eq:prediction-head}\widehat{\mathbf Y}_s=\operatorname{Reshape}_{O\times L_s}\!\left(\operatorname{MLP}_{\mathrm{pred}}\!\left(\operatorname{LN}(\operatorname{vec}(\mathbf F_s))\right)\right).\end{equation} The MLP uses two affine layers with GELU and dropout between them. The patch outputs are assembled, unpadded, and reordered to recover forecasts in the original station order.

\subsection{Lipschitz Bounds and Mutual Information Regularization}
\label{sec:cluster-conditioned-theory}
\label{sec:infoloss}
In this section, we analyze the potential sensitivity advantage of cluster-conditioned local attention over dense attention. Following bounded-domain analyses of attention~\citep{castin2024smooth}, we compare the two operators under identical projection weights and feature-norm constraints. At one temporal patch, let $n=N'$, fix a hard partition $H$, and denote its largest cluster size by $m(H)$. Let $\mathcal A_H$ be the single-head residual local-attention operator in Eq.~\eqref{eq:local}, and $\mathcal A_{\mathrm{dense}}$ its fully connected counterpart. Features use the Frobenius norm on $\Omega_{R_X}=\{U:\max_s\|u_s\|_2\le R_X\}$, and weights use the spectral norm.

\begin{theorem}[Grouping-dependent sensitivity bound]
\label{thm:local-partition-lipschitz}
Let $v=\|W_V^{\mathrm{loc}}\|_2$ and
\[
\beta=\frac{R_X^2\|W_Q^{\mathrm{loc}}(W_K^{\mathrm{loc}})^\top\|_2}{\sqrt D}.
\]
For $q\ge1$, define the nondecreasing bound
\begin{equation}
\label{eq:attention-size-certificate}
B(q)=1+v\left[\beta+(1+\beta)\sqrt{\frac{q e^{2\beta}}{q-1+e^{2\beta}}}\right].
\end{equation}
Then the derived Lipschitz upper bounds satisfy
\begin{equation}
\label{eq:local-partition-bound}
B_{\mathrm{ours}}(H):=B(m(H))\le B(n)=:B_{\mathrm{dense}}.
\end{equation}
Moreover, $\operatorname{Lip}_F(\mathcal A_H)\le B_{\mathrm{ours}}(H)$ and $\operatorname{Lip}_F(\mathcal A_{\mathrm{dense}})\le B_{\mathrm{dense}}$.
\end{theorem}
Since $m(H)\le n$ and $B$ is nondecreasing, the derived local-attention bound is no larger than the dense bound and is strictly tighter when $m(H)<n$ and $v,\beta>0$. The maximum cluster size therefore controls the local-attention certificate. This observation motivates a loss that favors informative, non-degenerate assignments rather than allowing cluster collapse.

\textbf{Mutual information clustering loss.}
Let $R\in[0,1]^{n\times K}$ be row-stochastic, with $1\le K\le n$. Let $X$ be a station token sampled uniformly within the patch and $Y$ its cluster label, with $\Pr(Y=k\mid X=\mathbf x_s^{\mathrm{clu}})=R_{s,k}$. Using natural logarithms, we maximize the within-patch mutual information $I(X;Y)=H(Y)-H(Y\mid X)$~\citep{gomes2010rim}. Restoring the temporal-patch index $p$, write
\begin{equation}
\label{eq:info-marginal}
\pi_{p,k}=\frac1{N'}\sum_sR_{p,s,k},\qquad
H(Y_p\mid X_p)=\frac1{N'}\sum_sH(\mathbf R_{p,s,:}),
\end{equation}
where $H(\mathbf r)=-\sum_kr_k\log r_k$ and $H(Y_p)=H(\boldsymbol\pi_p)$. We minimize the negative mutual information
\begin{equation}
\label{eq:mi-loss}
\mathcal L_{\mathrm{info}}=-\sum_{\ell=1}^{L}\frac1{T'}\sum_{p=1}^{T'}I(X_p^{(\ell)};Y_p^{(\ell)})
=\sum_{\ell=1}^{L}\frac1{T'}\sum_{p=1}^{T'}\left[\frac1{N'}\sum_sH(\mathbf R_{p,s,:}^{(\ell)})-H(\boldsymbol\pi_p^{(\ell)})\right].
\end{equation}
Here $L$ is the number of encoder layers: InfoLoss is computed at every layer, averaged over its temporal patches, and summed over layers. Lower conditional entropy encourages confident assignments, while higher marginal entropy encourages balanced cluster use. We train with $\mathcal L=\mathcal L_{\mathrm{pred}}+\lambda\mathcal L_{\mathrm{info}}$, using MSE for $\mathcal L_{\mathrm{pred}}$ and a single regularization weight $\lambda$. The following theorem formalizes why this objective is relevant: increasing station--cluster mutual information tightens a probabilistic certificate on the largest sampled group, which in turn tightens the derived local-attention bound.

\begin{theorem}[Mutual-information control of sampled grouping]
\label{cor:mi-partition-bound}
Conditional on $R$, sample token labels independently according to its rows, as in training-time Gumbel-Max. For $\delta\in(0,1)$, define
\begin{equation}
\label{eq:mi-size-envelope}
M_\delta(I)=\min\!\left\{n,\;\frac nK+
\sqrt{\left(\frac{n^2}{2}+\frac n\delta\right)(\log K-I)}\right\}.
\end{equation}
Then, with conditional probability at least $1-\delta$,
\begin{equation}
\label{eq:mi-hard-size-bound}
m(H)\le M_\delta(I(X;Y)),\qquad
B_{\mathrm{ours}}(H)\le B\!\left(M_\delta(I(X;Y))\right).
\end{equation}
\end{theorem}
At fixed $n,K,v,\beta,\delta$, increasing mutual information tightens or preserves this probabilistic certificate. Together, the two results connect the training objective to the sensitivity analysis: InfoLoss promotes informative assignments, the mutual-information theorem controls the largest sampled group, and the grouping theorem maps that size to the local-attention bound. The failure probability $\delta$ is an analysis parameter, not a training hyperparameter. Appendix~\ref{app:cluster-conditioned-proofs} gives the proofs and a deterministic argmax counterpart; Appendix~\ref{sec:infoloss-properties} discusses further loss properties.

\section{Experiments}
\subsection{Experimental Setup}
\label{sec:experimental_setup}

\paragraph{Datasets and Baselines.}
We evaluate STCFormer on three real-world weather datasets---French, Hunan, and Global---covering eight temperature and wind forecasting tasks. All datasets use a 6:2:2 training/validation/test split and Z-score normalization based on training-set statistics. We compare with 15 baselines: six spatio-temporal methods and nine multivariate time series methods. Dataset descriptions and the complete baseline list are provided in Appendices~\ref{app:setup_datasets} and~\ref{app:setup_baselines}.

\paragraph{Implementation Details.}
All experiments use PyTorch and BasicTS~\cite{DBLP:journals/tkde/ShaoWXWYZYSJCCJC25} on an NVIDIA A100 GPU (40\,GB). We use a 48-hour look-back window to forecast the next 24 and 72 hours, with MSE as the training loss. Models are trained with Adam for up to 50 epochs, with an early-stopping patience of 5. We repeat each experiment three times and report mean and standard deviation. The main text presents 24-hour MSE; other metrics and 72-hour results appear in the appendix. Full training and model configurations are given in Appendix~\ref{app:setup_implementation}.

% Generated from Feishu main revision 323 and ablation revision 193.
% STCFormer denotes the STCFormerV4 results in the source tables.
% Corrformer / Hunan Temp. / 72h / MAE SD corrected to 0.245 with author confirmation.
\subsection{Main Results}
% Table~\ref{main_results_24} reports 24-hour MSE on eight dataset--variable tasks. Results for 24-hour MAE and SEDI and all 72-hour metrics are provided in Appendix~\ref{sec:additional_results}.

\begin{table}[!t]
\centering
\caption{24-hour forecasting results (MSE, $\downarrow$) with a 48-hour look-back window. Values are mean $\pm$ standard deviation over three runs. Bold and underlined means indicate the best and second-best reported values in each column; ties share the same marking.}
\label{main_results_24}
\begingroup
\fontsize{8}{9.5}\selectfont
\setlength{\tabcolsep}{0.8pt}
\renewcommand{\arraystretch}{1.10}
\begin{tabular*}{\textwidth}{@{\extracolsep{\fill}}l*{8}{c}@{}}
\toprule
\multirow{2}{*}{\textbf{Method}} & \multicolumn{3}{c}{\textbf{French}} & \multicolumn{3}{c}{\textbf{Hunan}} & \multicolumn{2}{c}{\textbf{Global}} \\
\cmidrule(lr){2-4}\cmidrule(lr){5-7}\cmidrule(lr){8-9}
 & Temp. & U-Wind & V-Wind & Temp. & U-Wind & V-Wind & Temp. & Wind \\
\midrule
\multicolumn{9}{l}{\textit{Spatio-Temporal Methods}} \\
DCRNN & 6.428{\fontsize{6}{7}\selectfont\textpm0.062} & 6.111{\fontsize{6}{7}\selectfont\textpm0.159} & 8.091{\fontsize{6}{7}\selectfont\textpm0.133} & \underline{4.921}{\fontsize{6}{7}\selectfont\textpm2.835} & 1.947{\fontsize{6}{7}\selectfont\textpm0.007} & 2.815{\fontsize{6}{7}\selectfont\textpm0.039} & 8.401{\fontsize{6}{7}\selectfont\textpm0.138} & 3.675{\fontsize{6}{7}\selectfont\textpm0.005} \\
Corrformer & 7.911{\fontsize{6}{7}\selectfont\textpm0.161} & 4.700{\fontsize{6}{7}\selectfont\textpm0.068} & 6.411{\fontsize{6}{7}\selectfont\textpm0.001} & 5.594{\fontsize{6}{7}\selectfont\textpm0.723} & 2.044{\fontsize{6}{7}\selectfont\textpm0.006} & 2.640{\fontsize{6}{7}\selectfont\textpm0.030} & 7.709{\fontsize{6}{7}\selectfont\textpm0.009} & 3.889{\fontsize{6}{7}\selectfont\textpm0.004} \\
EasyST & 7.823{\fontsize{6}{7}\selectfont\textpm0.697} & 4.735{\fontsize{6}{7}\selectfont\textpm0.092} & 5.884{\fontsize{6}{7}\selectfont\textpm0.047} & 6.673{\fontsize{6}{7}\selectfont\textpm1.390} & 2.026{\fontsize{6}{7}\selectfont\textpm0.034} & 2.680{\fontsize{6}{7}\selectfont\textpm0.039} & 7.333{\fontsize{6}{7}\selectfont\textpm0.028} & 3.715{\fontsize{6}{7}\selectfont\textpm0.007} \\
CDPNet & 7.592{\fontsize{6}{7}\selectfont\textpm0.136} & 4.478{\fontsize{6}{7}\selectfont\textpm0.049} & 4.998{\fontsize{6}{7}\selectfont\textpm0.065} & 5.432{\fontsize{6}{7}\selectfont\textpm0.006} & 1.979{\fontsize{6}{7}\selectfont\textpm0.000} & 2.590{\fontsize{6}{7}\selectfont\textpm0.047} & \underline{7.029}{\fontsize{6}{7}\selectfont\textpm0.018} & 3.638{\fontsize{6}{7}\selectfont\textpm0.001} \\
STELLA & 26.347{\fontsize{6}{7}\selectfont\textpm3.948} & 9.844{\fontsize{6}{7}\selectfont\textpm1.123} & 9.345{\fontsize{6}{7}\selectfont\textpm1.033} & 21.410{\fontsize{6}{7}\selectfont\textpm2.085} & 2.411{\fontsize{6}{7}\selectfont\textpm0.044} & 3.541{\fontsize{6}{7}\selectfont\textpm0.179} & 8.371{\fontsize{6}{7}\selectfont\textpm0.021} & 3.835{\fontsize{6}{7}\selectfont\textpm0.003} \\
S$^2$Transformer & 7.807{\fontsize{6}{7}\selectfont\textpm0.053} & 4.373{\fontsize{6}{7}\selectfont\textpm0.131} & 5.112{\fontsize{6}{7}\selectfont\textpm0.372} & 6.864{\fontsize{6}{7}\selectfont\textpm0.494} & 1.912{\fontsize{6}{7}\selectfont\textpm0.001} & 2.492{\fontsize{6}{7}\selectfont\textpm0.070} & 7.149{\fontsize{6}{7}\selectfont\textpm0.002} & 3.708{\fontsize{6}{7}\selectfont\textpm0.015} \\
\midrule
\multicolumn{9}{l}{\textit{Multivariate Time Series Methods}} \\
DLinear & 6.902{\fontsize{6}{7}\selectfont\textpm0.013} & 4.676{\fontsize{6}{7}\selectfont\textpm0.002} & 5.335{\fontsize{6}{7}\selectfont\textpm0.001} & 5.564{\fontsize{6}{7}\selectfont\textpm0.012} & 2.080{\fontsize{6}{7}\selectfont\textpm0.000} & 2.448{\fontsize{6}{7}\selectfont\textpm0.000} & 9.180{\fontsize{6}{7}\selectfont\textpm0.005} & 4.114{\fontsize{6}{7}\selectfont\textpm0.004} \\
PatchTST & 6.921{\fontsize{6}{7}\selectfont\textpm0.030} & 4.788{\fontsize{6}{7}\selectfont\textpm0.019} & 5.627{\fontsize{6}{7}\selectfont\textpm0.027} & 5.509{\fontsize{6}{7}\selectfont\textpm0.058} & 1.952{\fontsize{6}{7}\selectfont\textpm0.001} & 2.546{\fontsize{6}{7}\selectfont\textpm0.009} & 8.177{\fontsize{6}{7}\selectfont\textpm0.001} & 4.139{\fontsize{6}{7}\selectfont\textpm0.010} \\
Crossformer & 7.288{\fontsize{6}{7}\selectfont\textpm0.027} & 4.172{\fontsize{6}{7}\selectfont\textpm0.022} & 4.998{\fontsize{6}{7}\selectfont\textpm0.036} & 5.061{\fontsize{6}{7}\selectfont\textpm0.173} & 1.900{\fontsize{6}{7}\selectfont\textpm0.025} & 2.523{\fontsize{6}{7}\selectfont\textpm0.030} & 7.774{\fontsize{6}{7}\selectfont\textpm0.046} & \underline{3.609}{\fontsize{6}{7}\selectfont\textpm0.017} \\
TimeXer & \underline{6.388}{\fontsize{6}{7}\selectfont\textpm0.134} & 4.348{\fontsize{6}{7}\selectfont\textpm0.033} & 5.338{\fontsize{6}{7}\selectfont\textpm0.015} & 5.390{\fontsize{6}{7}\selectfont\textpm0.321} & 1.947{\fontsize{6}{7}\selectfont\textpm0.010} & 2.447{\fontsize{6}{7}\selectfont\textpm0.014} & 7.855{\fontsize{6}{7}\selectfont\textpm0.029} & 3.872{\fontsize{6}{7}\selectfont\textpm0.005} \\
CCM & 6.491{\fontsize{6}{7}\selectfont\textpm0.019} & 4.159{\fontsize{6}{7}\selectfont\textpm0.011} & 5.146{\fontsize{6}{7}\selectfont\textpm0.001} & 4.958{\fontsize{6}{7}\selectfont\textpm0.032} & 1.922{\fontsize{6}{7}\selectfont\textpm0.002} & \underline{2.434}{\fontsize{6}{7}\selectfont\textpm0.002} & 8.161{\fontsize{6}{7}\selectfont\textpm0.006} & 3.981{\fontsize{6}{7}\selectfont\textpm0.002} \\
DUET & 6.833{\fontsize{6}{7}\selectfont\textpm0.138} & 4.507{\fontsize{6}{7}\selectfont\textpm0.039} & \underline{4.997}{\fontsize{6}{7}\selectfont\textpm0.097} & 5.882{\fontsize{6}{7}\selectfont\textpm0.020} & \underline{1.895}{\fontsize{6}{7}\selectfont\textpm0.022} & 2.859{\fontsize{6}{7}\selectfont\textpm0.058} & 7.699{\fontsize{6}{7}\selectfont\textpm0.098} & 3.708{\fontsize{6}{7}\selectfont\textpm0.070} \\
TimeFilter & 6.760{\fontsize{6}{7}\selectfont\textpm0.078} & \underline{4.151}{\fontsize{6}{7}\selectfont\textpm0.022} & \underline{4.997}{\fontsize{6}{7}\selectfont\textpm0.018} & 5.041{\fontsize{6}{7}\selectfont\textpm0.065} & \underline{1.895}{\fontsize{6}{7}\selectfont\textpm0.023} & 2.435{\fontsize{6}{7}\selectfont\textpm0.045} & 7.848{\fontsize{6}{7}\selectfont\textpm0.024} & 3.650{\fontsize{6}{7}\selectfont\textpm0.072} \\
TQNet & 6.613{\fontsize{6}{7}\selectfont\textpm0.005} & 4.317{\fontsize{6}{7}\selectfont\textpm0.017} & 5.172{\fontsize{6}{7}\selectfont\textpm0.051} & 4.978{\fontsize{6}{7}\selectfont\textpm0.024} & 1.920{\fontsize{6}{7}\selectfont\textpm0.002} & 2.469{\fontsize{6}{7}\selectfont\textpm0.006} & 7.734{\fontsize{6}{7}\selectfont\textpm0.011} & 3.812{\fontsize{6}{7}\selectfont\textpm0.002} \\
xPatch & 7.010{\fontsize{6}{7}\selectfont\textpm0.008} & 4.517{\fontsize{6}{7}\selectfont\textpm0.004} & 5.334{\fontsize{6}{7}\selectfont\textpm0.011} & 5.227{\fontsize{6}{7}\selectfont\textpm0.021} & 1.924{\fontsize{6}{7}\selectfont\textpm0.001} & 2.516{\fontsize{6}{7}\selectfont\textpm0.002} & 8.380{\fontsize{6}{7}\selectfont\textpm0.007} & 3.930{\fontsize{6}{7}\selectfont\textpm0.003} \\
\midrule
\textbf{STCFormer} & \textbf{6.291}{\fontsize{6}{7}\selectfont\textpm0.090} & \textbf{4.045}{\fontsize{6}{7}\selectfont\textpm0.028} & \textbf{4.883}{\fontsize{6}{7}\selectfont\textpm0.013} & \textbf{4.743}{\fontsize{6}{7}\selectfont\textpm0.149} & \textbf{1.832}{\fontsize{6}{7}\selectfont\textpm0.001} & \textbf{2.402}{\fontsize{6}{7}\selectfont\textpm0.003} & \textbf{6.917}{\fontsize{6}{7}\selectfont\textpm0.002} & \textbf{3.584}{\fontsize{6}{7}\selectfont\textpm0.002} \\
\bottomrule
\end{tabular*}
\endgroup
\end{table}

Table~\ref{main_results_24} reports 24-hour MSE on eight dataset--variable tasks. The corresponding 24-hour MAE and SEDI, together with all 72-hour metrics, are reported in Appendix~\ref{sec:additional_results}. STCFormer achieves the lowest mean 24-hour MSE on all eight tasks, consistently outperforming both spatio-temporal and multivariate forecasting methods. Relative to the second-best mean in each task, it reduces MSE by 2.2\% on average. The additional results in Appendix~\ref{sec:additional_results} further demonstrate this consistency.

Across all 48 dataset--variable--horizon--metric combinations, STCFormer ranks first in 42 and second in five. It also outperforms the clustering-based methods DUET and CCM and the dynamic interaction-selection method TimeFilter on every task in the main table. We attribute this advantage to locally adaptive clustering, which tracks changing station relationships, together with global attention over regional summaries, which restores information exchange across clusters. Overall, the results support a stable predictive advantage across heterogeneous station networks.

\FloatBarrier
\subsection{Ablation Studies}
We evaluate six variants: one grouping shared across the observation window (\textbf{Window-shared}), full spatial attention instead of CGAB (\textbf{Full Attention}), fixed geographic grouping (\textbf{Static Clustering}), and removal of \textbf{InfoLoss}, \textbf{global attention}, or \textbf{local attention}. Table~\ref{tab:ablation_full} reports 24-hour MSE; the corresponding MAE and SEDI results appear in Appendix~\ref{sec:additional_ablation}.

\begin{table}[!htbp]
\centering
\caption{24-hour ablation results (MSE, $\downarrow$) with a 48-hour look-back window. Values are mean $\pm$ standard deviation over three runs. Bold means indicate the best reported values in each column; ties are both bold.}
\label{tab:ablation_full}
\begingroup
\fontsize{8}{9.5}\selectfont
\setlength{\tabcolsep}{0.8pt}
\renewcommand{\arraystretch}{1.10}
\begin{tabular*}{\textwidth}{@{\extracolsep{\fill}}l*{8}{c}@{}}
\toprule
\multirow{2}{*}{\textbf{Variant}} & \multicolumn{3}{c}{\textbf{French}} & \multicolumn{3}{c}{\textbf{Hunan}} & \multicolumn{2}{c}{\textbf{Global}} \\
\cmidrule(lr){2-4}\cmidrule(lr){5-7}\cmidrule(lr){8-9}
 & Temp. & U-Wind & V-Wind & Temp. & U-Wind & V-Wind & Temp. & Wind \\
\midrule
Window-shared & 6.516{\fontsize{6}{7}\selectfont\textpm0.093} & 4.181{\fontsize{6}{7}\selectfont\textpm0.041} & 4.996{\fontsize{6}{7}\selectfont\textpm0.081} & 4.891{\fontsize{6}{7}\selectfont\textpm0.078} & 1.913{\fontsize{6}{7}\selectfont\textpm0.024} & 2.428{\fontsize{6}{7}\selectfont\textpm0.032} & 6.954{\fontsize{6}{7}\selectfont\textpm0.027} & 3.758{\fontsize{6}{7}\selectfont\textpm0.012} \\
Full Attention & 6.334{\fontsize{6}{7}\selectfont\textpm0.059} & 4.197{\fontsize{6}{7}\selectfont\textpm0.118} & 5.044{\fontsize{6}{7}\selectfont\textpm0.183} & 4.930{\fontsize{6}{7}\selectfont\textpm0.118} & 1.932{\fontsize{6}{7}\selectfont\textpm0.004} & 2.518{\fontsize{6}{7}\selectfont\textpm0.012} & 6.976{\fontsize{6}{7}\selectfont\textpm0.019} & 3.787{\fontsize{6}{7}\selectfont\textpm0.004} \\
Static Clustering & 6.503{\fontsize{6}{7}\selectfont\textpm0.175} & 4.204{\fontsize{6}{7}\selectfont\textpm0.062} & 5.079{\fontsize{6}{7}\selectfont\textpm0.130} & 4.904{\fontsize{6}{7}\selectfont\textpm0.401} & 1.905{\fontsize{6}{7}\selectfont\textpm0.013} & 2.437{\fontsize{6}{7}\selectfont\textpm0.018} & 6.932{\fontsize{6}{7}\selectfont\textpm0.016} & 3.796{\fontsize{6}{7}\selectfont\textpm0.006} \\
w/o InfoLoss & 6.634{\fontsize{6}{7}\selectfont\textpm0.310} & 4.269{\fontsize{6}{7}\selectfont\textpm0.044} & 4.891{\fontsize{6}{7}\selectfont\textpm0.027} & 5.119{\fontsize{6}{7}\selectfont\textpm0.018} & 1.902{\fontsize{6}{7}\selectfont\textpm0.017} & 2.439{\fontsize{6}{7}\selectfont\textpm0.025} & 6.930{\fontsize{6}{7}\selectfont\textpm0.006} & 3.765{\fontsize{6}{7}\selectfont\textpm0.008} \\
w/o Global Attn. & 6.363{\fontsize{6}{7}\selectfont\textpm0.065} & 4.193{\fontsize{6}{7}\selectfont\textpm0.019} & 4.926{\fontsize{6}{7}\selectfont\textpm0.048} & 5.185{\fontsize{6}{7}\selectfont\textpm0.083} & 1.902{\fontsize{6}{7}\selectfont\textpm0.002} & 2.421{\fontsize{6}{7}\selectfont\textpm0.007} & 6.966{\fontsize{6}{7}\selectfont\textpm0.012} & 3.760{\fontsize{6}{7}\selectfont\textpm0.003} \\
w/o Local Attn. & 6.388{\fontsize{6}{7}\selectfont\textpm0.041} & 4.260{\fontsize{6}{7}\selectfont\textpm0.085} & 4.931{\fontsize{6}{7}\selectfont\textpm0.016} & 4.874{\fontsize{6}{7}\selectfont\textpm0.210} & 1.907{\fontsize{6}{7}\selectfont\textpm0.000} & 2.497{\fontsize{6}{7}\selectfont\textpm0.022} & 7.159{\fontsize{6}{7}\selectfont\textpm0.011} & 3.811{\fontsize{6}{7}\selectfont\textpm0.013} \\
\midrule
\textbf{STCFormer} & \textbf{6.291}{\fontsize{6}{7}\selectfont\textpm0.090} & \textbf{4.045}{\fontsize{6}{7}\selectfont\textpm0.028} & \textbf{4.883}{\fontsize{6}{7}\selectfont\textpm0.013} & \textbf{4.743}{\fontsize{6}{7}\selectfont\textpm0.149} & \textbf{1.832}{\fontsize{6}{7}\selectfont\textpm0.001} & \textbf{2.402}{\fontsize{6}{7}\selectfont\textpm0.003} & \textbf{6.917}{\fontsize{6}{7}\selectfont\textpm0.002} & \textbf{3.584}{\fontsize{6}{7}\selectfont\textpm0.002} \\
\bottomrule
\end{tabular*}
\endgroup
\end{table}

The full model achieves the lowest mean MSE on all eight tasks. \textbf{(1)} Sharing one grouping across the window increases MSE on all tasks, supporting grouping updates within the window. Averaged over the eight tasks, Window-shared raises MSE by 2.9\%, whereas replacing CGAB with full attention raises it by 3.6\%. The latter result indicates that selective, structured interaction is beneficial beyond merely reducing attention cost. \textbf{(2)} Fixed geographic grouping produces a 3.3\% average increase, showing that geographic proximity alone cannot capture the input-dependent relationships exposed by evolving weather conditions. Removing the within-patch InfoLoss regularizer also increases MSE, with task-dependent effects: Hunan temperature rises from 4.743 to 5.119, while Global temperature changes only from 6.917 to 6.930. This variation is consistent with InfoLoss acting most strongly when the task benefits from clearer assignment structure. \textbf{(3)} Removing either attention branch worsens all tasks. On Global wind, MSE rises from 3.584 to 3.760 without global attention and 3.811 without local attention, supporting their complementary roles. Together, these comparisons isolate the contributions of assignment adaptivity, regularized grouping, and cross-region information exchange rather than attributing the gain to a single component.

% Let the conclusion share the page with the ablation table.

% Current experiment plan: Feishu Muxewbrs3icRr8kiaiHcSTSnnFb.
\subsection{Model Analysis}
\label{sec:model_analysis}
\subsubsection{Hyperparameter Sensitivity}
\label{sec:hyperparameter_sensitivity}
We evaluate the InfoLoss weight $\lambda$, cluster ratio $\rho$, and embedding dimension $d_{\mathrm{model}}$ on the French, Hunan, and Global temperature tasks using 24-hour MSE. The cluster count is $K\approx\rho N'$, where $N'$ denotes the station tokens after spatial patching. Look-back and patch-size studies are provided in Appendix~\ref{app:additional_hyperparameters}.

\begingroup
\setlength{\textfloatsep}{8pt plus 2pt minus 2pt}
\setlength{\intextsep}{6pt plus 2pt minus 2pt}
\begin{figure}[!ht]
\centering
\captionsetup{skip=3pt}
\captionsetup[subfigure]{skip=1pt}
\begin{subfigure}[t]{0.325\linewidth}
\centering
\includegraphics[width=\linewidth]{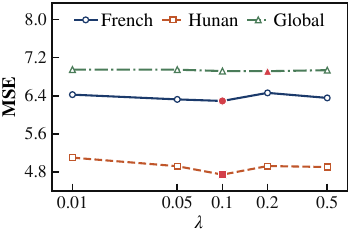}
\caption{MSE w.r.t. different $\lambda$.}
\end{subfigure}\hfill
\begin{subfigure}[t]{0.325\linewidth}
\centering
\includegraphics[width=\linewidth]{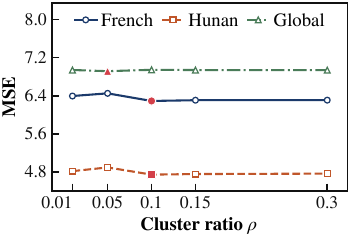}
\caption{MSE w.r.t. different $\rho$.}
\end{subfigure}\hfill
\begin{subfigure}[t]{0.325\linewidth}
\centering
\includegraphics[width=\linewidth]{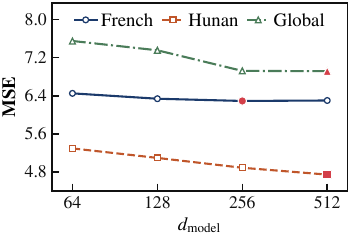}
\caption{MSE w.r.t. different $d_{\text{model}}$.}
\end{subfigure}
\caption{Hyperparameter sensitivity on three temperature tasks (24-hour MSE). Red markers indicate the lowest reported MSE within each dataset's sweep.}
\label{fig:hyperparameters_main}
\end{figure}
\endgroup

Figure~\ref{fig:hyperparameters_main} shows moderate sensitivity across all three settings. The best $\lambda$ is $0.1$ on French and Hunan and $0.2$ on Global, while larger weights bring no further gain. Intermediate cluster ratios are preferred: $\rho=0.1$ performs best on French and Hunan, and $0.05$ on Global, with nearby settings remaining competitive. French favors $d_{\mathrm{model}}=256$, whereas Hunan and Global favor 512; however, gains beyond 256 are small. Overall, STCFormer remains stable over practical ranges and does not require uniformly large regularization, cluster counts, or feature dimensions.
\clearpage

\subsubsection{Efficiency Analysis}
\label{sec:efficiency_analysis}
\begin{wrapfigure}[16]{l}{0.47\columnwidth}
  \vspace{-0.8\baselineskip}
  \centering
  \includegraphics[width=\linewidth]{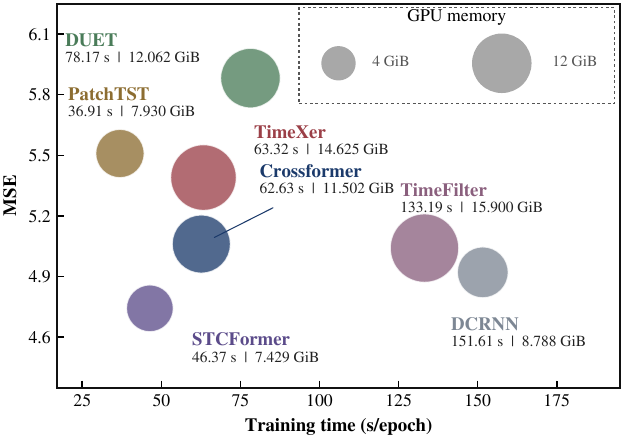}
  \captionsetup{hypcap=false}
  \caption{Efficiency on Hunan temperature (48-to-24-hour forecasting, batch size 8).}
  \label{fig:efficiency_hunan}
\end{wrapfigure}
Figure~\ref{fig:efficiency_hunan} compares accuracy, mean training time per epoch, and peak memory on Hunan. Among the methods shown, STCFormer achieves the lowest reported MSE (4.743) and memory usage (7.429\,GiB), taking $46.372\pm0.095$\,s/epoch. Compared with Crossformer, it reduces training time by 26.0\% and memory by 35.4\%; it is $2.87\times$ faster than TimeFilter. PatchTST trains faster (36.913\,s/epoch), but has higher MSE (5.509) and memory usage (7.930\,GiB). The bubble sizes make clear that STCFormer's accuracy gain does not come with a larger memory footprint. This result is consistent with CGAB's selective local interactions and compact regional summaries, which retain cross-cluster information without dense station-wise attention. STCFormer thus provides a favorable accuracy--resource trade-off on the larger station network. Appendix~\ref{app:additional_efficiency} reports the French comparison.
\FloatBarrier

\subsubsection{Dynamic Clustering Case Study}
\label{sec:clustering_case_study}
\begin{wrapfigure}[14]{r}{0.31\columnwidth}
  \vspace{-0.9\baselineskip}
  \centering
  \includegraphics[width=\linewidth]{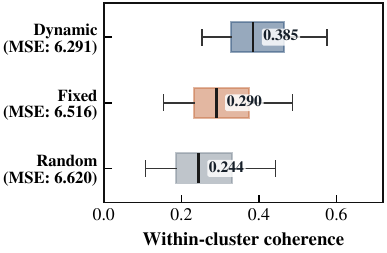}
  \captionsetup{hypcap=false}
  \caption{Dynamic, fixed, and random grouping on all French-temperature test windows ($48\!\rightarrow\!24$). Parentheses report MSE.}
  \label{fig:dynamic_clustering_case}
  \vspace{-0.9\baselineskip}
\end{wrapfigure}
To test whether learned groups track within-window relationship changes, we evaluate all test windows from the French-temperature 48-to-24-hour forecasting task. For each window, we measure within-cluster evolution coherence using Pearson correlations between hourly temperature increments for valid co-clustered station pairs and average the scores over local temporal blocks. Higher values indicate that stations grouped together exhibit more consistent short-term evolution.

Figure~\ref{fig:dynamic_clustering_case} compares dynamic grouping with a fixed window-shared partition and size-matched random grouping. Dynamic grouping achieves the highest median coherence (0.385), compared with Fixed (0.290) and Random (0.244), while also obtaining the lowest forecasting MSE (6.291 versus 6.516 and 6.620). Fixed remains better than Random, indicating persistent station dependence; the additional gain from Dynamic shows that updating assignments within a window captures transient couplings that a single shared partition obscures.
\FloatBarrier

\section{Conclusion}
We introduced STCFormer to model locally changing station dependencies while preserving global context for weather forecasting. Evolution-driven grouping adapts spatial interactions at each temporal patch, and CGAB combines fine-grained local attention within clusters and global attention over regional state summaries. Our theoretical analysis shows that the derived Lipschitz upper bound for fixed-partition local attention is no larger than its dense counterpart. Its connection to station--cluster mutual information motivates InfoLoss. Across three real-world datasets, STCFormer achieves the lowest 24-hour MSE on all eight tasks and ranks first or second in 47 of 48 metric--task--horizon comparisons. Ablations and case studies support the complementary roles of local grouping updates, global context, and clustering regularization. These findings support locally adaptive, selective spatial modeling for station-based weather forecasting.
\FloatBarrier

\subsection*{AI use statement}
Generative AI assisted the authors with preliminary exploration, literature review, and manuscript polishing and editing. The authors remain fully responsible for the accuracy of the text, the cited references, and the reported research.

\bibliographystyle{paper/iclr2027_conference}
\bibliography{sample-base,experiment_references}

@inproceedings{fu2025stella,
  title = {On the Integration of Spatial-Temporal Knowledge: A Lightweight Approach to Atmospheric Time Series Forecasting},
  author = {Fu, Yisong and Wang, Fei and Shao, Zezhi and Diao, Boyu and Wu, Lin and An, Zhulin and Yu, Chengqing and Li, Yujie and Xu, Yongjun},
  booktitle = {Advances in Neural Information Processing Systems},
  volume = {38},
  year = {2025},
  url = {https://papers.nips.cc/paper_files/paper/2025/file/7691484a7a35d5e2742279c1d926b778-Paper-Conference.pdf}
}

@article{chen2026s2transformer,
  title = {{S$^2$Transformer}: Scalable Structured Transformers for Global Station Weather Forecasting},
  author = {Chen, Hongyi and Li, Xiucheng and Chen, Xinyang and Cheng, Yun and Li, Jing and Chen, Kehai and Nie, Liqiang},
  journal = {Transactions on Machine Learning Research},
  year = {2026},
  url = {https://openreview.net/forum?id=AL2VnKno5n}
}

@inproceedings{lin2025tqnet,
  title = {Temporal Query Network for Efficient Multivariate Time Series Forecasting},
  author = {Lin, Shengsheng and Chen, Haojun and Wu, Haijie and Qiu, Chunyun and Lin, Weiwei},
  booktitle = {Proceedings of the 42nd International Conference on Machine Learning},
  pages = {37797--37814},
  year = {2025},
  volume = {267},
  series = {Proceedings of Machine Learning Research},
  publisher = {PMLR},
  url = {https://proceedings.mlr.press/v267/lin25e.html}
}

@inproceedings{stitsyuk2025xpatch,
  title = {{xPatch}: Dual-Stream Time Series Forecasting with Exponential Seasonal-Trend Decomposition},
  author = {Stitsyuk, Artyom and Choi, Jaesik},
  booktitle = {Proceedings of the AAAI Conference on Artificial Intelligence},
  volume = {39},
  pages = {20601--20609},
  year = {2025},
  url = {https://github.com/stitsyuk/xPatch}
}

@article{kim2026sigmaformer,
  title = {{SIGMAformer}: a spatiotemporal Gaussian mixture correlation transformer for global weather forecasting},
  author = {Kim, Do-Yeon and Suk, Heung-Il},
  journal = {npj Climate and Atmospheric Science},
  volume = {9},
  pages = {113},
  year = {2026},
  doi = {10.1038/s41612-026-01385-w},
  url = {https://www.nature.com/articles/s41612-026-01385-w}
}

@inproceedings{gomes2010rim,
  title = {Discriminative Clustering by Regularized Information Maximization},
  author = {Gomes, Ryan and Krause, Andreas and Perona, Pietro},
  booktitle = {Advances in Neural Information Processing Systems},
  volume = {23},
  year = {2010},
  url = {https://proceedings.neurips.cc/paper_files/paper/2010/file/42998cf32d552343bc8e460416382dca-Paper.pdf}
}

@inproceedings{castin2024smooth,
  title = {How Smooth Is Attention?},
  author = {Castin, Val{\'e}rie and Ablin, Pierre and Peyr{\'e}, Gabriel},
  booktitle = {Proceedings of the 41st International Conference on Machine Learning},
  pages = {5817--5840},
  year = {2024},
  volume = {235},
  series = {Proceedings of Machine Learning Research},
  publisher = {PMLR},
  url = {https://proceedings.mlr.press/v235/castin24a.html}
}

@inproceedings{liu2024timer,
  title = {{Timer}: Generative Pre-trained Transformers Are Large Time Series Models},
  author = {Liu, Yong and Zhang, Haoran and Li, Chenyu and Huang, Xiangdong and Wang, Jianmin and Long, Mingsheng},
  booktitle = {Proceedings of the 41st International Conference on Machine Learning},
  pages = {32369--32399},
  volume = {235},
  year = {2024},
  url = {https://proceedings.mlr.press/v235/liu24cb.html}
}

@article{ansari2024chronos,
  title = {{Chronos}: Learning the Language of Time Series},
  author = {Ansari, Abdul Fatir and Stella, Lorenzo and Turkmen, Caner and Zhang, Xiyuan and Mercado, Pedro and Shen, Huibin and Shchur, Oleksandr and Rangapuram, Syama Syndar and Pineda Arango, Sebastian and Kapoor, Shubham and Zschiegner, Jasper and Maddix, Danielle C. and Mahoney, Michael W. and Torkkola, Kari and Wilson, Andrew Gordon and Bohlke-Schneider, Michael and Wang, Yuyang},
  journal = {Transactions on Machine Learning Research},
  year = {2024},
  url = {https://openreview.net/forum?id=gerNCVqqtR}
}

@inproceedings{shi2025timemoe,
  title = {{Time-MoE}: Billion-Scale Time Series Foundation Models with Mixture of Experts},
  author = {Shi, Xiaoming and Wang, Shiyu and Nie, Yuqi and Li, Dianqi and Ye, Zhou and Wen, Qingsong and Jin, Ming},
  booktitle = {International Conference on Learning Representations},
  year = {2025},
  url = {https://proceedings.iclr.cc/paper_files/paper/2025/hash/558d48c1f08675daa636e09bfe94a89e-Abstract-Conference.html}
}

@inproceedings{DBLP:conf/iclr/NieNSK23,
  author       = {Yuqi Nie and
                  Nam H. Nguyen and
                  Phanwadee Sinthong and
                  Jayant Kalagnanam},
  title        = {A Time Series is Worth 64 Words: Long-term Forecasting with Transformers},
  booktitle    = {The Eleventh International Conference on Learning Representations,
                  {ICLR} 2023, Kigali, Rwanda, May 1-5, 2023},
  year         = {2023},
}

@article{DBLP:conf/aaai/HuangSZCDZW24,
  author       = {Qihe Huang and
                  Lei Shen and
                  Ruixin Zhang and
                  Jiahuan Cheng and
                  Shouhong Ding and
                  Zhengyang Zhou and
                  Yang Wang},
  title        = {HDMixer: Hierarchical Dependency with Extendable Patch for Multivariate
                  Time Series Forecasting},
  booktitle    = {Thirty-Eighth {AAAI} Conference on Artificial Intelligence, {AAAI}
                  2024, Thirty-Sixth Conference on Innovative Applications of Artificial
                  Intelligence, {IAAI} 2024, Fourteenth Symposium on Educational Advances
                  in Artificial Intelligence, {EAAI} 2024, February 20-27, 2024, Vancouver,
                  Canada},  pages        = {12608--12616},
  year         = {2024},
  journal      = {Proceedings of the {AAAI} Conference on Artificial Intelligence},
  volume       = {38},
  number       = {11},
  doi          = {10.1609/aaai.v38i11.29155},
}

@article{DBLP:journals/pami/HuijbenKPS23,
  author       = {Iris A. M. Huijben and
                  Wouter Kool and
                  Max B. Paulus and
                  Ruud J. G. van Sloun},
  title        = {A Review of the Gumbel-max Trick and its Extensions for Discrete Stochasticity
                  in Machine Learning},
  journal      = {{IEEE} Trans. Pattern Anal. Mach. Intell.},
  volume       = {45},
  number       = {2},
  pages        = {1353--1371},
  year         = {2023},
}

@inproceedings{DBLP:conf/aaai/ZhouZPZLXZ21,
  author       = {Haoyi Zhou and
                  Shanghang Zhang and
                  Jieqi Peng and
                  Shuai Zhang and
                  Jianxin Li and
                  Hui Xiong and
                  Wancai Zhang},
  title        = {Informer: Beyond Efficient Transformer for Long Sequence Time-Series
                  Forecasting},
  booktitle    = {Thirty-Fifth {AAAI} Conference on Artificial Intelligence, {AAAI}
                  2021, Thirty-Third Conference on Innovative Applications of Artificial
                  Intelligence, {IAAI} 2021, The Eleventh Symposium on Educational Advances
                  in Artificial Intelligence, {EAAI} 2021, Virtual Event, February 2-9,
                  2021},
  pages        = {11106--11115},
  year         = {2021},
}

@inproceedings{DBLP:conf/nips/WuXWL21,
  author       = {Haixu Wu and
                  Jiehui Xu and
                  Jianmin Wang and
                  Mingsheng Long},
  title        = {Autoformer: Decomposition Transformers with Auto-Correlation for Long-Term
                  Series Forecasting},
  booktitle    = {Advances in Neural Information Processing Systems 34: Annual Conference
                  on Neural Information Processing Systems 2021, NeurIPS 2021, December
                  6-14, 2021, virtual},
  pages        = {22419--22430},
  year         = {2021},
}

@inproceedings{DBLP:conf/iclr/ZhangY23,
  author       = {Yunhao Zhang and
                  Junchi Yan},
  title        = {Crossformer: Transformer Utilizing Cross-Dimension Dependency for
                  Multivariate Time Series Forecasting},
  booktitle    = {The Eleventh International Conference on Learning Representations,
                  {ICLR} 2023, Kigali, Rwanda, May 1-5, 2023},
  year         = {2023},
}

@inproceedings{DBLP:conf/iclr/LiuHZWWML24,
  author       = {Yong Liu and
                  Tengge Hu and
                  Haoran Zhang and
                  Haixu Wu and
                  Shiyu Wang and
                  Lintao Ma and
                  Mingsheng Long},
  title        = {iTransformer: Inverted Transformers Are Effective for Time Series
                  Forecasting},
  booktitle    = {The Twelfth International Conference on Learning Representations,
                  {ICLR} 2024, Vienna, Austria, May 7-11, 2024},
  year         = {2024},
}

@inproceedings{DBLP:conf/aaai/ZengCZ023,
  author       = {Ailing Zeng and
                  Muxi Chen and
                  Lei Zhang and
                  Qiang Xu},
  title        = {Are Transformers Effective for Time Series Forecasting?},
  booktitle    = {Thirty-Seventh {AAAI} Conference on Artificial Intelligence, {AAAI}
                  2023, Thirty-Fifth Conference on Innovative Applications of Artificial
                  Intelligence, {IAAI} 2023, Thirteenth Symposium on Educational Advances
                  in Artificial Intelligence, {EAAI} 2023, Washington, DC, USA, February
                  7-14, 2023},
  pages        = {11121--11128},
  year         = {2023},
}

@inproceedings{DBLP:conf/iclr/LiYS018,
  author       = {Yaguang Li and
                  Rose Yu and
                  Cyrus Shahabi and
                  Yan Liu},
  title        = {Diffusion Convolutional Recurrent Neural Network: Data-Driven Traffic
                  Forecasting},
  booktitle    = {6th International Conference on Learning Representations, {ICLR} 2018,
                  Vancouver, BC, Canada, April 30 - May 3, 2018, Conference Track Proceedings},
  year         = {2018},
}

@inproceedings{DBLP:conf/ijcai/YuYZ18,
  author       = {Bing Yu and
                  Haoteng Yin and
                  Zhanxing Zhu},
  title        = {Spatio-Temporal Graph Convolutional Networks: {A} Deep Learning Framework
                  for Traffic Forecasting},
  booktitle    = {Proceedings of the Twenty-Seventh International Joint Conference on
                  Artificial Intelligence, {IJCAI} 2018, July 13-19, 2018, Stockholm,
                  Sweden},
  pages        = {3634--3640},
  year         = {2018},
}

@inproceedings{DBLP:conf/ijcai/WuPLJZ19,
  author       = {Zonghan Wu and
                  Shirui Pan and
                  Guodong Long and
                  Jing Jiang and
                  Chengqi Zhang},
  title        = {Graph WaveNet for Deep Spatial-Temporal Graph Modeling},
  booktitle    = {Proceedings of the Twenty-Eighth International Joint Conference on
                  Artificial Intelligence, {IJCAI} 2019, Macao, China, August 10-16,
                  2019},
  pages        = {1907--1913},
  year         = {2019},
}

@inproceedings{DBLP:conf/kdd/WuPL0CZ20,
  author       = {Zonghan Wu and
                  Shirui Pan and
                  Guodong Long and
                  Jing Jiang and
                  Xiaojun Chang and
                  Chengqi Zhang},
  title        = {Connecting the Dots: Multivariate Time Series Forecasting with Graph
                  Neural Networks},
  booktitle    = {{KDD} '20: The 26th {ACM} {SIGKDD} Conference on Knowledge Discovery
                  and Data Mining, Virtual Event, CA, USA, August 23-27, 2020},
  pages        = {753--763},
  year         = {2020},
}

@article{lazo2009300,
  title={300 billion served: Sources, perceptions, uses, and values of weather forecasts},
  author={Lazo, Jeffrey K and Morss, Rebecca E and Demuth, Julie L},
  journal={Bulletin of the American Meteorological Society},
  volume={90},
  number={6},
  pages={785--798},
  year={2009},
}

@article{freebairn2002economic,
  title={Economic benefits of meteorological services},
  author={Freebairn, John W and Zillman, John W},
  journal={Meteorological Applications},
  volume={9},
  number={1},
  pages={33--44},
  year={2002},
  publisher={Cambridge University Press}
}

@article{baklanov2018urban,
  title={From urban meteorology, climate and environment research to integrated city services},
  author={Baklanov, A and Grimmond, C Sue B and Carlson, D and Terblanche, D and Tang, X and Bouchet, V and Lee, B and Langendijk, G and Kolli, R Kumar and Hovsepyan, A},
  journal={Urban Climate},
  volume={23},
  pages={330--341},
  year={2018},
  publisher={Elsevier}
}

@article{rolnick2022tackling,
  title={Tackling climate change with machine learning},
  author={Rolnick, David and Donti, Priya L and Kaack, Lynn H and Kochanski, Kelly and Lacoste, Alexandre and Sankaran, Kris and Ross, Andrew Slavin and Milojevic-Dupont, Nikola and Jaques, Natasha and Waldman-Brown, Anna and others},
  journal={ACM Computing Surveys (CSUR)},
  volume={55},
  number={2},
  pages={1--96},
  year={2022},
}

@article{DBLP:journals/natmi/WuZLW23,
  author       = {Haixu Wu and
                  Hang Zhou and
                  Mingsheng Long and
                  Jianmin Wang},
  title        = {Interpretable weather forecasting for worldwide stations with a unified
                  deep model},
  journal      = {Nat. Mac. Intell.},
  volume       = {5},
  number       = {6},
  pages        = {602--611},
  year         = {2023},
}

@article{hewage2020temporal,
  title={Temporal convolutional neural (TCN) network for an effective weather forecasting using time-series data from the local weather station},
  author={Hewage, Pradeep and Behera, Ardhendu and Trovati, Marcello and Pereira, Ella and Ghahremani, Morteza and Palmieri, Francesco and Liu, Yonghuai},
  journal={Soft Computing},
  volume={24},
  number={21},
  pages={16453--16482},
  year={2020},
}

@inproceedings{DBLP:conf/nips/SenYD19,
  author       = {Rajat Sen and
                  Hsiang{-}Fu Yu and
                  Inderjit S. Dhillon},
  editor       = {Hanna M. Wallach and
                  Hugo Larochelle and
                  Alina Beygelzimer and
                  Florence d'Alch{\'{e}}{-}Buc and
                  Emily B. Fox and
                  Roman Garnett},
  title        = {Think Globally, Act Locally: {A} Deep Neural Network Approach to High-Dimensional
                  Time Series Forecasting},
  booktitle    = {Advances in Neural Information Processing Systems 32: Annual Conference
                  on Neural Information Processing Systems 2019, NeurIPS 2019, December
                  8-14, 2019, Vancouver, BC, Canada},
  pages        = {4838--4847},
  year         = {2019},
}

@inproceedings{DBLP:conf/nips/WangWDQZLQWL24,
  author       = {Yuxuan Wang and
                  Haixu Wu and
                  Jiaxiang Dong and
                  Guo Qin and
                  Haoran Zhang and
                  Yong Liu and
                  Yunzhong Qiu and
                  Jianmin Wang and
                  Mingsheng Long},
  title        = {TimeXer: Empowering Transformers for Time Series Forecasting with
                  Exogenous Variables},
  booktitle    = {Advances in Neural Information Processing Systems 37: Annual Conference
                  on Neural Information Processing Systems 2024, NeurIPS 2024, Vancouver,
                  BC, Canada, December 10 - 15, 2024},
  year         = {2024},
}

@article{DBLP:journals/tmlr/DasKLMSY23,
  author       = {Abhimanyu Das and
                  Weihao Kong and
                  Andrew Leach and
                  Shaan Mathur and
                  Rajat Sen and
                  Rose Yu},
  title        = {Long-term Forecasting with TiDE: Time-series Dense Encoder},
  journal      = {Trans. Mach. Learn. Res.},
  volume       = {2023},
  year         = {2023},
}

@inproceedings{DBLP:conf/kdd/QiuW0GH025,
  author       = {Xiangfei Qiu and
                  Xingjian Wu and
                  Yan Lin and
                  Chenjuan Guo and
                  Jilin Hu and
                  Bin Yang},
  title        = {{DUET:} Dual Clustering Enhanced Multivariate Time Series Forecasting},
  booktitle    = {Proceedings of the 31st {ACM} {SIGKDD} Conference on Knowledge Discovery
                  and Data Mining, V.1, {KDD} 2025, Toronto, ON, Canada, August 3-7,
                  2025},
  pages        = {1185--1196},
  year         = {2025},
}

@inproceedings{DBLP:conf/nips/ChenLFHFTLY24,
  author       = {Jialin Chen and
                  Jan Eric Lenssen and
                  Aosong Feng and
                  Weihua Hu and
                  Matthias Fey and
                  Leandros Tassiulas and
                  Jure Leskovec and
                  Rex Ying},
  title        = {From Similarity to Superiority: Channel Clustering for Time Series
                  Forecasting},
  booktitle    = {Advances in Neural Information Processing Systems 37: Annual Conference
                  on Neural Information Processing Systems 2024, NeurIPS 2024, Vancouver,
                  BC, Canada, December 10 - 15, 2024},
  year         = {2024},
}

@inproceedings{DBLP:conf/icml/HuZLLLC0XP25,
  author       = {Yifan Hu and
                  Guibin Zhang and
                  Peiyuan Liu and
                  Disen Lan and
                  Naiqi Li and
                  Dawei Cheng and
                  Tao Dai and
                  Shu{-}Tao Xia and
                  Shirui Pan},
  title        = {TimeFilter: Patch-Specific Spatial-Temporal Graph Filtration for Time
                  Series Forecasting},
  booktitle    = {Forty-second International Conference on Machine Learning, {ICML}
                  2025, Vancouver, BC, Canada, July 13-19, 2025},
  year         = {2025},
}

@inproceedings{DBLP:conf/aaai/Jiang0YJCK0FS23,
  author       = {Renhe Jiang and
                  Zhaonan Wang and
                  Jiawei Yong and
                  Puneet Jeph and
                  Quanjun Chen and
                  Yasumasa Kobayashi and
                  Xuan Song and
                  Shintaro Fukushima and
                  Toyotaro Suzumura},
  title        = {Spatio-Temporal Meta-Graph Learning for Traffic Forecasting},
  booktitle    = {Thirty-Seventh {AAAI} Conference on Artificial Intelligence, {AAAI}
                  2023, Thirty-Fifth Conference on Innovative Applications of Artificial
                  Intelligence, {IAAI} 2023, Thirteenth Symposium on Educational Advances
                  in Artificial Intelligence, {EAAI} 2023, Washington, DC, USA, February
                  7-14, 2023},
  pages        = {8078--8086},
  year         = {2023},
}

@inproceedings{DBLP:conf/kdd/Liu0SH22,
  author       = {Dachuan Liu and
                  Jin Wang and
                  Shuo Shang and
                  Peng Han},
  title        = {{MSDR:} Multi-Step Dependency Relation Networks for Spatial Temporal
                  Forecasting},
  booktitle    = {{KDD} '22: The 28th {ACM} {SIGKDD} Conference on Knowledge Discovery
                  and Data Mining, Washington, DC, USA, August 14 - 18, 2022},
  pages        = {1042--1050},
  year         = {2022},
  doi          = {10.1145/3534678.3539397},
}

@inproceedings{DBLP:conf/ijcai/XuMDGW0025,
  author       = {Chujie Xu and
                  Yuqing Ma and
                  Haoyuan Deng and
                  Yajun Gao and
                  Yudie Wang and
                  Kai Lv and
                  Xianglong Liu},
  title        = {Continuous Diffusive Prediction Network for Multi-Station Weather
                  Prediction},
  booktitle    = {Proceedings of the Thirty-Fourth International Joint Conference on
                  Artificial Intelligence, {IJCAI} 2025, Montreal, Canada, August 16-22,
                  2025},
  pages        = {6714--6722},
  year         = {2025},
  doi          = {10.24963/IJCAI.2025/747},
}

@article{DBLP:journals/tkde/ShaoWXWYZYSJCCJC25,
  author       = {Zezhi Shao and
                  Fei Wang and
                  Yongjun Xu and
                  Wei Wei and
                  Chengqing Yu and
                  Zhao Zhang and
                  Di Yao and
                  Tao Sun and
                  Guangyin Jin and
                  Xin Cao and
                  Gao Cong and
                  Christian S. Jensen and
                  Xueqi Cheng},
  title        = {Exploring Progress in Multivariate Time Series Forecasting: Comprehensive
                  Benchmarking and Heterogeneity Analysis},
  journal      = {{IEEE} Trans. Knowl. Data Eng.},
  volume       = {37},
  number       = {1},
  pages        = {291--305},
  year         = {2025},
}

@inproceedings{DBLP:conf/cikm/Tang0X024,
  author       = {Jiabin Tang and
                  Wei Wei and
                  Lianghao Xia and
                  Chao Huang},
  editor       = {Edoardo Serra and
                  Francesca Spezzano},
  title        = {EasyST: {A} Simple Framework for Spatio-Temporal Prediction},
  booktitle    = {Proceedings of the 33rd {ACM} International Conference on Information
                  and Knowledge Management, {CIKM} 2024, Boise, ID, USA, October 21-25,
                  2024},
  pages        = {2220--2229},
  year         = {2024},
}

@inproceedings{DBLP:conf/cvpr/GuenT20a,
  author       = {Vincent Le Guen and
                  Nicolas Thome},
  title        = {Disentangling Physical Dynamics From Unknown Factors for Unsupervised
                  Video Prediction},
  booktitle    = {2020 {IEEE/CVF} Conference on Computer Vision and Pattern Recognition,
                  {CVPR} 2020, Seattle, WA, USA, June 13-19, 2020},
  pages        = {11471--11481},
  publisher    = {Computer Vision Foundation / {IEEE}},
  year         = {2020},
}

@article{DBLP:journals/pami/WangWZGWYL23,
  author       = {Yunbo Wang and
                  Haixu Wu and
                  Jianjin Zhang and
                  Zhifeng Gao and
                  Jianmin Wang and
                  Philip S. Yu and
                  Mingsheng Long},
  title        = {PredRNN: {A} Recurrent Neural Network for Spatiotemporal Predictive
                  Learning},
  journal      = {{IEEE} Trans. Pattern Anal. Mach. Intell.},
  volume       = {45},
  number       = {2},
  pages        = {2208--2225},
  year         = {2023},
}

@inproceedings{DBLP:conf/cvpr/Gao0WL22,
  author       = {Zhangyang Gao and
                  Cheng Tan and
                  Lirong Wu and
                  Stan Z. Li},
  title        = {SimVP: Simpler yet Better Video Prediction},
  booktitle    = {{IEEE/CVF} Conference on Computer Vision and Pattern Recognition,
                  {CVPR} 2022, New Orleans, LA, USA, June 18-24, 2022},
  pages        = {3160--3170},
  publisher    = {{IEEE}},
  year         = {2022},
}

@inproceedings{NEURIPS2024_b48bb4ad,
 author = {Zhou, Peng and Li, Rongwen and Du, Liang},
 booktitle = {Advances in Neural Information Processing Systems},
 pages = {99686--99714},
 title = {Fair Kernel K-Means: from Single Kernel to Multiple Kernel},
 year = {2024}
}

@inproceedings{li2024one,
  title={One-stage fair multi-view spectral clustering},
  author={Li, Rongwen and Hu, Haiyang and Du, Liang and Chen, Jiarong and Jiang, Bingbing and Zhou, Peng},
  booktitle={Proceedings of the 32nd ACM International Conference on Multimedia},
  pages={1407--1416},
  year={2024}
}

\clearpage
\appendix
\makeatletter\setlength{\@fptop}{0pt}\makeatother
\section{Related Work}

\textbf{Spatio-Temporal Forecasting:} ST methods combine temporal modeling with spatial structures, learned relationships, or domain knowledge. Graph-based approaches represent locations as nodes and dependencies as edges~\cite{DBLP:conf/aaai/Jiang0YJCK0FS23,DBLP:conf/kdd/Liu0SH22}. DCRNN~\cite{DBLP:conf/iclr/LiYS018} uses diffusion convolution and recurrent modeling, while STGCN~\cite{DBLP:conf/ijcai/YuYZ18} combines graph convolutions with gated temporal convolutions. Graph WaveNet~\cite{DBLP:conf/ijcai/WuPLJZ19} and MTGNN~\cite{DBLP:conf/kdd/WuPL0CZ20} learn adjacency structures from data, although a graph learned during training need not adapt within each observation window. Weather models further incorporate domain structure: Corrformer~\cite{DBLP:journals/natmi/WuZLW23} organizes correlations hierarchically, and STELLA~\cite{fu2025stella} integrates spatial-temporal knowledge into a lightweight model. SIGMAformer~\cite{kim2026sigmaformer} combines Gaussian-mixture pattern extraction with dynamic spatio-temporal correlation modeling for global station weather forecasting. Our focus is on grouping stations from local evolution at each temporal patch while retaining information from all groups.

\textbf{Multivariate Time Series Forecasting:} MTS methods treat station sequences as channels. Channel-dependent methods model series jointly, whereas channel-independent methods process them separately with shared parameters. Following recurrent and convolutional approaches~\cite{hewage2020temporal,DBLP:conf/nips/SenYD19}, Informer~\cite{DBLP:conf/aaai/ZhouZPZLXZ21} and Autoformer~\cite{DBLP:conf/nips/WuXWL21} improve temporal modeling through sparse attention and decomposition. iTransformer~\cite{DBLP:conf/iclr/LiuHZWWML24} attends over variate tokens to capture cross-channel dependencies. Channel-independent models avoid unnecessary mixing~\cite{DBLP:journals/tmlr/DasKLMSY23,DBLP:conf/iclr/NieNSK23}: DLinear~\cite{DBLP:conf/aaai/ZengCZ023} uses decomposition and linear mappings, while PatchTST~\cite{DBLP:conf/iclr/NieNSK23} applies a shared Transformer to temporal patches. Independent processing, however, does not explicitly exploit other stations' predictive information.

Clustering balances shared modeling and channel specialization. CCM~\cite{DBLP:conf/nips/ChenLFHFTLY24} groups similar channels and learns cluster-specific shared parameters. DUET~\cite{DBLP:conf/kdd/QiuW0GH025} combines temporal distribution clustering with frequency-based channel soft clustering. These mechanisms differ from fully isolated hard partitions. TimeFilter~\cite{DBLP:conf/icml/HuZLLLC0XP25} further adapts dependencies through patch-specific spatio-temporal graph filtration. Across these designs, grouping granularity and information coverage matter: window-shared assignments may obscure local relationship changes, while exclusively intra-group interactions can omit useful global context. STCFormer combines locally updated groups with attention over all regional content centers.

\section{Efficient Implementation of Intra-Cluster Local Attention}
\label{sec:b}
Equation~\eqref{eq:local} describes the hard-partition attention pattern. Our experiments evaluate it with packed local attention, avoiding an explicit $N'\times N'$ mask. Assignments are recomputed at every temporal patch and encoder layer.

\paragraph{Hard routing with a differentiable membership bias.}
Training obtains straight-through assignments $\mathbf H_{\mathrm{ST}}$ using \texttt{F.gumbel\_softmax(log(R), hard=True)}. Let $\kappa_s$ be the selected cluster index and $\mathcal I_k=\{s:\kappa_s=k\}$. The indices and their sorting are nondifferentiable. Crucially, we retain the selected membership $a_s=(H_{\mathrm{ST}})_{s,\kappa_s}$ as a differentiable tensor. For group $k$, gather the state features $\mathbf X_k=\mathbf X^{\mathrm{att}}_{\mathcal I_k,:}$ and the membership vector $\mathbf a_k=(a_s)_{s\in\mathcal I_k}$. Packed attention computes
\begin{equation}
\label{eq:app-packed-local}
\mathbf Y_k=\mathbf X_k+\operatorname{Softmax}\!\left(
\frac{(\mathbf X_k\mathbf W_Q^{\mathrm{loc}})(\mathbf X_k\mathbf W_K^{\mathrm{loc}})^\top}{\sqrt D}
+\mathbf1_{n_k}(\log\mathbf a_k)^\top\right)
(\mathbf X_k\mathbf W_V^{\mathrm{loc}}).
\end{equation}
The logarithmic membership is a key bias, broadcast across queries. Selected memberships equal one in the hard forward pass, so this bias is zero and each token attends to exactly its hard cluster. In the backward pass, the membership bias propagates gradients through the continuous Gumbel--Softmax relaxation to $\mathbf R$, despite nondifferentiable group indexing. At inference, routing uses deterministic $\arg\max_k R_{s,k}$.

\section{Computational Complexity Analysis}
\label{sec:c}
We analyze the current pipeline: patch embedding, dual-view temporal representation, dynamic clustering, intra-cluster attention, station-to-region attention, and a direct prediction head. Let $n=N'=\lceil N/L_s\rceil$ and $m=T'=\lceil T/L_t\rceil$, where $N,T$ are the original station count and history length. Let $D$ be the hidden width and $K\le n$ the number of clusters. Clustering uses one E--M update followed by a final E-step. Costs below are per sample and per block, with a fixed number of heads and feed-forward widths proportional to $D$. For a batch, multiply operation counts by the batch size.

\subsection{Module-wise computation}
\paragraph{Patch embedding and dual views.}
Flattening and linearly projecting the value patches costs $\mathcal O(mnL_tL_sCD)$, where $C$ is the number of observed variables; geographic projection and positional addition contribute separate lightweight embedding operations. The state view applies temporal attention to $m$ tokens for each of the $n$ spatial patches, costing $\mathcal O(nm^2D)$. Temporal differences cost $\mathcal O(mnD)$, and the evolution MLP, attention projections, and feed-forward updates cost $\mathcal O(mnD^2)$ for fixed expansion ratios. The evolution view introduces no additional all-pairs temporal attention.

\paragraph{Dynamic station clustering.}
At each temporal patch, the initial E-step computes $\mathbf X^{\mathrm{clu}}\mathbf P^\top$, the M-step aggregates $\widetilde{\mathbf R}^{\top}\mathbf X^{\mathrm{clu}}$ and blends the result with $\mathbf P$, and the final E-step computes $\mathbf X^{\mathrm{clu}}\mathbf C^\top$. Each matrix product costs $\mathcal O(nKD)$; prototype mixing and normalization are lower-order operations. Thus, the fixed E--M--E sequence over the window costs
\begin{equation}
\mathcal C_{\mathrm{clu}}=\mathcal O(mnKD).
\end{equation}
Sampling or taking the row-wise argmax of the final assignments adds $\mathcal O(mnK)$. No station-to-station distance matrix or reuse of the previous patch's partition is required.

\paragraph{Intra-cluster attention.}
Let $n_{p,k}=|\mathcal I_{p,k}|$ and $S_p=\sum_{k=1}^K n_{p,k}^2$. The packed evaluation in Appendix~\ref{sec:b} has attention-product cost
\begin{equation}
\mathcal C_{\mathrm{loc}}=\mathcal O\!\left(D\sum_{p=1}^{m}S_p\right),
\qquad \frac{n^2}{K}\le S_p\le n^2.
\end{equation}
The lower bound follows from Cauchy--Schwarz and is attained by equal group sizes when divisibility permits. Approximately balanced groups give $S_p\approx n^2/K$; collapse into one group gives $S_p=n^2$. InfoLoss encourages balanced soft masses but does not enforce equal hard group sizes.

\paragraph{Station-to-region global attention.}
Soft-assignment-weighted aggregation $\mathbf R^\top\mathbf X^{\mathrm{loc}}$ costs $\mathcal O(mnKD)$; computing and normalizing the regional soft masses costs $\mathcal O(mnK+mKD)$. Cross-attention between $n$ station queries and $K$ regional content centers costs $\mathcal O(mnKD)$, with an additional $\mathcal O(m(n+K)D^2)$ for projections. These are station-to-region interactions, rather than a $K\times K$ centroid self-attention stage. The centroids used in the E--M updates and the content centers used as attention keys and values are computed separately.

\paragraph{InfoLoss during training.}
At each temporal patch, computing $\boldsymbol\pi_p=n^{-1}\mathbf R_p^\top\mathbf1$ and the sum of row entropies takes $\mathcal O(nK)$ operations. The marginal entropy adds $\mathcal O(K)$, giving
\begin{equation}
\label{eq:app-infoloss-cost}
\mathcal C_{\mathrm{info}}=\mathcal O(mnK).
\end{equation}
No $K\times K$ joint matrix or adjacent-patch pairing is required.
The loss averages over temporal patches within each layer and sums across layers; these reductions do not change the asymptotic cost. This term is absent at inference. Backpropagation through the differentiable clustering updates has the same asymptotic arithmetic order as their forward operations for the fixed E--M--E sequence; retained intermediate activations and the chosen assignment surrogate affect training memory.

\paragraph{Direct prediction head.}
For one target variable, the head in Eq.~\eqref{eq:prediction-head} flattens the $m\times D$ encoded features of each spatial patch. Layer normalization costs $\mathcal O(nmD)$, and the two projections cost $\mathcal O(nmDD_h+nD_h O L_s)$, where $D_h$ is the head's hidden width and $O$ the forecasting horizon. Reshaping, removing padding, and restoring station order add no learned pairwise interaction.

\subsection{Overall cost and memory}
Combining the representation and interaction modules gives
\begin{equation}
\label{eq:complexity_bound}
\mathcal C_{\mathrm{block}}=\mathcal O\!\left(
mnD^2+nm^2D+mnKD+D\sum_{p=1}^{m}S_p\right).
\end{equation}
With approximately balanced groups, the last term becomes $mn^2D/K$; in the worst case it becomes $mn^2D$. Clustering and global attention remain in both expressions and cannot be discarded merely because $K<n$. For $B$ blocks with unchanged dimensions, the backbone cost is $B\mathcal C_{\mathrm{block}}$. If dimensions differ across blocks, their costs must instead be summed explicitly. The full model additionally incurs patch embedding and direct-head costs; training adds InfoLoss wherever it is applied. No layer-wise temporal shrinking or autoregressive decoder is assumed.

Ignoring widths and rounding, the principal terms are
\begin{equation}
\label{eq:app-patch-costs}
\underbrace{\frac{NT^2}{L_sL_t^2}}_{\text{temporal attention}},\qquad
\underbrace{\frac{TN^2}{L_tL_s^2K}}_{\text{balanced local attention}},\qquad
\underbrace{\frac{TNK}{L_tL_s}}_{\text{clustering and global attention}}.
\end{equation}
Increasing $K$ reduces balanced local-attention work but increases clustering, global-attention, and InfoLoss costs. In particular, setting $K$ proportional to $n$ does not make the full block linear in $n$: clustering and global attention remain quadratic, while within-patch InfoLoss costs $\mathcal O(mnK)$ and is quadratic in $n$ at a fixed cluster ratio. Patching reduces the token grid, but these trade-offs and hardware utilization determine practical efficiency.

For conventional materialized attention, score storage is $\mathcal O(nm^2+\sum_pS_p+mnK)$, while token features and assignments require $\mathcal O(mnD+mnK)$. The two E-steps and single M-step retain only a fixed number of clustering intermediates during training. Fused attention kernels can avoid storing the full score matrices without changing the arithmetic orders above. InfoLoss reduces the existing soft assignments into $m$ cluster-marginal vectors and row-entropy sums. These summaries require $\mathcal O(mK+mn)$ storage if materialized together; a conventional elementwise entropy implementation may additionally retain $\mathcal O(mnK)$ intermediates for backward computation. Sequential or fused reductions can reduce forward workspace, with backward storage depending on checkpointing. These are activation and workspace costs, excluding model parameters and optimizer states.

\subsection{Comparison with current baselines}
Table~\ref{com_comp} compares representative dependency-modeling operators from the baselines in our experiments. Hidden widths, head counts, and depth are fixed; each expression covers the operator named in the table, not end-to-end training or forecasting. Embedding, prediction heads, and auxiliary objectives are excluded. This scope makes explicit the distinction between dense station attention, temporal patch attention, compressed global interaction, and dynamically grouped interaction.

\begin{table}[t]
\centering
\small
\setlength{\tabcolsep}{4pt}
\renewcommand{\arraystretch}{1.18}
\caption{Computational complexity of representative dependency-modeling operators among the current baselines. $N,T$ denote stations and history steps. $P_t$ is each baseline's actual temporal-token count (accounting for its patch length and stride), $r$ is Crossformer's router count, $K_{\mathrm{CCM}}$ is CCM's cluster count, and $q,|\mathcal E|$ are DCRNN's diffusion order and graph-edge count. For STCFormer, $n=\lceil N/L_s\rceil$, $m=\lceil T/L_t\rceil$, with spatial and temporal patch sizes $L_s,L_t$. Widths are fixed; the stated scope excludes embeddings, output heads, and training losses.}
\label{com_comp}
\begin{tabular}{@{}p{0.23\linewidth}p{0.35\linewidth}p{0.37\linewidth}@{}}
\toprule
Model & Operator compared & Time complexity \\
\midrule
DCRNN & Diffusion recurrence over $T$ steps & $\mathcal O(Tq(N+|\mathcal E|))$ \\
Corrformer & Multi-correlation & $\mathcal O(NT\log T)$ \\
\midrule
PatchTST & Per-channel temporal attention & $\mathcal O(NP_t^2)$ \\
Crossformer & Temporal attention and routers & $\mathcal O(NP_t^2+NP_t r)$ \\
CCM (DLinear) & Cluster assignment and prototype cross-attention & $\mathcal O(NK_{\mathrm{CCM}})$ \\
DUET & Masked channel-attention fusion & $\mathcal O(N^2)$ \\
TimeFilter & Dense patch-pair graph construction & $\mathcal O(N^2P_t^2)$ \\
\midrule
STCFormer & Dual views, clustering, local--global attention & $\mathcal O(nm^2+mnK+\sum_pS_p)$ \\
\bottomrule
\end{tabular}
\end{table}

DCRNN's sparse diffusion cost depends on graph edges rather than necessarily on $N^2$~\citep{DBLP:conf/iclr/LiYS018}. Corrformer's multi-correlation uses a spatial tree and temporal correlation to obtain linear spatial and log-linear temporal scaling~\citep{DBLP:journals/natmi/WuZLW23}. PatchTST attends independently along each channel's temporal patches~\citep{DBLP:conf/iclr/NieNSK23}, whereas Crossformer additionally uses $r$ routers per temporal segment for cross-channel communication~\citep{DBLP:conf/iclr/ZhangY23}. CCM's assignment and prototype cross-attention cost $\mathcal O(NK_{\mathrm{CCM}})$ at fixed hidden width~\citep{DBLP:conf/nips/ChenLFHFTLY24}. We use DLinear as its backbone; the displayed term excludes that backbone and the cluster-aware prediction heads, and is not an end-to-end cost.

DUET's displayed $\mathcal O(N^2)$ term concerns fusion after temporal extraction and channel-mask construction; its frequency transform and learned distance metric incur additional costs~\citep{DBLP:conf/kdd/QiuW0GH025}. TimeFilter constructs pairwise scores over $NP_t$ patch tokens before neighbor selection and filtration, so a sparse final graph alone does not remove this construction cost~\citep{DBLP:conf/icml/HuZLLLC0XP25}. For STCFormer, substitute $\sum_pS_p\approx mn^2/K$ for balanced groups or $\sum_pS_p\le mn^2$ for the worst case. The matched runtime and memory measurements in Sections~\ref{sec:efficiency_analysis} and~\ref{app:additional_efficiency} complement these asymptotic comparisons and show STCFormer's favorable accuracy--resource trade-off.
\FloatBarrier

\section{Proofs for Cluster-Conditioned Local Attention}
\label{app:cluster-conditioned-proofs}

\subsection{Notation and Preliminaries}
\label{app:local-preliminaries}
We analyze the single-head residual local-attention stage in Eq.~\eqref{eq:local}, equivalently evaluated by the packed implementation in Appendix~\ref{sec:b}. The evolution view determines the partition and the state view supplies the attention content. Holding the realized partition and projection weights fixed isolates how restricted interaction changes the sensitivity bound relative to dense attention.

Feature perturbations use the Frobenius norm; weights use the spectral norm. The convex domain is $\Omega_{R_X}=\{U\in\mathbb R^{n\times D}:\max_s\|u_s\|_2\le R_X\}$. For a mapping $f$, define
\begin{equation}
\label{eq:app-lipschitz-definition}
\operatorname{Lip}_F(f)=\sup_{\substack{U,\widetilde U\in\Omega_{R_X}\\U\ne\widetilde U}}
\frac{\|f(U)-f(\widetilde U)\|_F}{\|U-\widetilde U\|_F}.
\end{equation}
Integration of the derivative along the segment joining two inputs gives
\begin{equation}
\label{eq:app-jacobian-lipschitz}
\operatorname{Lip}_F(f)\le\sup_{U\in\Omega_{R_X}}\|\mathrm Df(U)\|_{F\to F},
\qquad
\|\mathrm Df(U)\|_{F\to F}=\sup_{E\ne0}\frac{\|\mathrm Df(U)[E]\|_F}{\|E\|_F},
\end{equation}
for mappings smooth on a neighborhood of the domain. Also, $\operatorname{Lip}_F(\operatorname{Id}+f)\le1+\operatorname{Lip}_F(f)$. The hard assignment matrix $H$ has one nonzero entry per row; empty clusters are permitted and produce no output rows. We use natural logarithms, $0\log0=0$, and $1\le K\le n$.

\subsection{Proof of Theorem~\ref{thm:local-partition-lipschitz}}
\label{app:local-partition-proof}
Write the residual operator and its attention part as
\begin{equation}
\label{eq:app-local-operator}
\begin{gathered}
\mathcal A_H(U)=U+\Phi_H(U),\qquad
\Phi_H(U)=\operatorname{Softmax}\!\left(UAU^\top+\log(HH^\top)\right)UW_V^{\mathrm{loc}},\\
A=\frac{W_Q^{\mathrm{loc}}(W_K^{\mathrm{loc}})^\top}{\sqrt D},
\qquad v=\|W_V^{\mathrm{loc}}\|_2,\qquad\beta=R_X^2\|A\|_2.
\end{gathered}
\end{equation}
Here $\log0=-\infty$ implements the fixed mask. Every attention row contains its self-entry. Dense attention corresponds to a single group containing all $n$ tokens.

\begin{proof}
\textbf{Step 1: a single attention block.}\quad
Let $\mathcal I$ be a nonempty cluster of size $q$. Write token rows as columns $a_i=u_i^\top$ and set $V=(W_V^{\mathrm{loc}})^\top$. Its probabilities, weighted mean, and output are
\begin{equation}
\label{eq:app-attention-probabilities}
p_{ij}=\frac{\exp(a_i^\top A a_j)}{\sum_{\ell\in\mathcal I}\exp(a_i^\top A a_\ell)},
\qquad\mu_i=\sum_{j\in\mathcal I}p_{ij}a_j,
\qquad g_i=V\mu_i.
\end{equation}
The row-norm assumption implies $|a_i^\top A a_j|\le\beta$. Consequently,
\begin{equation}
\label{eq:app-column-mass}
p_{ij}\le\frac{e^{2\beta}}{q-1+e^{2\beta}},\qquad
\sum_{i\in\mathcal I}p_{ij}\le s_\beta(q):=\frac{q e^{2\beta}}{q-1+e^{2\beta}}\le\min\{q,e^{2\beta}\}.
\end{equation}
The probability bound follows by placing its numerator at the largest admissible logit and the other $q-1$ logits at the smallest admissible value. It sharpens the generic column-sum bound $q$ by accounting for bounded logits, consistent with the bounded-domain perspective of \citet{castin2024smooth}.

For a perturbation with token columns $e_i$, differentiating the bilinear logits and softmax gives
\begin{equation}
\label{eq:app-derivative-decomposition}
\begin{aligned}
\mathrm Dg_i[E]&=V\bigl(T_i^{\mathrm{val}}+T_i^{\mathrm{qry}}+T_i^{\mathrm{key}}\bigr),\\
T_i^{\mathrm{val}}&=\sum_jp_{ij}e_j,\qquad
T_i^{\mathrm{qry}}=C_iA^\top e_i,\\
T_i^{\mathrm{key}}&=\sum_jp_{ij}(a_i^\top A e_j)(a_j-\mu_i),
\qquad C_i=\sum_jp_{ij}(a_j-\mu_i)(a_j-\mu_i)^\top.
\end{aligned}
\end{equation}
All sums in this step are over $\mathcal I$. The covariance matrix satisfies
\begin{equation}
\label{eq:app-covariance-bound}
\|C_i\|_2\le\operatorname{tr}(C_i)
=\sum_jp_{ij}\|a_j\|_2^2-\|\mu_i\|_2^2\le R_X^2.
\end{equation}
Let $T^{\mathrm{val}},T^{\mathrm{qry}},T^{\mathrm{key}}$ collect the corresponding row contributions. Jensen's inequality, Eq.~\eqref{eq:app-column-mass}, and weighted Cauchy--Schwarz imply
\begin{equation}
\label{eq:app-value-query-bounds}
\begin{aligned}
\|T^{\mathrm{val}}\|_F^2
&\le\sum_{i,j}p_{ij}\|e_j\|_2^2
\le s_\beta(q)\|E_{\mathcal I}\|_F^2,\\
\|T^{\mathrm{qry}}\|_F^2
&\le R_X^4\|A\|_2^2\sum_i\|e_i\|_2^2
=\beta^2\|E_{\mathcal I}\|_F^2,\\
\|T^{\mathrm{key}}\|_F^2
&\le\sum_i\left(\sum_jp_{ij}(a_i^\top A e_j)^2\right)
\left(\sum_jp_{ij}\|a_j-\mu_i\|_2^2\right)\\
&\le R_X^4\|A\|_2^2\sum_{i,j}p_{ij}\|e_j\|_2^2
\le\beta^2s_\beta(q)\|E_{\mathcal I}\|_F^2.
\end{aligned}
\end{equation}
In the key term, $|a_i^\top A e_j|\le R_X\|A\|_2\|e_j\|_2$ and the second factor is at most $R_X^2$. Thus the non-residual block obeys
\begin{equation}
\label{eq:app-block-derivative}
\|\mathrm D\Phi_{\mathcal I}(U_{\mathcal I})[E_{\mathcal I}]\|_F
\le v\left[\beta+(1+\beta)\sqrt{s_\beta(q)}\right]\|E_{\mathcal I}\|_F.
\end{equation}
Equation~\eqref{eq:app-jacobian-lipschitz} therefore gives the same bound for $\operatorname{Lip}_F(\Phi_{\mathcal I})$.

\par\smallskip\noindent\textbf{Step 2: disjoint groups and the dense comparison.}\quad
Let $q_k=|\mathcal I_k|$ and $m(H)=\max_kq_k$. Reordering tokens by their fixed groups preserves Frobenius norms and makes the mapping block separable. Hence
\begin{equation}
\label{eq:app-parallel-composition}
\begin{aligned}
\|\Phi_H(U)-\Phi_H(\widetilde U)\|_F^2
&=\sum_k\|\Phi_{\mathcal I_k}(U_{\mathcal I_k})-\Phi_{\mathcal I_k}(\widetilde U_{\mathcal I_k})\|_F^2\\
&\le\left(\max_k\operatorname{Lip}_F(\Phi_{\mathcal I_k})\right)^2\|U-\widetilde U\|_F^2.
\end{aligned}
\end{equation}
The maximum is over nonempty clusters. For real $q\ge1$,
\begin{equation}
\label{eq:app-size-monotonicity}
s_\beta'(q)=\frac{e^{2\beta}(e^{2\beta}-1)}{(q-1+e^{2\beta})^2}\ge0.
\end{equation}
Adding the residual and applying monotonicity gives
\begin{equation}
\label{eq:app-fixed-partition-bound}
\operatorname{Lip}_F(\mathcal A_H)
\le1+v\left[\beta+(1+\beta)\sqrt{s_\beta(m(H))}\right]=B(m(H))\le B(n).
\end{equation}
Applying exactly the same single-block calculation to $q=n$ establishes $\operatorname{Lip}_F(\mathcal A_{\mathrm{dense}})\le B(n)$. Defining $B_{\mathrm{ours}}(H)=B(m(H))$ and $B_{\mathrm{dense}}=B(n)$ gives Eq.~\eqref{eq:local-partition-bound}. Throughout this appendix, $B(\cdot)$ denotes the derived upper bound rather than the exact Lipschitz constant in Eq.~\eqref{eq:app-lipschitz-definition}.
\end{proof}

For $v>0$, $\beta>0$, and $m(H)<n$, the comparison is strict: $B(m(H))<B(n)$. Since $s_\beta(q)\le q$, $1+v[\beta+(1+\beta)\sqrt q]$ is also valid, while the bounded-logit expression is tighter because $s_\beta(q)$ saturates at $e^{2\beta}$. The numerical diagnostics in Appendix~\ref{app:theory_validation} confirm that the predicted cluster-size dependence is visible in trained models.

\subsection{Proof of Theorem~\ref{cor:mi-partition-bound}}
\label{app:mi-sampling-proof}
\begin{proof}
Fix a temporal patch and condition on its pre-draw soft assignments $R=(r_{s,k})$. Sample token labels independently with $\Pr(\kappa_s=k\mid R)=r_{s,k}$. This is the hard forward distribution of Gumbel-Max, $\kappa_s=\arg\max_k(\log r_{s,k}+g_{s,k})$, with fresh independent standard Gumbel noises across tokens and labels. Let $h_{s,k}=\mathbf1\{\kappa_s=k\}$ and $M_k=\sum_sh_{s,k}$. The marginal and conditional entropy of the variables $X,Y$ in the main text are
\begin{equation}
\label{eq:app-joint-construction}
\pi_k=\frac1n\sum_s r_{s,k},\qquad
\bar h=H(Y\mid X)=\frac1n\sum_sH(r_{s,:}).
\end{equation}
Writing $d=\operatorname{KL}(\pi\|u_K)$, the information gap decomposes as
\begin{equation}
\label{eq:app-loss-gap}
\varepsilon=\log K-I(X;Y)=d+\bar h.
\end{equation}
Both terms are nonnegative. Pinsker's inequality with natural logarithms gives
\begin{equation}
\label{eq:app-soft-balance}
\max_k\pi_k\le\frac1K+\frac12\|\pi-u_K\|_1
\le\frac1K+\sqrt{\frac d2}.
\end{equation}
Next, center the count vector by $Z_k=M_k-n\pi_k$. Independence across tokens and $1-x\le-\log x$ imply
\begin{equation}
\label{eq:app-uncertainty-sum}
\begin{aligned}
\mathbb E[\|Z\|_2^2\mid R]
&=\sum_k\operatorname{Var}(M_k\mid R)
=\sum_{s,k}r_{s,k}(1-r_{s,k})\\
&\le-\sum_{s,k}r_{s,k}\log r_{s,k}=n\bar h.
\end{aligned}
\end{equation}
Zero entries are interpreted by continuity. No independence between the cluster counts is required: the squared Euclidean norm is a sum of coordinate squares without cross terms. For $\bar h>0$, Markov's inequality yields
\begin{equation}
\label{eq:app-count-vector-concentration}
\Pr\!\left(\|Z\|_2>\sqrt{\frac{n\bar h}{\delta}}\;\middle|\;R\right)\le\delta.
\end{equation}
If $\bar h=0$, every assignment is deterministic and $Z=0$ almost surely, so the same event bound holds. Combining this event with Eq.~\eqref{eq:app-soft-balance}, and then applying Cauchy--Schwarz, gives
\begin{equation}
\label{eq:app-mi-probability-derivation}
\begin{aligned}
m(H)&=\max_k(n\pi_k+Z_k)
\le\frac nK+n\sqrt{\frac d2}+\sqrt{\frac{n\bar h}{\delta}}\\
&\le\frac nK+\sqrt{\left(\frac{n^2}{2}+\frac n\delta\right)(d+\bar h)}\\
&=\frac nK+\sqrt{\left(\frac{n^2}{2}+\frac n\delta\right)(\log K-I(X;Y))}.
\end{aligned}
\end{equation}
Using $m(H)\le n$ proves $m(H)\le M_\delta(I(X;Y))$ with conditional probability at least $1-\delta$. Since $M_\delta(I(X;Y))\ge n/K\ge1$, monotonicity of $B$ and Theorem~\ref{thm:local-partition-lipschitz} give, on the same event,
\begin{equation}
\label{eq:app-mi-probability-sensitivity}
\operatorname{Lip}_F(\mathcal A_H)
\le B_{\mathrm{ours}}(H)
\le B\!\left(M_\delta(I(X;Y))\right).
\end{equation}
No additional failure probability is needed for the sensitivity statement, because the structural theorem holds for every realized partition.
\end{proof}

At fixed $n,K,v,\beta,\delta$, increasing $I(X;Y)$ tightens or preserves the certificate. At its maximum $I(X;Y)=\log K$, the assignments are deterministic with a uniform marginal, yielding groups of size $n/K$ whenever divisibility permits. Thus, InfoLoss directly favors the assignment regime associated with the tightest certificate.

\subsection{Deterministic argmax grouping at inference}
\label{app:mi-argmax-proof}
For row-wise argmax $H(R)$ with any fixed tie-breaking rule, a separate deterministic bound is
\begin{equation}
\label{eq:app-argmax-certificate}
\begin{aligned}
m(H(R))&\le M_{\mathrm A}(\varepsilon):=\min\!\left\{n,\;n\left[\frac1K+\sqrt{\frac\varepsilon2}+\varepsilon\right]\right\},\\
\operatorname{Lip}_F(\mathcal A_{H(R)})&\le B(M_{\mathrm A}(\varepsilon)),\qquad
\varepsilon=\log K-I(X;Y).
\end{aligned}
\end{equation}
\begin{proof}
Retain $\pi,\bar h,\varepsilon$ from Eqs.~\eqref{eq:app-joint-construction}--\eqref{eq:app-loss-gap}. Let $h_{s,:}$ be the one-hot row obtained by argmax, and write $a_s=\max_kr_{s,k}$. For each coordinate,
\begin{equation}
\label{eq:app-argmax-entropy}
|h_{s,k}-r_{s,k}|\le1-a_s\le-\log a_s\le H(r_{s,:}).
\end{equation}
The first inequality holds for the selected coordinate and every other coordinate, including ties. The last follows from $-\log r_{s,k}\ge-\log a_s$ for all positive entries. Thus, for $M_k=\sum_sh_{s,k}$, Eq.~\eqref{eq:app-soft-balance} and $d,\bar h\le\varepsilon$ give
\begin{equation}
\label{eq:app-argmax-size}
\frac{M_k}{n}\le\pi_k+\frac1n\sum_s|h_{s,k}-r_{s,k}|
\le\frac1K+\sqrt{\frac\varepsilon2}+\bar h
\le\frac1K+\sqrt{\frac\varepsilon2}+\varepsilon.
\end{equation}
Taking the maximum and using $m(H(R))\le n$ proves the size bound. Since $M_{\mathrm A}(\varepsilon)\ge n/K\ge1$, Theorem~\ref{thm:local-partition-lipschitz} proves the sensitivity bound. No sampling assumption or failure probability is needed.
\end{proof}

\subsection{Patching, multiple heads, and simultaneous bounds}
\label{app:patching-scope}
At fixed $K,v,\beta,\delta$ and information gap, the dense ceiling $B(n)$ and the information-controlled size bounds are nondecreasing in $n$. Spatial patching replaces $N$ with $n=\lceil N/L_s\rceil$, and approximately balanced groups reduce the maximum interacting set to about $n/K$. This gives spatial patching and clustering complementary roles in controlling the derived bound.

The deterministic bound in Appendix~\ref{app:mi-argmax-proof} applies separately to every temporal patch. For a simultaneous sampled bound over $T'\ge1$ patches, put $\varepsilon_* =\max_p\varepsilon_p$. Applying Theorem~\ref{cor:mi-partition-bound} with failure probability $\delta/T'$ per patch and taking a union bound replaces $n/\delta$ by $nT'/\delta$ and the information gap by $\varepsilon_*$ in Eq.~\eqref{eq:mi-size-envelope}. No independence between temporal patches is required, provided each analyzed patch uses independent token-wise sampling conditional on its pre-draw soft assignments. Temporal patching changes the confidence factor, not the size of a single spatial-attention block.

For layer $\ell$, let $\overline{\mathcal L}_{\mathrm{info}}^{(\ell)}=-T'^{-1}\sum_p I_p^{(\ell)}$. Then $\sum_p\varepsilon_p^{(\ell)}=T'(\overline{\mathcal L}_{\mathrm{info}}^{(\ell)}+\log K_\ell)$, and Eq.~\eqref{eq:mi-loss} sums these layer-wise averages. The training objective therefore reduces the aggregate information gap across layers and temporal patches while remaining invariant to cluster-label permutations, allowing groups to adapt locally over time.

The block argument also extends to multiple heads sharing a partition. If head $h$ has non-residual bound $b_h(q)=v_h[\beta_h+(1+\beta_h)\sqrt{s_{\beta_h}(q)}]$ and the concatenated outputs are projected by $W_O$, the corresponding residual multi-head bound is
\begin{equation}
\label{eq:app-multihead-bound}
\operatorname{Lip}_F(\mathcal A_H^{\mathrm{MH}})
\le1+\|W_O\|_2\left(\sum_h b_h(m(H))^2\right)^{1/2}.
\end{equation}
This follows by summing squared head-output differences before applying the output projection. Normalization, feed-forward layers, regional attention, and the prediction head require their own bounds when analyzing a larger composition.

\section{Properties of InfoLoss}
\label{sec:infoloss-properties}
InfoLoss maximizes mutual information between station-token data and cluster assignments separately within each temporal patch. At patch $p$, a token is drawn uniformly from the $n=N'$ data items and its cluster is sampled according to the corresponding row of $\mathbf R_p$. The empirical conditional distribution is therefore $\Pr(Y_p=k\mid X_p=\mathbf x^{\mathrm{clu}}_{p,s})=R_{p,s,k}$, and the marginal is $\pi_{p,k}=n^{-1}\sum_sR_{p,s,k}$. The following properties concern the soft assignments; hard groups are obtained by sampling or argmax as described in Section~\ref{sec:sc-section}.

\paragraph{Cluster balance.}
The marginal entropy satisfies $H(Y_p)=\log K-\operatorname{KL}(\boldsymbol\pi_p\|u_K)$, where $u_K$ is uniform. Maximizing this entropy encourages balanced soft cluster masses and discourages collapse into a few clusters. This equal-mass preference discourages cluster collapse and promotes broad use of the available regional units.

\paragraph{Assignment confidence.}
The conditional entropy is exactly the average row entropy,
\begin{equation}
H(Y_p\mid X_p)=\frac1n\sum_sH(\mathbf R_{p,s,:}).
\end{equation}
Minimizing it favors confident, near-deterministic token-to-cluster assignments. Together with evolution-driven centroids, this produces clear and input-adaptive regional memberships.

\paragraph{Joint effect and temporal scope.}
For $\ell_p=-I(X_p;Y_p)$, the gap to the lower bound $-\log K$ is
\begin{equation}
\label{eq:app-info-properties-gap}
\varepsilon_p=\ell_p+\log K
=H(Y_p\mid X_p)+\operatorname{KL}(\boldsymbol\pi_p\|u_K).
\end{equation}
Both terms are nonnegative, so a small gap jointly promotes confident assignments and balanced cluster usage. Theorem~\ref{cor:mi-partition-bound} translates this gap into a probabilistic upper bound on the largest sampled group and hence into the fixed-partition Lipschitz upper bound of Theorem~\ref{thm:local-partition-lipschitz}. Appendix~\ref{app:mi-argmax-proof} provides the deterministic argmax counterpart. A balanced deterministic partition attains zero when divisibility permits. Because the loss is evaluated independently at each temporal patch and is invariant to cluster-label permutations, it supports the locally adaptive regrouping central to STCFormer.

% Detailed setup moved from the former main-text Sections 4.1 and 4.2.
\section{Experimental Setup Details}
\label{sec:setting}
\label{app:experimental_setup}

\subsection{Datasets and Preprocessing}
\label{app:setup_datasets}
\begin{table}[!htbp]
\centering
\captionsetup{font=small}
\caption{Statistics of the processed weather datasets. Each variable defines a forecasting task.}
\label{tab:dataset_statistics}
\scriptsize
\setlength{\tabcolsep}{2.2pt}
\renewcommand{\arraystretch}{1.10}
\begin{tabular*}{0.78\textwidth}{@{\extracolsep{\fill}}lccc@{}}
\toprule
Statistic & French & Hunan & Global \\
\midrule
Stations & 234 & 3,570 & 3,850 \\
Time steps & 8,784 & 4,368 & 17,544 \\
Start & \shortstack{2016-01-01\\00:00} & \shortstack{2023-04-01\\09:00} & \shortstack{2019-01-01\\00:00} \\
End & \shortstack{2016-12-31\\23:00} & \shortstack{2023-09-30\\08:00} & \shortstack{2020-12-31\\23:00} \\
Interval & 1 hour & 1 hour & 1 hour \\
Variables & \shortstack{Temp., U-wind,\\V-wind} & \shortstack{Temp., U-wind,\\V-wind} & \shortstack{Temp.,\\wind speed} \\
\bottomrule
\end{tabular*}
\end{table}

The \textbf{French dataset}\footnote{https://github.com/meteofrance/meteonet} comprises hourly temperature, u-wind, and v-wind variables collected from 234 meteorological stations, spanning from January 1, 2016, to December 31, 2016. This dataset is derived from the MeteoNet dataset.
The \textbf{Hunan dataset} contains hourly temperature, u-wind, and v-wind data collected from 3,570 meteorological stations, covering the period from April 1, 2023, to September 30, 2023. This dataset is provided by the Hunan Meteorological Bureau.
The \textbf{Global dataset} includes hourly temperature and wind speed variables collected from 3,850 meteorological stations worldwide, spanning from January 1, 2019, to December 31, 2020. This dataset was curated by Corrformer~\cite{DBLP:journals/natmi/WuZLW23}.
For all datasets, we partition the data into training, validation, and test sets with a ratio of 6:2:2. We apply Z-score normalization using the statistics of the training set.

Table~\ref{tab:dataset_statistics} summarizes the processed datasets, comprising eight separate dataset--variable forecasting tasks. Station counts refer to the original stations before spatial patching. The chronological split uses $\lfloor 0.2T\rfloor$ time steps for each of validation and testing, with the remainder used for training, where $T$ is the total sequence length.

\subsection{Baselines}
\label{app:setup_baselines}
In our experiments, we compare STCFormer against 15 baselines in two categories. The six spatio-temporal methods are DCRNN~\cite{DBLP:conf/iclr/LiYS018}, Corrformer~\cite{DBLP:journals/natmi/WuZLW23}, EasyST~\cite{DBLP:conf/cikm/Tang0X024}, CDPNet~\cite{DBLP:conf/ijcai/XuMDGW0025}, STELLA~\cite{fu2025stella}, and S$^2$Transformer~\cite{chen2026s2transformer}. The nine multivariate time series methods are DLinear~\cite{DBLP:conf/aaai/ZengCZ023}, PatchTST~\cite{DBLP:conf/iclr/NieNSK23}, Crossformer~\cite{DBLP:conf/iclr/ZhangY23}, TimeXer~\cite{DBLP:conf/nips/WangWDQZLQWL24}, CCM~\cite{DBLP:conf/nips/ChenLFHFTLY24}, DUET~\cite{DBLP:conf/kdd/QiuW0GH025}, TimeFilter~\cite{DBLP:conf/icml/HuZLLLC0XP25}, TQNet~\cite{lin2025tqnet}, and xPatch~\cite{stitsyuk2025xpatch}. CCM uses DLinear as its forecasting backbone and learns cluster-specific shared parameters; tables abbreviate this combination as CCM. These baselines cover graph-based spatial modeling, station-specific forecasting, and multivariate temporal modeling.

\subsection{Implementation and Hyperparameters}
\label{app:setup_implementation}
\paragraph{Training and evaluation.}
All experiments are implemented with PyTorch and BasicTS~\cite{DBLP:journals/tkde/ShaoWXWYZYSJCCJC25} on a server equipped with an NVIDIA A100 GPU with 40\,GB of memory. We use the past 48 hours to predict the next 24 and 72 hours, with MSE as the training loss. The maximum number of training epochs is 50, with an early-stopping patience of 5. All models use Adam with an initial learning rate of $10^{-4}$, a weight decay of $10^{-4}$, and a cosine annealing learning rate scheduler. For TimeFilter on Hunan and Global, we use a temporal patch length of 24 and 30 training epochs; the French configuration is unchanged. All reported experimental results, including hyperparameter sweeps, are means of three runs; tables also report standard deviations where available. Efficiency benchmarks use batch size 8 for every method. The main text reports 24-hour MSE. Additional forecasting metrics and 72-hour results are provided in Appendix~\ref{sec:additional_results}, and additional ablation metrics in Appendix~\ref{sec:additional_ablation}.

\paragraph{Evaluation metrics.}
MSE and MAE are computed after inverse normalization, in the original data space. For each dataset's target-variable task, SEDI uses one fixed threshold shared by all stations: the 95th percentile of the training-period observations pooled across all stations. The threshold is then held fixed for all methods and evaluated horizons. For each run, we pool all test samples into a single contingency table and compute SEDI once, rather than averaging station-wise scores. Let $h=\mathrm{TP}/(\mathrm{TP}+\mathrm{FN})$ and $f=\mathrm{FP}/(\mathrm{FP}+\mathrm{TN})$ denote the pooled hit and false-alarm rates. Then $\mathrm{SEDI}=[\log f-\log h-\log(1-f)+\log(1-h)]/[\log f+\log h+\log(1-f)+\log(1-h)]$, with larger values indicating better extreme-event discrimination. Reported SEDI values summarize the three run-level scores.\paragraph{Model configuration.}
The common settings across datasets are an embedding dimension of 256, an FFN dimension of 512, a dropout rate of 0.1, and four attention heads. The temporal and spatial patch sizes $(L_t,L_s)$ are $(4,2)$ for French, $(4,7)$ for Hunan, and $(6,7)$ for Global. STCFormer uses three layers, a clustering ratio of 0.05, and an InfoLoss weight of $\lambda=0.1$ on all datasets. InfoLoss is the negative within-patch mutual information in Eq.~\eqref{eq:mi-loss}, computed at every layer, averaged over temporal patches, and summed across layers; no additional entropy coefficient is used. The layer counts of the baselines follow their original papers. For models requiring a predefined graph, we construct a 20-nearest-neighbor spatial adjacency matrix using station latitude and longitude.

% Generated from Feishu main revision 310 and ablation revision 193.
% STCFormer denotes the STCFormerV4 results in the source tables.
% Corrformer / Hunan Temp. / 72h / MAE SD corrected to 0.245 with author confirmation.
\section{Additional Forecasting Results}
\label{sec:additional_results}
Tables~\ref{tab:results_mae_24}--\ref{tab:results_sedi_72} report 24-hour MAE and SEDI and all 72-hour metrics under the main experimental setup. Across the complete 48 comparisons, STCFormer ranks first in 42 and second in five, demonstrating consistently strong performance across metrics and horizons.

\begin{table}[!htbp]
\centering
\caption{24-hour forecasting results (MAE, $\downarrow$) with a 48-hour look-back window. Values are mean $\pm$ standard deviation over three runs. Bold and underlined means indicate the best and second-best reported values in each column; ties share the same marking.}
\label{tab:results_mae_24}
\begingroup
\fontsize{8}{9.5}\selectfont
\setlength{\tabcolsep}{0.8pt}
\renewcommand{\arraystretch}{1.10}
\begin{tabular*}{\textwidth}{@{\extracolsep{\fill}}l*{8}{c}@{}}
\toprule
\multirow{2}{*}{\textbf{Method}} & \multicolumn{3}{c}{\textbf{French}} & \multicolumn{3}{c}{\textbf{Hunan}} & \multicolumn{2}{c}{\textbf{Global}} \\
\cmidrule(lr){2-4}\cmidrule(lr){5-7}\cmidrule(lr){8-9}
 & Temp. & U-Wind & V-Wind & Temp. & U-Wind & V-Wind & Temp. & Wind \\
\midrule
\multicolumn{9}{l}{\textit{Spatio-Temporal Methods}} \\
DCRNN & \underline{1.894}{\fontsize{6}{7}\selectfont\textpm0.017} & 1.791{\fontsize{6}{7}\selectfont\textpm0.040} & 2.157{\fontsize{6}{7}\selectfont\textpm0.025} & 1.558{\fontsize{6}{7}\selectfont\textpm0.456} & 0.654{\fontsize{6}{7}\selectfont\textpm0.005} & 0.783{\fontsize{6}{7}\selectfont\textpm0.009} & 1.941{\fontsize{6}{7}\selectfont\textpm0.023} & \underline{1.282}{\fontsize{6}{7}\selectfont\textpm0.000} \\
Corrformer & 2.121{\fontsize{6}{7}\selectfont\textpm0.018} & 1.540{\fontsize{6}{7}\selectfont\textpm0.016} & 1.794{\fontsize{6}{7}\selectfont\textpm0.005} & 1.662{\fontsize{6}{7}\selectfont\textpm0.126} & 0.675{\fontsize{6}{7}\selectfont\textpm0.001} & 0.756{\fontsize{6}{7}\selectfont\textpm0.002} & 1.888{\fontsize{6}{7}\selectfont\textpm0.004} & 1.304{\fontsize{6}{7}\selectfont\textpm0.011} \\
EasyST & 2.138{\fontsize{6}{7}\selectfont\textpm0.100} & 1.554{\fontsize{6}{7}\selectfont\textpm0.023} & 1.673{\fontsize{6}{7}\selectfont\textpm0.000} & 1.829{\fontsize{6}{7}\selectfont\textpm0.230} & 0.683{\fontsize{6}{7}\selectfont\textpm0.012} & 0.762{\fontsize{6}{7}\selectfont\textpm0.012} & 1.843{\fontsize{6}{7}\selectfont\textpm0.003} & 1.289{\fontsize{6}{7}\selectfont\textpm0.002} \\
CDPNet & 2.083{\fontsize{6}{7}\selectfont\textpm0.022} & 1.502{\fontsize{6}{7}\selectfont\textpm0.004} & 1.535{\fontsize{6}{7}\selectfont\textpm0.007} & 1.696{\fontsize{6}{7}\selectfont\textpm0.009} & 0.691{\fontsize{6}{7}\selectfont\textpm0.000} & 0.789{\fontsize{6}{7}\selectfont\textpm0.004} & 1.843{\fontsize{6}{7}\selectfont\textpm0.001} & 1.290{\fontsize{6}{7}\selectfont\textpm0.002} \\
STELLA & 4.099{\fontsize{6}{7}\selectfont\textpm0.291} & 2.438{\fontsize{6}{7}\selectfont\textpm0.166} & 2.336{\fontsize{6}{7}\selectfont\textpm0.155} & 3.674{\fontsize{6}{7}\selectfont\textpm0.170} & 1.038{\fontsize{6}{7}\selectfont\textpm0.022} & 1.304{\fontsize{6}{7}\selectfont\textpm0.057} & 1.966{\fontsize{6}{7}\selectfont\textpm0.002} & 1.315{\fontsize{6}{7}\selectfont\textpm0.000} \\
S$^2$Transformer & 2.141{\fontsize{6}{7}\selectfont\textpm0.016} & 1.484{\fontsize{6}{7}\selectfont\textpm0.030} & 1.567{\fontsize{6}{7}\selectfont\textpm0.057} & 1.955{\fontsize{6}{7}\selectfont\textpm0.099} & 0.689{\fontsize{6}{7}\selectfont\textpm0.001} & 0.805{\fontsize{6}{7}\selectfont\textpm0.001} & \underline{1.802}{\fontsize{6}{7}\selectfont\textpm0.001} & 1.290{\fontsize{6}{7}\selectfont\textpm0.000} \\
\midrule
\multicolumn{9}{l}{\textit{Multivariate Time Series Methods}} \\
DLinear & 1.964{\fontsize{6}{7}\selectfont\textpm0.003} & 1.519{\fontsize{6}{7}\selectfont\textpm0.001} & 1.575{\fontsize{6}{7}\selectfont\textpm0.000} & 1.583{\fontsize{6}{7}\selectfont\textpm0.002} & 0.643{\fontsize{6}{7}\selectfont\textpm0.000} & 0.723{\fontsize{6}{7}\selectfont\textpm0.000} & 2.062{\fontsize{6}{7}\selectfont\textpm0.000} & 1.392{\fontsize{6}{7}\selectfont\textpm0.000} \\
PatchTST & 1.971{\fontsize{6}{7}\selectfont\textpm0.006} & 1.515{\fontsize{6}{7}\selectfont\textpm0.003} & 1.608{\fontsize{6}{7}\selectfont\textpm0.002} & 1.577{\fontsize{6}{7}\selectfont\textpm0.005} & 0.652{\fontsize{6}{7}\selectfont\textpm0.001} & 0.725{\fontsize{6}{7}\selectfont\textpm0.001} & 1.918{\fontsize{6}{7}\selectfont\textpm0.000} & 1.380{\fontsize{6}{7}\selectfont\textpm0.001} \\
Crossformer & 2.111{\fontsize{6}{7}\selectfont\textpm0.005} & \underline{1.408}{\fontsize{6}{7}\selectfont\textpm0.000} & 1.520{\fontsize{6}{7}\selectfont\textpm0.004} & 1.634{\fontsize{6}{7}\selectfont\textpm0.015} & 0.659{\fontsize{6}{7}\selectfont\textpm0.003} & 0.750{\fontsize{6}{7}\selectfont\textpm0.003} & 1.903{\fontsize{6}{7}\selectfont\textpm0.007} & 1.299{\fontsize{6}{7}\selectfont\textpm0.003} \\
TimeXer & 1.907{\fontsize{6}{7}\selectfont\textpm0.025} & 1.445{\fontsize{6}{7}\selectfont\textpm0.008} & 1.551{\fontsize{6}{7}\selectfont\textpm0.001} & 1.634{\fontsize{6}{7}\selectfont\textpm0.026} & 0.659{\fontsize{6}{7}\selectfont\textpm0.001} & 0.730{\fontsize{6}{7}\selectfont\textpm0.000} & 1.912{\fontsize{6}{7}\selectfont\textpm0.003} & 1.352{\fontsize{6}{7}\selectfont\textpm0.000} \\
CCM & 1.903{\fontsize{6}{7}\selectfont\textpm0.004} & 1.410{\fontsize{6}{7}\selectfont\textpm0.002} & 1.502{\fontsize{6}{7}\selectfont\textpm0.001} & 1.505{\fontsize{6}{7}\selectfont\textpm0.003} & 0.643{\fontsize{6}{7}\selectfont\textpm0.000} & 0.718{\fontsize{6}{7}\selectfont\textpm0.001} & 1.891{\fontsize{6}{7}\selectfont\textpm0.000} & 1.330{\fontsize{6}{7}\selectfont\textpm0.000} \\
DUET & 1.971{\fontsize{6}{7}\selectfont\textpm0.029} & 1.465{\fontsize{6}{7}\selectfont\textpm0.004} & 1.484{\fontsize{6}{7}\selectfont\textpm0.012} & 1.697{\fontsize{6}{7}\selectfont\textpm0.004} & 0.646{\fontsize{6}{7}\selectfont\textpm0.006} & 0.786{\fontsize{6}{7}\selectfont\textpm0.011} & 1.872{\fontsize{6}{7}\selectfont\textpm0.016} & 1.288{\fontsize{6}{7}\selectfont\textpm0.006} \\
TimeFilter & 1.940{\fontsize{6}{7}\selectfont\textpm0.012} & 1.416{\fontsize{6}{7}\selectfont\textpm0.004} & \underline{1.483}{\fontsize{6}{7}\selectfont\textpm0.002} & 1.529{\fontsize{6}{7}\selectfont\textpm0.018} & \underline{0.639}{\fontsize{6}{7}\selectfont\textpm0.004} & \underline{0.712}{\fontsize{6}{7}\selectfont\textpm0.009} & 1.853{\fontsize{6}{7}\selectfont\textpm0.025} & \underline{1.282}{\fontsize{6}{7}\selectfont\textpm0.009} \\
TQNet & 1.930{\fontsize{6}{7}\selectfont\textpm0.001} & 1.426{\fontsize{6}{7}\selectfont\textpm0.003} & 1.505{\fontsize{6}{7}\selectfont\textpm0.008} & \underline{1.492}{\fontsize{6}{7}\selectfont\textpm0.008} & 0.646{\fontsize{6}{7}\selectfont\textpm0.000} & 0.719{\fontsize{6}{7}\selectfont\textpm0.002} & 1.840{\fontsize{6}{7}\selectfont\textpm0.001} & 1.297{\fontsize{6}{7}\selectfont\textpm0.000} \\
xPatch & 1.982{\fontsize{6}{7}\selectfont\textpm0.003} & 1.472{\fontsize{6}{7}\selectfont\textpm0.000} & 1.547{\fontsize{6}{7}\selectfont\textpm0.001} & 1.545{\fontsize{6}{7}\selectfont\textpm0.003} & 0.643{\fontsize{6}{7}\selectfont\textpm0.000} & 0.720{\fontsize{6}{7}\selectfont\textpm0.000} & 1.925{\fontsize{6}{7}\selectfont\textpm0.001} & 1.322{\fontsize{6}{7}\selectfont\textpm0.000} \\
\midrule
\textbf{STCFormer} & \textbf{1.861}{\fontsize{6}{7}\selectfont\textpm0.015} & \textbf{1.400}{\fontsize{6}{7}\selectfont\textpm0.008} & \textbf{1.466}{\fontsize{6}{7}\selectfont\textpm0.004} & \textbf{1.482}{\fontsize{6}{7}\selectfont\textpm0.026} & \textbf{0.622}{\fontsize{6}{7}\selectfont\textpm0.001} & \textbf{0.700}{\fontsize{6}{7}\selectfont\textpm0.002} & \textbf{1.757}{\fontsize{6}{7}\selectfont\textpm0.000} & \textbf{1.273}{\fontsize{6}{7}\selectfont\textpm0.001} \\
\bottomrule
\end{tabular*}
\endgroup
\end{table}
\begin{table}[!htbp]
\centering
\caption{24-hour forecasting results (SEDI, $\uparrow$) with a 48-hour look-back window. Values are mean $\pm$ standard deviation over three runs. Bold and underlined means indicate the best and second-best reported values in each column; ties share the same marking.}
\label{tab:results_sedi_24}
\begingroup
\fontsize{8}{9.5}\selectfont
\setlength{\tabcolsep}{0.8pt}
\renewcommand{\arraystretch}{1.10}
\begin{tabular*}{\textwidth}{@{\extracolsep{\fill}}l*{8}{c}@{}}
\toprule
\multirow{2}{*}{\textbf{Method}} & \multicolumn{3}{c}{\textbf{French}} & \multicolumn{3}{c}{\textbf{Hunan}} & \multicolumn{2}{c}{\textbf{Global}} \\
\cmidrule(lr){2-4}\cmidrule(lr){5-7}\cmidrule(lr){8-9}
 & Temp. & U-Wind & V-Wind & Temp. & U-Wind & V-Wind & Temp. & Wind \\
\midrule
\multicolumn{9}{l}{\textit{Spatio-Temporal Methods}} \\
DCRNN & 0.556{\fontsize{6}{7}\selectfont\textpm0.061} & 0.676{\fontsize{6}{7}\selectfont\textpm0.006} & 0.607{\fontsize{6}{7}\selectfont\textpm0.053} & 0.583{\fontsize{6}{7}\selectfont\textpm0.091} & 0.561{\fontsize{6}{7}\selectfont\textpm0.030} & 0.613{\fontsize{6}{7}\selectfont\textpm0.052} & \textbf{0.920}{\fontsize{6}{7}\selectfont\textpm0.003} & 0.648{\fontsize{6}{7}\selectfont\textpm0.008} \\
Corrformer & 0.503{\fontsize{6}{7}\selectfont\textpm0.012} & 0.645{\fontsize{6}{7}\selectfont\textpm0.019} & \textbf{0.616}{\fontsize{6}{7}\selectfont\textpm0.015} & 0.649{\fontsize{6}{7}\selectfont\textpm0.058} & 0.532{\fontsize{6}{7}\selectfont\textpm0.001} & 0.627{\fontsize{6}{7}\selectfont\textpm0.001} & 0.887{\fontsize{6}{7}\selectfont\textpm0.008} & 0.595{\fontsize{6}{7}\selectfont\textpm0.002} \\
EasyST & 0.632{\fontsize{6}{7}\selectfont\textpm0.004} & 0.598{\fontsize{6}{7}\selectfont\textpm0.000} & 0.483{\fontsize{6}{7}\selectfont\textpm0.003} & 0.701{\fontsize{6}{7}\selectfont\textpm0.134} & 0.541{\fontsize{6}{7}\selectfont\textpm0.010} & 0.551{\fontsize{6}{7}\selectfont\textpm0.007} & 0.901{\fontsize{6}{7}\selectfont\textpm0.001} & 0.631{\fontsize{6}{7}\selectfont\textpm0.001} \\
CDPNet & 0.625{\fontsize{6}{7}\selectfont\textpm0.002} & 0.671{\fontsize{6}{7}\selectfont\textpm0.004} & 0.607{\fontsize{6}{7}\selectfont\textpm0.004} & 0.735{\fontsize{6}{7}\selectfont\textpm0.021} & 0.522{\fontsize{6}{7}\selectfont\textpm0.006} & \underline{0.629}{\fontsize{6}{7}\selectfont\textpm0.015} & 0.916{\fontsize{6}{7}\selectfont\textpm0.001} & 0.639{\fontsize{6}{7}\selectfont\textpm0.008} \\
STELLA & 0.452{\fontsize{6}{7}\selectfont\textpm0.048} & 0.530{\fontsize{6}{7}\selectfont\textpm0.029} & 0.400{\fontsize{6}{7}\selectfont\textpm0.041} & 0.492{\fontsize{6}{7}\selectfont\textpm0.089} & 0.480{\fontsize{6}{7}\selectfont\textpm0.025} & 0.533{\fontsize{6}{7}\selectfont\textpm0.054} & 0.899{\fontsize{6}{7}\selectfont\textpm0.003} & 0.597{\fontsize{6}{7}\selectfont\textpm0.001} \\
S$^2$Transformer & 0.554{\fontsize{6}{7}\selectfont\textpm0.015} & 0.576{\fontsize{6}{7}\selectfont\textpm0.006} & 0.493{\fontsize{6}{7}\selectfont\textpm0.001} & 0.653{\fontsize{6}{7}\selectfont\textpm0.069} & 0.533{\fontsize{6}{7}\selectfont\textpm0.035} & 0.592{\fontsize{6}{7}\selectfont\textpm0.027} & 0.911{\fontsize{6}{7}\selectfont\textpm0.001} & 0.617{\fontsize{6}{7}\selectfont\textpm0.005} \\
\midrule
\multicolumn{9}{l}{\textit{Multivariate Time Series Methods}} \\
DLinear & 0.644{\fontsize{6}{7}\selectfont\textpm0.002} & 0.607{\fontsize{6}{7}\selectfont\textpm0.000} & 0.524{\fontsize{6}{7}\selectfont\textpm0.000} & \underline{0.744}{\fontsize{6}{7}\selectfont\textpm0.000} & 0.502{\fontsize{6}{7}\selectfont\textpm0.000} & 0.547{\fontsize{6}{7}\selectfont\textpm0.000} & 0.898{\fontsize{6}{7}\selectfont\textpm0.001} & 0.602{\fontsize{6}{7}\selectfont\textpm0.001} \\
PatchTST & 0.644{\fontsize{6}{7}\selectfont\textpm0.004} & 0.649{\fontsize{6}{7}\selectfont\textpm0.000} & 0.587{\fontsize{6}{7}\selectfont\textpm0.006} & 0.742{\fontsize{6}{7}\selectfont\textpm0.012} & 0.567{\fontsize{6}{7}\selectfont\textpm0.001} & 0.598{\fontsize{6}{7}\selectfont\textpm0.011} & 0.910{\fontsize{6}{7}\selectfont\textpm0.000} & 0.621{\fontsize{6}{7}\selectfont\textpm0.001} \\
Crossformer & 0.587{\fontsize{6}{7}\selectfont\textpm0.042} & \underline{0.684}{\fontsize{6}{7}\selectfont\textpm0.013} & 0.602{\fontsize{6}{7}\selectfont\textpm0.012} & 0.741{\fontsize{6}{7}\selectfont\textpm0.008} & 0.513{\fontsize{6}{7}\selectfont\textpm0.005} & 0.621{\fontsize{6}{7}\selectfont\textpm0.008} & 0.908{\fontsize{6}{7}\selectfont\textpm0.001} & 0.647{\fontsize{6}{7}\selectfont\textpm0.009} \\
TimeXer & 0.630{\fontsize{6}{7}\selectfont\textpm0.001} & 0.645{\fontsize{6}{7}\selectfont\textpm0.000} & 0.603{\fontsize{6}{7}\selectfont\textpm0.016} & 0.670{\fontsize{6}{7}\selectfont\textpm0.005} & 0.572{\fontsize{6}{7}\selectfont\textpm0.003} & \underline{0.629}{\fontsize{6}{7}\selectfont\textpm0.002} & 0.913{\fontsize{6}{7}\selectfont\textpm0.001} & 0.646{\fontsize{6}{7}\selectfont\textpm0.001} \\
CCM & 0.643{\fontsize{6}{7}\selectfont\textpm0.006} & 0.671{\fontsize{6}{7}\selectfont\textpm0.005} & 0.607{\fontsize{6}{7}\selectfont\textpm0.006} & 0.707{\fontsize{6}{7}\selectfont\textpm0.012} & 0.571{\fontsize{6}{7}\selectfont\textpm0.002} & \underline{0.629}{\fontsize{6}{7}\selectfont\textpm0.002} & 0.916{\fontsize{6}{7}\selectfont\textpm0.000} & 0.628{\fontsize{6}{7}\selectfont\textpm0.001} \\
DUET & 0.619{\fontsize{6}{7}\selectfont\textpm0.005} & 0.637{\fontsize{6}{7}\selectfont\textpm0.003} & 0.587{\fontsize{6}{7}\selectfont\textpm0.008} & 0.717{\fontsize{6}{7}\selectfont\textpm0.042} & 0.525{\fontsize{6}{7}\selectfont\textpm0.013} & 0.600{\fontsize{6}{7}\selectfont\textpm0.004} & 0.910{\fontsize{6}{7}\selectfont\textpm0.002} & 0.616{\fontsize{6}{7}\selectfont\textpm0.000} \\
TimeFilter & 0.638{\fontsize{6}{7}\selectfont\textpm0.022} & 0.670{\fontsize{6}{7}\selectfont\textpm0.009} & 0.609{\fontsize{6}{7}\selectfont\textpm0.007} & 0.723{\fontsize{6}{7}\selectfont\textpm0.013} & \underline{0.573}{\fontsize{6}{7}\selectfont\textpm0.001} & \underline{0.629}{\fontsize{6}{7}\selectfont\textpm0.005} & 0.917{\fontsize{6}{7}\selectfont\textpm0.001} & \underline{0.650}{\fontsize{6}{7}\selectfont\textpm0.001} \\
TQNet & \underline{0.669}{\fontsize{6}{7}\selectfont\textpm0.007} & 0.671{\fontsize{6}{7}\selectfont\textpm0.004} & 0.605{\fontsize{6}{7}\selectfont\textpm0.003} & 0.740{\fontsize{6}{7}\selectfont\textpm0.005} & 0.571{\fontsize{6}{7}\selectfont\textpm0.003} & 0.611{\fontsize{6}{7}\selectfont\textpm0.003} & 0.916{\fontsize{6}{7}\selectfont\textpm0.001} & 0.643{\fontsize{6}{7}\selectfont\textpm0.001} \\
xPatch & 0.647{\fontsize{6}{7}\selectfont\textpm0.011} & 0.658{\fontsize{6}{7}\selectfont\textpm0.001} & 0.589{\fontsize{6}{7}\selectfont\textpm0.001} & 0.740{\fontsize{6}{7}\selectfont\textpm0.009} & 0.567{\fontsize{6}{7}\selectfont\textpm0.001} & 0.604{\fontsize{6}{7}\selectfont\textpm0.001} & 0.912{\fontsize{6}{7}\selectfont\textpm0.000} & 0.638{\fontsize{6}{7}\selectfont\textpm0.001} \\
\midrule
\textbf{STCFormer} & \textbf{0.672}{\fontsize{6}{7}\selectfont\textpm0.001} & \textbf{0.689}{\fontsize{6}{7}\selectfont\textpm0.007} & \underline{0.615}{\fontsize{6}{7}\selectfont\textpm0.006} & \textbf{0.751}{\fontsize{6}{7}\selectfont\textpm0.009} & \textbf{0.580}{\fontsize{6}{7}\selectfont\textpm0.001} & \textbf{0.639}{\fontsize{6}{7}\selectfont\textpm0.017} & \underline{0.918}{\fontsize{6}{7}\selectfont\textpm0.002} & \textbf{0.653}{\fontsize{6}{7}\selectfont\textpm0.001} \\
\bottomrule
\end{tabular*}
\endgroup
\end{table}
\begin{table}[!htbp]
\centering
\caption{72-hour forecasting results (MAE, $\downarrow$) with a 48-hour look-back window. Values are mean $\pm$ standard deviation over three runs. Bold and underlined means indicate the best and second-best reported values in each column; ties share the same marking.}
\label{tab:results_mae_72}
\begingroup
\fontsize{8}{9.5}\selectfont
\setlength{\tabcolsep}{0.8pt}
\renewcommand{\arraystretch}{1.10}
\begin{tabular*}{\textwidth}{@{\extracolsep{\fill}}l*{8}{c}@{}}
\toprule
\multirow{2}{*}{\textbf{Method}} & \multicolumn{3}{c}{\textbf{French}} & \multicolumn{3}{c}{\textbf{Hunan}} & \multicolumn{2}{c}{\textbf{Global}} \\
\cmidrule(lr){2-4}\cmidrule(lr){5-7}\cmidrule(lr){8-9}
 & Temp. & U-Wind & V-Wind & Temp. & U-Wind & V-Wind & Temp. & Wind \\
\midrule
\multicolumn{9}{l}{\textit{Spatio-Temporal Methods}} \\
DCRNN & 3.302{\fontsize{6}{7}\selectfont\textpm0.016} & 1.987{\fontsize{6}{7}\selectfont\textpm0.126} & 2.103{\fontsize{6}{7}\selectfont\textpm0.127} & 2.388{\fontsize{6}{7}\selectfont\textpm0.034} & 0.733{\fontsize{6}{7}\selectfont\textpm0.003} & 0.891{\fontsize{6}{7}\selectfont\textpm0.026} & 2.924{\fontsize{6}{7}\selectfont\textpm0.340} & 1.492{\fontsize{6}{7}\selectfont\textpm0.003} \\
Corrformer & 2.711{\fontsize{6}{7}\selectfont\textpm0.030} & 2.100{\fontsize{6}{7}\selectfont\textpm0.034} & 1.996{\fontsize{6}{7}\selectfont\textpm0.008} & 2.813{\fontsize{6}{7}\selectfont\textpm0.245} & 0.721{\fontsize{6}{7}\selectfont\textpm0.002} & 0.893{\fontsize{6}{7}\selectfont\textpm0.005} & 2.784{\fontsize{6}{7}\selectfont\textpm0.005} & 1.465{\fontsize{6}{7}\selectfont\textpm0.002} \\
EasyST & 2.878{\fontsize{6}{7}\selectfont\textpm0.063} & 1.992{\fontsize{6}{7}\selectfont\textpm0.016} & 1.957{\fontsize{6}{7}\selectfont\textpm0.049} & 4.431{\fontsize{6}{7}\selectfont\textpm0.660} & 0.710{\fontsize{6}{7}\selectfont\textpm0.002} & 0.855{\fontsize{6}{7}\selectfont\textpm0.010} & \underline{2.502}{\fontsize{6}{7}\selectfont\textpm0.018} & 1.498{\fontsize{6}{7}\selectfont\textpm0.000} \\
CDPNet & 2.862{\fontsize{6}{7}\selectfont\textpm0.007} & 1.973{\fontsize{6}{7}\selectfont\textpm0.005} & 1.966{\fontsize{6}{7}\selectfont\textpm0.010} & 2.418{\fontsize{6}{7}\selectfont\textpm0.028} & 0.723{\fontsize{6}{7}\selectfont\textpm0.003} & 0.869{\fontsize{6}{7}\selectfont\textpm0.001} & 2.511{\fontsize{6}{7}\selectfont\textpm0.001} & 1.499{\fontsize{6}{7}\selectfont\textpm0.002} \\
STELLA & 4.631{\fontsize{6}{7}\selectfont\textpm0.019} & 2.754{\fontsize{6}{7}\selectfont\textpm0.077} & 2.487{\fontsize{6}{7}\selectfont\textpm0.089} & 4.042{\fontsize{6}{7}\selectfont\textpm0.269} & 1.015{\fontsize{6}{7}\selectfont\textpm0.013} & 1.266{\fontsize{6}{7}\selectfont\textpm0.009} & 2.588{\fontsize{6}{7}\selectfont\textpm0.005} & 1.474{\fontsize{6}{7}\selectfont\textpm0.002} \\
S$^2$Transformer & 2.910{\fontsize{6}{7}\selectfont\textpm0.148} & 2.062{\fontsize{6}{7}\selectfont\textpm0.071} & 1.969{\fontsize{6}{7}\selectfont\textpm0.016} & 2.940{\fontsize{6}{7}\selectfont\textpm0.162} & 0.765{\fontsize{6}{7}\selectfont\textpm0.013} & 0.978{\fontsize{6}{7}\selectfont\textpm0.015} & 2.564{\fontsize{6}{7}\selectfont\textpm0.000} & 1.465{\fontsize{6}{7}\selectfont\textpm0.001} \\
\midrule
\multicolumn{9}{l}{\textit{Multivariate Time Series Methods}} \\
DLinear & 2.598{\fontsize{6}{7}\selectfont\textpm0.000} & 1.944{\fontsize{6}{7}\selectfont\textpm0.003} & 1.936{\fontsize{6}{7}\selectfont\textpm0.002} & 2.383{\fontsize{6}{7}\selectfont\textpm0.003} & 0.714{\fontsize{6}{7}\selectfont\textpm0.000} & \underline{0.843}{\fontsize{6}{7}\selectfont\textpm0.000} & 2.644{\fontsize{6}{7}\selectfont\textpm0.000} & 1.515{\fontsize{6}{7}\selectfont\textpm0.000} \\
PatchTST & 2.599{\fontsize{6}{7}\selectfont\textpm0.006} & 1.975{\fontsize{6}{7}\selectfont\textpm0.002} & 1.952{\fontsize{6}{7}\selectfont\textpm0.001} & 2.420{\fontsize{6}{7}\selectfont\textpm0.033} & 0.716{\fontsize{6}{7}\selectfont\textpm0.001} & 0.861{\fontsize{6}{7}\selectfont\textpm0.001} & 2.559{\fontsize{6}{7}\selectfont\textpm0.000} & 1.495{\fontsize{6}{7}\selectfont\textpm0.000} \\
Crossformer & 2.633{\fontsize{6}{7}\selectfont\textpm0.002} & 1.977{\fontsize{6}{7}\selectfont\textpm0.039} & 1.956{\fontsize{6}{7}\selectfont\textpm0.019} & \underline{2.375}{\fontsize{6}{7}\selectfont\textpm0.046} & 0.713{\fontsize{6}{7}\selectfont\textpm0.001} & 0.860{\fontsize{6}{7}\selectfont\textpm0.003} & 2.505{\fontsize{6}{7}\selectfont\textpm0.011} & 1.464{\fontsize{6}{7}\selectfont\textpm0.004} \\
TimeXer & 2.597{\fontsize{6}{7}\selectfont\textpm0.008} & 1.923{\fontsize{6}{7}\selectfont\textpm0.007} & 1.944{\fontsize{6}{7}\selectfont\textpm0.007} & 2.420{\fontsize{6}{7}\selectfont\textpm0.018} & \underline{0.708}{\fontsize{6}{7}\selectfont\textpm0.002} & 0.852{\fontsize{6}{7}\selectfont\textpm0.005} & 2.531{\fontsize{6}{7}\selectfont\textpm0.001} & 1.462{\fontsize{6}{7}\selectfont\textpm0.002} \\
CCM & 2.599{\fontsize{6}{7}\selectfont\textpm0.003} & \textbf{1.912}{\fontsize{6}{7}\selectfont\textpm0.002} & 1.948{\fontsize{6}{7}\selectfont\textpm0.005} & 2.456{\fontsize{6}{7}\selectfont\textpm0.017} & 0.718{\fontsize{6}{7}\selectfont\textpm0.000} & 0.858{\fontsize{6}{7}\selectfont\textpm0.002} & 2.568{\fontsize{6}{7}\selectfont\textpm0.001} & 1.500{\fontsize{6}{7}\selectfont\textpm0.000} \\
DUET & 2.646{\fontsize{6}{7}\selectfont\textpm0.028} & 1.952{\fontsize{6}{7}\selectfont\textpm0.001} & \underline{1.934}{\fontsize{6}{7}\selectfont\textpm0.022} & 2.406{\fontsize{6}{7}\selectfont\textpm0.002} & 0.717{\fontsize{6}{7}\selectfont\textpm0.006} & 0.902{\fontsize{6}{7}\selectfont\textpm0.003} & \underline{2.502}{\fontsize{6}{7}\selectfont\textpm0.008} & \underline{1.459}{\fontsize{6}{7}\selectfont\textpm0.000} \\
TimeFilter & \underline{2.591}{\fontsize{6}{7}\selectfont\textpm0.013} & 1.956{\fontsize{6}{7}\selectfont\textpm0.008} & 1.953{\fontsize{6}{7}\selectfont\textpm0.012} & 2.376{\fontsize{6}{7}\selectfont\textpm0.004} & 0.715{\fontsize{6}{7}\selectfont\textpm0.002} & 0.862{\fontsize{6}{7}\selectfont\textpm0.015} & 2.661{\fontsize{6}{7}\selectfont\textpm0.011} & 1.472{\fontsize{6}{7}\selectfont\textpm0.008} \\
TQNet & 2.594{\fontsize{6}{7}\selectfont\textpm0.001} & 1.940{\fontsize{6}{7}\selectfont\textpm0.006} & \underline{1.934}{\fontsize{6}{7}\selectfont\textpm0.004} & 2.451{\fontsize{6}{7}\selectfont\textpm0.008} & 0.719{\fontsize{6}{7}\selectfont\textpm0.000} & 0.854{\fontsize{6}{7}\selectfont\textpm0.010} & 2.512{\fontsize{6}{7}\selectfont\textpm0.002} & 1.467{\fontsize{6}{7}\selectfont\textpm0.000} \\
xPatch & 2.596{\fontsize{6}{7}\selectfont\textpm0.002} & 1.977{\fontsize{6}{7}\selectfont\textpm0.001} & 1.952{\fontsize{6}{7}\selectfont\textpm0.001} & 2.417{\fontsize{6}{7}\selectfont\textpm0.011} & 0.716{\fontsize{6}{7}\selectfont\textpm0.000} & 0.862{\fontsize{6}{7}\selectfont\textpm0.000} & 2.582{\fontsize{6}{7}\selectfont\textpm0.001} & 1.486{\fontsize{6}{7}\selectfont\textpm0.000} \\
\midrule
\textbf{STCFormer} & \textbf{2.585}{\fontsize{6}{7}\selectfont\textpm0.006} & \underline{1.917}{\fontsize{6}{7}\selectfont\textpm0.044} & \textbf{1.925}{\fontsize{6}{7}\selectfont\textpm0.015} & \textbf{2.328}{\fontsize{6}{7}\selectfont\textpm0.100} & \textbf{0.707}{\fontsize{6}{7}\selectfont\textpm0.003} & \textbf{0.827}{\fontsize{6}{7}\selectfont\textpm0.007} & \textbf{2.493}{\fontsize{6}{7}\selectfont\textpm0.004} & \textbf{1.436}{\fontsize{6}{7}\selectfont\textpm0.001} \\
\bottomrule
\end{tabular*}
\endgroup
\end{table}
\begin{table}[!htbp]
\centering
\caption{72-hour forecasting results (MSE, $\downarrow$) with a 48-hour look-back window. Values are mean $\pm$ standard deviation over three runs. Bold and underlined means indicate the best and second-best reported values in each column; ties share the same marking.}
\label{tab:results_mse_72}
\begingroup
\fontsize{8}{9.5}\selectfont
\setlength{\tabcolsep}{0.8pt}
\renewcommand{\arraystretch}{1.10}
\begin{tabular*}{\textwidth}{@{\extracolsep{\fill}}l*{8}{c}@{}}
\toprule
\multirow{2}{*}{\textbf{Method}} & \multicolumn{3}{c}{\textbf{French}} & \multicolumn{3}{c}{\textbf{Hunan}} & \multicolumn{2}{c}{\textbf{Global}} \\
\cmidrule(lr){2-4}\cmidrule(lr){5-7}\cmidrule(lr){8-9}
 & Temp. & U-Wind & V-Wind & Temp. & U-Wind & V-Wind & Temp. & Wind \\
\midrule
\multicolumn{9}{l}{\textit{Spatio-Temporal Methods}} \\
DCRNN & 19.294{\fontsize{6}{7}\selectfont\textpm0.154} & 7.819{\fontsize{6}{7}\selectfont\textpm0.883} & 9.671{\fontsize{6}{7}\selectfont\textpm1.655} & 11.447{\fontsize{6}{7}\selectfont\textpm0.063} & 2.369{\fontsize{6}{7}\selectfont\textpm0.052} & 3.567{\fontsize{6}{7}\selectfont\textpm0.109} & 17.582{\fontsize{6}{7}\selectfont\textpm4.554} & 4.852{\fontsize{6}{7}\selectfont\textpm0.018} \\
Corrformer & 12.053{\fontsize{6}{7}\selectfont\textpm0.319} & 9.043{\fontsize{6}{7}\selectfont\textpm0.302} & 8.401{\fontsize{6}{7}\selectfont\textpm0.023} & 14.779{\fontsize{6}{7}\selectfont\textpm0.658} & 2.223{\fontsize{6}{7}\selectfont\textpm0.027} & 3.735{\fontsize{6}{7}\selectfont\textpm0.001} & 13.125{\fontsize{6}{7}\selectfont\textpm0.010} & 4.826{\fontsize{6}{7}\selectfont\textpm0.006} \\
EasyST & 13.579{\fontsize{6}{7}\selectfont\textpm0.614} & 7.361{\fontsize{6}{7}\selectfont\textpm0.078} & 7.437{\fontsize{6}{7}\selectfont\textpm0.387} & 31.283{\fontsize{6}{7}\selectfont\textpm7.987} & \underline{2.216}{\fontsize{6}{7}\selectfont\textpm0.021} & \underline{3.263}{\fontsize{6}{7}\selectfont\textpm0.074} & 12.895{\fontsize{6}{7}\selectfont\textpm0.239} & 4.892{\fontsize{6}{7}\selectfont\textpm0.017} \\
CDPNet & 13.508{\fontsize{6}{7}\selectfont\textpm0.064} & 7.838{\fontsize{6}{7}\selectfont\textpm0.022} & 8.214{\fontsize{6}{7}\selectfont\textpm0.105} & 11.485{\fontsize{6}{7}\selectfont\textpm0.351} & 2.371{\fontsize{6}{7}\selectfont\textpm0.026} & 3.494{\fontsize{6}{7}\selectfont\textpm0.001} & 12.985{\fontsize{6}{7}\selectfont\textpm0.077} & 5.011{\fontsize{6}{7}\selectfont\textpm0.008} \\
STELLA & 33.753{\fontsize{6}{7}\selectfont\textpm0.423} & 12.604{\fontsize{6}{7}\selectfont\textpm0.623} & 10.618{\fontsize{6}{7}\selectfont\textpm0.737} & 25.984{\fontsize{6}{7}\selectfont\textpm3.066} & 2.537{\fontsize{6}{7}\selectfont\textpm0.029} & 3.797{\fontsize{6}{7}\selectfont\textpm0.019} & 13.075{\fontsize{6}{7}\selectfont\textpm0.067} & 4.797{\fontsize{6}{7}\selectfont\textpm0.011} \\
S$^2$Transformer & 13.937{\fontsize{6}{7}\selectfont\textpm1.781} & 7.855{\fontsize{6}{7}\selectfont\textpm0.363} & \underline{7.197}{\fontsize{6}{7}\selectfont\textpm0.091} & 14.260{\fontsize{6}{7}\selectfont\textpm1.227} & 2.220{\fontsize{6}{7}\selectfont\textpm0.034} & 3.357{\fontsize{6}{7}\selectfont\textpm0.125} & 13.135{\fontsize{6}{7}\selectfont\textpm0.127} & \underline{4.754}{\fontsize{6}{7}\selectfont\textpm0.016} \\
\midrule
\multicolumn{9}{l}{\textit{Multivariate Time Series Methods}} \\
DLinear & 11.287{\fontsize{6}{7}\selectfont\textpm0.019} & \underline{7.214}{\fontsize{6}{7}\selectfont\textpm0.012} & 7.722{\fontsize{6}{7}\selectfont\textpm0.009} & 11.458{\fontsize{6}{7}\selectfont\textpm0.025} & 2.261{\fontsize{6}{7}\selectfont\textpm0.000} & 3.333{\fontsize{6}{7}\selectfont\textpm0.000} & 14.201{\fontsize{6}{7}\selectfont\textpm0.014} & 5.018{\fontsize{6}{7}\selectfont\textpm0.007} \\
PatchTST & 11.399{\fontsize{6}{7}\selectfont\textpm0.044} & 7.793{\fontsize{6}{7}\selectfont\textpm0.000} & 8.109{\fontsize{6}{7}\selectfont\textpm0.039} & 11.629{\fontsize{6}{7}\selectfont\textpm0.158} & 2.331{\fontsize{6}{7}\selectfont\textpm0.008} & 3.474{\fontsize{6}{7}\selectfont\textpm0.001} & 13.570{\fontsize{6}{7}\selectfont\textpm0.006} & 5.004{\fontsize{6}{7}\selectfont\textpm0.004} \\
Crossformer & 11.447{\fontsize{6}{7}\selectfont\textpm0.030} & 8.025{\fontsize{6}{7}\selectfont\textpm0.478} & 7.969{\fontsize{6}{7}\selectfont\textpm0.106} & 12.589{\fontsize{6}{7}\selectfont\textpm0.247} & 2.315{\fontsize{6}{7}\selectfont\textpm0.012} & 3.463{\fontsize{6}{7}\selectfont\textpm0.098} & 13.322{\fontsize{6}{7}\selectfont\textpm0.119} & 4.842{\fontsize{6}{7}\selectfont\textpm0.017} \\
TimeXer & 11.383{\fontsize{6}{7}\selectfont\textpm0.082} & 7.446{\fontsize{6}{7}\selectfont\textpm0.044} & 8.065{\fontsize{6}{7}\selectfont\textpm0.034} & 11.702{\fontsize{6}{7}\selectfont\textpm0.274} & 2.258{\fontsize{6}{7}\selectfont\textpm0.007} & 3.309{\fontsize{6}{7}\selectfont\textpm0.042} & 13.390{\fontsize{6}{7}\selectfont\textpm0.023} & 4.808{\fontsize{6}{7}\selectfont\textpm0.005} \\
CCM & 11.293{\fontsize{6}{7}\selectfont\textpm0.022} & 7.295{\fontsize{6}{7}\selectfont\textpm0.036} & 8.023{\fontsize{6}{7}\selectfont\textpm0.050} & 12.934{\fontsize{6}{7}\selectfont\textpm0.129} & 2.329{\fontsize{6}{7}\selectfont\textpm0.001} & 3.384{\fontsize{6}{7}\selectfont\textpm0.007} & 13.656{\fontsize{6}{7}\selectfont\textpm0.017} & 5.041{\fontsize{6}{7}\selectfont\textpm0.004} \\
DUET & 11.929{\fontsize{6}{7}\selectfont\textpm0.207} & 7.236{\fontsize{6}{7}\selectfont\textpm0.024} & 7.895{\fontsize{6}{7}\selectfont\textpm0.162} & \textbf{11.337}{\fontsize{6}{7}\selectfont\textpm0.117} & 2.222{\fontsize{6}{7}\selectfont\textpm0.020} & 3.627{\fontsize{6}{7}\selectfont\textpm0.017} & \underline{12.633}{\fontsize{6}{7}\selectfont\textpm0.107} & 4.904{\fontsize{6}{7}\selectfont\textpm0.069} \\
TimeFilter & 11.415{\fontsize{6}{7}\selectfont\textpm0.186} & 7.666{\fontsize{6}{7}\selectfont\textpm0.116} & 8.086{\fontsize{6}{7}\selectfont\textpm0.112} & \underline{11.365}{\fontsize{6}{7}\selectfont\textpm0.093} & 2.303{\fontsize{6}{7}\selectfont\textpm0.016} & 3.390{\fontsize{6}{7}\selectfont\textpm0.129} & 13.287{\fontsize{6}{7}\selectfont\textpm0.043} & 4.808{\fontsize{6}{7}\selectfont\textpm0.019} \\
TQNet & 11.430{\fontsize{6}{7}\selectfont\textpm0.011} & 7.408{\fontsize{6}{7}\selectfont\textpm0.030} & 7.980{\fontsize{6}{7}\selectfont\textpm0.060} & 11.377{\fontsize{6}{7}\selectfont\textpm0.062} & 2.317{\fontsize{6}{7}\selectfont\textpm0.003} & 3.336{\fontsize{6}{7}\selectfont\textpm0.080} & 13.121{\fontsize{6}{7}\selectfont\textpm0.014} & 4.872{\fontsize{6}{7}\selectfont\textpm0.001} \\
xPatch & \underline{11.285}{\fontsize{6}{7}\selectfont\textpm0.003} & 7.696{\fontsize{6}{7}\selectfont\textpm0.004} & 8.094{\fontsize{6}{7}\selectfont\textpm0.005} & 11.709{\fontsize{6}{7}\selectfont\textpm0.140} & 2.326{\fontsize{6}{7}\selectfont\textpm0.001} & 3.469{\fontsize{6}{7}\selectfont\textpm0.006} & 13.803{\fontsize{6}{7}\selectfont\textpm0.007} & 4.952{\fontsize{6}{7}\selectfont\textpm0.003} \\
\midrule
\textbf{STCFormer} & \textbf{11.145}{\fontsize{6}{7}\selectfont\textpm0.015} & \textbf{7.120}{\fontsize{6}{7}\selectfont\textpm0.454} & \textbf{7.043}{\fontsize{6}{7}\selectfont\textpm0.065} & 11.432{\fontsize{6}{7}\selectfont\textpm0.981} & \textbf{2.208}{\fontsize{6}{7}\selectfont\textpm0.013} & \textbf{3.249}{\fontsize{6}{7}\selectfont\textpm0.051} & \textbf{12.531}{\fontsize{6}{7}\selectfont\textpm0.074} & \textbf{4.750}{\fontsize{6}{7}\selectfont\textpm0.003} \\
\bottomrule
\end{tabular*}
\endgroup
\end{table}
\begin{table}[!htbp]
\centering
\caption{72-hour forecasting results (SEDI, $\uparrow$) with a 48-hour look-back window. Values are mean $\pm$ standard deviation over three runs. Bold and underlined means indicate the best and second-best reported values in each column; ties share the same marking.}
\label{tab:results_sedi_72}
\begingroup
\fontsize{8}{9.5}\selectfont
\setlength{\tabcolsep}{0.8pt}
\renewcommand{\arraystretch}{1.10}
\begin{tabular*}{\textwidth}{@{\extracolsep{\fill}}l*{8}{c}@{}}
\toprule
\multirow{2}{*}{\textbf{Method}} & \multicolumn{3}{c}{\textbf{French}} & \multicolumn{3}{c}{\textbf{Hunan}} & \multicolumn{2}{c}{\textbf{Global}} \\
\cmidrule(lr){2-4}\cmidrule(lr){5-7}\cmidrule(lr){8-9}
 & Temp. & U-Wind & V-Wind & Temp. & U-Wind & V-Wind & Temp. & Wind \\
\midrule
\multicolumn{9}{l}{\textit{Spatio-Temporal Methods}} \\
DCRNN & 0.394{\fontsize{6}{7}\selectfont\textpm0.173} & 0.456{\fontsize{6}{7}\selectfont\textpm0.004} & 0.406{\fontsize{6}{7}\selectfont\textpm0.010} & 0.548{\fontsize{6}{7}\selectfont\textpm0.025} & 0.416{\fontsize{6}{7}\selectfont\textpm0.006} & 0.441{\fontsize{6}{7}\selectfont\textpm0.025} & 0.867{\fontsize{6}{7}\selectfont\textpm0.010} & 0.521{\fontsize{6}{7}\selectfont\textpm0.003} \\
Corrformer & 0.342{\fontsize{6}{7}\selectfont\textpm0.015} & 0.472{\fontsize{6}{7}\selectfont\textpm0.010} & 0.370{\fontsize{6}{7}\selectfont\textpm0.014} & 0.328{\fontsize{6}{7}\selectfont\textpm0.046} & 0.481{\fontsize{6}{7}\selectfont\textpm0.007} & 0.473{\fontsize{6}{7}\selectfont\textpm0.016} & 0.849{\fontsize{6}{7}\selectfont\textpm0.004} & 0.450{\fontsize{6}{7}\selectfont\textpm0.003} \\
EasyST & 0.499{\fontsize{6}{7}\selectfont\textpm0.003} & 0.463{\fontsize{6}{7}\selectfont\textpm0.003} & 0.371{\fontsize{6}{7}\selectfont\textpm0.010} & 0.447{\fontsize{6}{7}\selectfont\textpm0.034} & 0.424{\fontsize{6}{7}\selectfont\textpm0.009} & 0.463{\fontsize{6}{7}\selectfont\textpm0.005} & 0.874{\fontsize{6}{7}\selectfont\textpm0.001} & 0.494{\fontsize{6}{7}\selectfont\textpm0.002} \\
CDPNet & 0.498{\fontsize{6}{7}\selectfont\textpm0.002} & 0.513{\fontsize{6}{7}\selectfont\textpm0.003} & 0.395{\fontsize{6}{7}\selectfont\textpm0.000} & \underline{0.550}{\fontsize{6}{7}\selectfont\textpm0.006} & 0.456{\fontsize{6}{7}\selectfont\textpm0.004} & 0.460{\fontsize{6}{7}\selectfont\textpm0.000} & 0.880{\fontsize{6}{7}\selectfont\textpm0.005} & 0.520{\fontsize{6}{7}\selectfont\textpm0.002} \\
STELLA & 0.277{\fontsize{6}{7}\selectfont\textpm0.014} & 0.397{\fontsize{6}{7}\selectfont\textpm0.012} & 0.284{\fontsize{6}{7}\selectfont\textpm0.028} & 0.349{\fontsize{6}{7}\selectfont\textpm0.029} & 0.381{\fontsize{6}{7}\selectfont\textpm0.006} & 0.402{\fontsize{6}{7}\selectfont\textpm0.010} & 0.865{\fontsize{6}{7}\selectfont\textpm0.004} & 0.455{\fontsize{6}{7}\selectfont\textpm0.002} \\
S$^2$Transformer & 0.304{\fontsize{6}{7}\selectfont\textpm0.084} & 0.419{\fontsize{6}{7}\selectfont\textpm0.001} & 0.353{\fontsize{6}{7}\selectfont\textpm0.014} & 0.376{\fontsize{6}{7}\selectfont\textpm0.046} & 0.412{\fontsize{6}{7}\selectfont\textpm0.020} & 0.449{\fontsize{6}{7}\selectfont\textpm0.037} & 0.871{\fontsize{6}{7}\selectfont\textpm0.013} & 0.477{\fontsize{6}{7}\selectfont\textpm0.004} \\
\midrule
\multicolumn{9}{l}{\textit{Multivariate Time Series Methods}} \\
DLinear & 0.504{\fontsize{6}{7}\selectfont\textpm0.005} & 0.462{\fontsize{6}{7}\selectfont\textpm0.000} & 0.386{\fontsize{6}{7}\selectfont\textpm0.000} & 0.509{\fontsize{6}{7}\selectfont\textpm0.001} & 0.400{\fontsize{6}{7}\selectfont\textpm0.000} & 0.428{\fontsize{6}{7}\selectfont\textpm0.000} & 0.858{\fontsize{6}{7}\selectfont\textpm0.001} & 0.480{\fontsize{6}{7}\selectfont\textpm0.001} \\
PatchTST & 0.486{\fontsize{6}{7}\selectfont\textpm0.011} & 0.494{\fontsize{6}{7}\selectfont\textpm0.002} & 0.406{\fontsize{6}{7}\selectfont\textpm0.005} & 0.486{\fontsize{6}{7}\selectfont\textpm0.010} & 0.478{\fontsize{6}{7}\selectfont\textpm0.005} & 0.463{\fontsize{6}{7}\selectfont\textpm0.001} & 0.882{\fontsize{6}{7}\selectfont\textpm0.000} & 0.503{\fontsize{6}{7}\selectfont\textpm0.000} \\
Crossformer & -0.270{\fontsize{6}{7}\selectfont\textpm0.005} & \textbf{0.521}{\fontsize{6}{7}\selectfont\textpm0.009} & 0.425{\fontsize{6}{7}\selectfont\textpm0.019} & 0.504{\fontsize{6}{7}\selectfont\textpm0.027} & 0.467{\fontsize{6}{7}\selectfont\textpm0.004} & 0.476{\fontsize{6}{7}\selectfont\textpm0.017} & 0.878{\fontsize{6}{7}\selectfont\textpm0.002} & 0.511{\fontsize{6}{7}\selectfont\textpm0.006} \\
TimeXer & 0.470{\fontsize{6}{7}\selectfont\textpm0.005} & 0.516{\fontsize{6}{7}\selectfont\textpm0.012} & 0.424{\fontsize{6}{7}\selectfont\textpm0.005} & 0.479{\fontsize{6}{7}\selectfont\textpm0.012} & 0.484{\fontsize{6}{7}\selectfont\textpm0.009} & \textbf{0.481}{\fontsize{6}{7}\selectfont\textpm0.002} & \underline{0.886}{\fontsize{6}{7}\selectfont\textpm0.001} & 0.526{\fontsize{6}{7}\selectfont\textpm0.001} \\
CCM & 0.510{\fontsize{6}{7}\selectfont\textpm0.001} & 0.519{\fontsize{6}{7}\selectfont\textpm0.001} & \underline{0.434}{\fontsize{6}{7}\selectfont\textpm0.006} & 0.527{\fontsize{6}{7}\selectfont\textpm0.003} & \underline{0.485}{\fontsize{6}{7}\selectfont\textpm0.002} & 0.468{\fontsize{6}{7}\selectfont\textpm0.001} & 0.880{\fontsize{6}{7}\selectfont\textpm0.001} & 0.506{\fontsize{6}{7}\selectfont\textpm0.000} \\
DUET & 0.473{\fontsize{6}{7}\selectfont\textpm0.004} & 0.458{\fontsize{6}{7}\selectfont\textpm0.004} & 0.423{\fontsize{6}{7}\selectfont\textpm0.002} & 0.487{\fontsize{6}{7}\selectfont\textpm0.008} & 0.444{\fontsize{6}{7}\selectfont\textpm0.005} & 0.467{\fontsize{6}{7}\selectfont\textpm0.010} & 0.875{\fontsize{6}{7}\selectfont\textpm0.010} & 0.488{\fontsize{6}{7}\selectfont\textpm0.004} \\
TimeFilter & 0.515{\fontsize{6}{7}\selectfont\textpm0.002} & 0.504{\fontsize{6}{7}\selectfont\textpm0.006} & 0.403{\fontsize{6}{7}\selectfont\textpm0.022} & 0.538{\fontsize{6}{7}\selectfont\textpm0.010} & 0.478{\fontsize{6}{7}\selectfont\textpm0.000} & 0.479{\fontsize{6}{7}\selectfont\textpm0.002} & 0.846{\fontsize{6}{7}\selectfont\textpm0.002} & \underline{0.527}{\fontsize{6}{7}\selectfont\textpm0.002} \\
TQNet & \underline{0.518}{\fontsize{6}{7}\selectfont\textpm0.007} & 0.513{\fontsize{6}{7}\selectfont\textpm0.002} & 0.425{\fontsize{6}{7}\selectfont\textpm0.004} & 0.513{\fontsize{6}{7}\selectfont\textpm0.001} & 0.484{\fontsize{6}{7}\selectfont\textpm0.001} & 0.464{\fontsize{6}{7}\selectfont\textpm0.004} & \underline{0.886}{\fontsize{6}{7}\selectfont\textpm0.001} & 0.524{\fontsize{6}{7}\selectfont\textpm0.001} \\
xPatch & 0.514{\fontsize{6}{7}\selectfont\textpm0.004} & 0.502{\fontsize{6}{7}\selectfont\textpm0.001} & 0.415{\fontsize{6}{7}\selectfont\textpm0.003} & 0.494{\fontsize{6}{7}\selectfont\textpm0.004} & 0.476{\fontsize{6}{7}\selectfont\textpm0.000} & 0.465{\fontsize{6}{7}\selectfont\textpm0.000} & 0.883{\fontsize{6}{7}\selectfont\textpm0.000} & 0.521{\fontsize{6}{7}\selectfont\textpm0.001} \\
\midrule
\textbf{STCFormer} & \textbf{0.527}{\fontsize{6}{7}\selectfont\textpm0.008} & \underline{0.520}{\fontsize{6}{7}\selectfont\textpm0.009} & \textbf{0.438}{\fontsize{6}{7}\selectfont\textpm0.009} & \textbf{0.581}{\fontsize{6}{7}\selectfont\textpm0.027} & \textbf{0.496}{\fontsize{6}{7}\selectfont\textpm0.007} & \underline{0.480}{\fontsize{6}{7}\selectfont\textpm0.018} & \textbf{0.889}{\fontsize{6}{7}\selectfont\textpm0.001} & \textbf{0.530}{\fontsize{6}{7}\selectfont\textpm0.000} \\
\bottomrule
\end{tabular*}
\endgroup
\end{table}
\FloatBarrier

\section{Additional Ablation Results}
\label{sec:additional_ablation}
Tables~\ref{tab:ablation_mae_24} and~\ref{tab:ablation_sedi_24} report the MAE and SEDI results corresponding to the 24-hour MSE ablation in Table~\ref{tab:ablation_full}. The full model attains the best reported mean on all eight tasks for both metrics, consistent with the MSE comparison. The variants and all other experimental settings follow the main ablation study.

\begin{table}[!htbp]
\centering
\caption{24-hour ablation results (MAE, $\downarrow$) with a 48-hour look-back window. Values are mean $\pm$ standard deviation over three runs. Bold means indicate the best reported values in each column; ties are both bold.}
\label{tab:ablation_mae_24}
\begingroup
\fontsize{8}{9.5}\selectfont
\setlength{\tabcolsep}{0.8pt}
\renewcommand{\arraystretch}{1.10}
\begin{tabular*}{\textwidth}{@{\extracolsep{\fill}}l*{8}{c}@{}}
\toprule
\multirow{2}{*}{\textbf{Variant}} & \multicolumn{3}{c}{\textbf{French}} & \multicolumn{3}{c}{\textbf{Hunan}} & \multicolumn{2}{c}{\textbf{Global}} \\
\cmidrule(lr){2-4}\cmidrule(lr){5-7}\cmidrule(lr){8-9}
 & Temp. & U-Wind & V-Wind & Temp. & U-Wind & V-Wind & Temp. & Wind \\
\midrule
Window-shared & 1.885{\fontsize{6}{7}\selectfont\textpm0.015} & 1.411{\fontsize{6}{7}\selectfont\textpm0.008} & 1.497{\fontsize{6}{7}\selectfont\textpm0.016} & 1.522{\fontsize{6}{7}\selectfont\textpm0.020} & 0.652{\fontsize{6}{7}\selectfont\textpm0.008} & 0.721{\fontsize{6}{7}\selectfont\textpm0.006} & 1.766{\fontsize{6}{7}\selectfont\textpm0.003} & 1.325{\fontsize{6}{7}\selectfont\textpm0.010} \\
Full Attention & 1.869{\fontsize{6}{7}\selectfont\textpm0.010} & 1.412{\fontsize{6}{7}\selectfont\textpm0.019} & 1.490{\fontsize{6}{7}\selectfont\textpm0.032} & 1.518{\fontsize{6}{7}\selectfont\textpm0.023} & 0.649{\fontsize{6}{7}\selectfont\textpm0.000} & 0.722{\fontsize{6}{7}\selectfont\textpm0.001} & 1.766{\fontsize{6}{7}\selectfont\textpm0.001} & 1.298{\fontsize{6}{7}\selectfont\textpm0.001} \\
Static Clustering & 1.886{\fontsize{6}{7}\selectfont\textpm0.023} & 1.414{\fontsize{6}{7}\selectfont\textpm0.013} & 1.492{\fontsize{6}{7}\selectfont\textpm0.009} & 1.508{\fontsize{6}{7}\selectfont\textpm0.043} & 0.649{\fontsize{6}{7}\selectfont\textpm0.002} & 0.725{\fontsize{6}{7}\selectfont\textpm0.002} & 1.764{\fontsize{6}{7}\selectfont\textpm0.003} & 1.298{\fontsize{6}{7}\selectfont\textpm0.001} \\
w/o InfoLoss & 1.906{\fontsize{6}{7}\selectfont\textpm0.046} & 1.412{\fontsize{6}{7}\selectfont\textpm0.009} & 1.479{\fontsize{6}{7}\selectfont\textpm0.005} & 1.549{\fontsize{6}{7}\selectfont\textpm0.009} & 0.659{\fontsize{6}{7}\selectfont\textpm0.002} & 0.719{\fontsize{6}{7}\selectfont\textpm0.001} & 1.767{\fontsize{6}{7}\selectfont\textpm0.000} & 1.295{\fontsize{6}{7}\selectfont\textpm0.001} \\
w/o Global Attn. & 1.877{\fontsize{6}{7}\selectfont\textpm0.009} & 1.410{\fontsize{6}{7}\selectfont\textpm0.003} & 1.477{\fontsize{6}{7}\selectfont\textpm0.011} & 1.572{\fontsize{6}{7}\selectfont\textpm0.009} & 0.653{\fontsize{6}{7}\selectfont\textpm0.000} & 0.728{\fontsize{6}{7}\selectfont\textpm0.002} & 1.773{\fontsize{6}{7}\selectfont\textpm0.001} & 1.294{\fontsize{6}{7}\selectfont\textpm0.001} \\
w/o Local Attn. & 1.887{\fontsize{6}{7}\selectfont\textpm0.006} & 1.418{\fontsize{6}{7}\selectfont\textpm0.019} & 1.474{\fontsize{6}{7}\selectfont\textpm0.001} & 1.491{\fontsize{6}{7}\selectfont\textpm0.022} & 0.650{\fontsize{6}{7}\selectfont\textpm0.000} & 0.724{\fontsize{6}{7}\selectfont\textpm0.001} & 1.783{\fontsize{6}{7}\selectfont\textpm0.001} & 1.301{\fontsize{6}{7}\selectfont\textpm0.002} \\
\midrule
\textbf{STCFormer} & \textbf{1.861}{\fontsize{6}{7}\selectfont\textpm0.015} & \textbf{1.400}{\fontsize{6}{7}\selectfont\textpm0.008} & \textbf{1.466}{\fontsize{6}{7}\selectfont\textpm0.004} & \textbf{1.482}{\fontsize{6}{7}\selectfont\textpm0.026} & \textbf{0.622}{\fontsize{6}{7}\selectfont\textpm0.001} & \textbf{0.700}{\fontsize{6}{7}\selectfont\textpm0.002} & \textbf{1.757}{\fontsize{6}{7}\selectfont\textpm0.000} & \textbf{1.273}{\fontsize{6}{7}\selectfont\textpm0.001} \\
\bottomrule
\end{tabular*}
\endgroup
\end{table}
\begin{table}[!htbp]
\centering
\caption{24-hour ablation results (SEDI, $\uparrow$) with a 48-hour look-back window. Values are mean $\pm$ standard deviation over three runs. Bold means indicate the best reported values in each column; ties are both bold.}
\label{tab:ablation_sedi_24}
\begingroup
\fontsize{8}{9.5}\selectfont
\setlength{\tabcolsep}{0.8pt}
\renewcommand{\arraystretch}{1.10}
\begin{tabular*}{\textwidth}{@{\extracolsep{\fill}}l*{8}{c}@{}}
\toprule
\multirow{2}{*}{\textbf{Variant}} & \multicolumn{3}{c}{\textbf{French}} & \multicolumn{3}{c}{\textbf{Hunan}} & \multicolumn{2}{c}{\textbf{Global}} \\
\cmidrule(lr){2-4}\cmidrule(lr){5-7}\cmidrule(lr){8-9}
 & Temp. & U-Wind & V-Wind & Temp. & U-Wind & V-Wind & Temp. & Wind \\
\midrule
Window-shared & 0.665{\fontsize{6}{7}\selectfont\textpm0.001} & 0.664{\fontsize{6}{7}\selectfont\textpm0.000} & 0.600{\fontsize{6}{7}\selectfont\textpm0.018} & 0.718{\fontsize{6}{7}\selectfont\textpm0.013} & 0.544{\fontsize{6}{7}\selectfont\textpm0.005} & 0.611{\fontsize{6}{7}\selectfont\textpm0.006} & 0.910{\fontsize{6}{7}\selectfont\textpm0.001} & 0.638{\fontsize{6}{7}\selectfont\textpm0.001} \\
Full Attention & 0.670{\fontsize{6}{7}\selectfont\textpm0.004} & 0.667{\fontsize{6}{7}\selectfont\textpm0.003} & 0.611{\fontsize{6}{7}\selectfont\textpm0.000} & 0.723{\fontsize{6}{7}\selectfont\textpm0.015} & 0.546{\fontsize{6}{7}\selectfont\textpm0.003} & 0.628{\fontsize{6}{7}\selectfont\textpm0.006} & 0.909{\fontsize{6}{7}\selectfont\textpm0.000} & 0.633{\fontsize{6}{7}\selectfont\textpm0.000} \\
Static Clustering & 0.668{\fontsize{6}{7}\selectfont\textpm0.013} & 0.656{\fontsize{6}{7}\selectfont\textpm0.004} & 0.609{\fontsize{6}{7}\selectfont\textpm0.031} & 0.724{\fontsize{6}{7}\selectfont\textpm0.036} & 0.546{\fontsize{6}{7}\selectfont\textpm0.001} & 0.625{\fontsize{6}{7}\selectfont\textpm0.009} & 0.914{\fontsize{6}{7}\selectfont\textpm0.002} & 0.649{\fontsize{6}{7}\selectfont\textpm0.001} \\
w/o InfoLoss & 0.670{\fontsize{6}{7}\selectfont\textpm0.003} & 0.655{\fontsize{6}{7}\selectfont\textpm0.002} & 0.596{\fontsize{6}{7}\selectfont\textpm0.007} & 0.713{\fontsize{6}{7}\selectfont\textpm0.010} & 0.553{\fontsize{6}{7}\selectfont\textpm0.007} & 0.632{\fontsize{6}{7}\selectfont\textpm0.004} & 0.911{\fontsize{6}{7}\selectfont\textpm0.000} & 0.645{\fontsize{6}{7}\selectfont\textpm0.000} \\
w/o Global Attn. & 0.665{\fontsize{6}{7}\selectfont\textpm0.000} & 0.658{\fontsize{6}{7}\selectfont\textpm0.006} & 0.607{\fontsize{6}{7}\selectfont\textpm0.017} & 0.712{\fontsize{6}{7}\selectfont\textpm0.017} & 0.553{\fontsize{6}{7}\selectfont\textpm0.001} & 0.627{\fontsize{6}{7}\selectfont\textpm0.013} & 0.912{\fontsize{6}{7}\selectfont\textpm0.001} & 0.646{\fontsize{6}{7}\selectfont\textpm0.001} \\
w/o Local Attn. & 0.661{\fontsize{6}{7}\selectfont\textpm0.003} & 0.661{\fontsize{6}{7}\selectfont\textpm0.009} & 0.598{\fontsize{6}{7}\selectfont\textpm0.011} & 0.714{\fontsize{6}{7}\selectfont\textpm0.020} & 0.557{\fontsize{6}{7}\selectfont\textpm0.003} & 0.630{\fontsize{6}{7}\selectfont\textpm0.006} & 0.909{\fontsize{6}{7}\selectfont\textpm0.000} & 0.649{\fontsize{6}{7}\selectfont\textpm0.000} \\
\midrule
\textbf{STCFormer} & \textbf{0.672}{\fontsize{6}{7}\selectfont\textpm0.001} & \textbf{0.689}{\fontsize{6}{7}\selectfont\textpm0.007} & \textbf{0.615}{\fontsize{6}{7}\selectfont\textpm0.006} & \textbf{0.751}{\fontsize{6}{7}\selectfont\textpm0.009} & \textbf{0.580}{\fontsize{6}{7}\selectfont\textpm0.001} & \textbf{0.639}{\fontsize{6}{7}\selectfont\textpm0.017} & \textbf{0.918}{\fontsize{6}{7}\selectfont\textpm0.002} & \textbf{0.653}{\fontsize{6}{7}\selectfont\textpm0.001} \\
\bottomrule
\end{tabular*}
\endgroup
\end{table}
\FloatBarrier

% Additional comparisons: Feishu experiment plan, revision 1066.
\section{Additional Analyses}
\label{app:additional_comparisons}
We extend the comparison to three time series foundation models and three grid-based predictive models. Table~\ref{tab:additional_foundation_grid} summarizes 24-hour MSE on all eight dataset--variable tasks using a 48-hour input window. STCFormer results are those of the main experiment. We report mean and standard deviation, with the best mean in each column in bold.

\begin{table}[!htbp]
\centering
\caption{Additional 24-hour forecasting comparisons (MSE, $\downarrow$) with a 48-hour input window. Values are mean $\pm$ standard deviation over three runs. Bold means indicate the best reported values in each column; ties are both bold.}
\label{tab:additional_foundation_grid}
\begingroup
\fontsize{8}{9.5}\selectfont
\setlength{\tabcolsep}{0.7pt}
\renewcommand{\arraystretch}{1.12}
\begin{tabular*}{\textwidth}{@{\extracolsep{\fill}}l*{8}{c}@{}}
\toprule
\multirow{2}{*}{Method} & \multicolumn{3}{c}{French} & \multicolumn{3}{c}{Hunan} & \multicolumn{2}{c}{Global} \\
\cmidrule(lr){2-4}\cmidrule(lr){5-7}\cmidrule(lr){8-9}
 & Temp. & U-Wind & V-Wind & Temp. & U-Wind & V-Wind & Temp. & Wind \\
\midrule
\multicolumn{9}{l}{\textit{Time Series Foundation Models}} \\
Timer & 7.304{\fontsize{6}{7}\selectfont\textpm0.014} & 4.831{\fontsize{6}{7}\selectfont\textpm0.008} & 5.694{\fontsize{6}{7}\selectfont\textpm0.033} & 5.041{\fontsize{6}{7}\selectfont\textpm0.012} & 1.987{\fontsize{6}{7}\selectfont\textpm0.008} & 2.769{\fontsize{6}{7}\selectfont\textpm0.020} & 8.002{\fontsize{6}{7}\selectfont\textpm0.008} & 4.873{\fontsize{6}{7}\selectfont\textpm0.015} \\
Chronos-Bolt-Tiny & 6.706{\fontsize{6}{7}\selectfont\textpm0.012} & 4.538{\fontsize{6}{7}\selectfont\textpm0.002} & 5.437{\fontsize{6}{7}\selectfont\textpm0.002} & 4.990{\fontsize{6}{7}\selectfont\textpm0.003} & 1.924{\fontsize{6}{7}\selectfont\textpm0.003} & 2.537{\fontsize{6}{7}\selectfont\textpm0.000} & 7.894{\fontsize{6}{7}\selectfont\textpm0.004} & 4.691{\fontsize{6}{7}\selectfont\textpm0.005} \\
Time-MoE & 6.427{\fontsize{6}{7}\selectfont\textpm0.014} & 4.231{\fontsize{6}{7}\selectfont\textpm0.017} & 5.026{\fontsize{6}{7}\selectfont\textpm0.009} & 5.138{\fontsize{6}{7}\selectfont\textpm0.008} & 1.947{\fontsize{6}{7}\selectfont\textpm0.011} & 2.429{\fontsize{6}{7}\selectfont\textpm0.014} & 7.267{\fontsize{6}{7}\selectfont\textpm0.007} & 3.798{\fontsize{6}{7}\selectfont\textpm0.006} \\
\midrule
\multicolumn{9}{l}{\textit{Grid-Based Predictive Models}} \\
PhyDNet & 10.225{\fontsize{6}{7}\selectfont\textpm0.487} & 5.224{\fontsize{6}{7}\selectfont\textpm0.278} & 6.259{\fontsize{6}{7}\selectfont\textpm0.106} & 6.028{\fontsize{6}{7}\selectfont\textpm0.050} & 2.315{\fontsize{6}{7}\selectfont\textpm0.026} & 3.192{\fontsize{6}{7}\selectfont\textpm0.054} & 22.217{\fontsize{6}{7}\selectfont\textpm0.422} & 6.830{\fontsize{6}{7}\selectfont\textpm0.005} \\
PredRNN & 8.793{\fontsize{6}{7}\selectfont\textpm0.409} & 4.278{\fontsize{6}{7}\selectfont\textpm0.069} & 5.276{\fontsize{6}{7}\selectfont\textpm0.236} & 5.359{\fontsize{6}{7}\selectfont\textpm0.057} & 2.313{\fontsize{6}{7}\selectfont\textpm0.004} & 3.276{\fontsize{6}{7}\selectfont\textpm0.065} & 20.139{\fontsize{6}{7}\selectfont\textpm0.169} & 6.574{\fontsize{6}{7}\selectfont\textpm0.007} \\
SimVP & 9.472{\fontsize{6}{7}\selectfont\textpm0.217} & 4.862{\fontsize{6}{7}\selectfont\textpm0.042} & 5.560{\fontsize{6}{7}\selectfont\textpm0.314} & 7.274{\fontsize{6}{7}\selectfont\textpm0.595} & 2.508{\fontsize{6}{7}\selectfont\textpm0.028} & 3.553{\fontsize{6}{7}\selectfont\textpm0.238} & 20.372{\fontsize{6}{7}\selectfont\textpm0.077} & 6.687{\fontsize{6}{7}\selectfont\textpm0.096} \\
\midrule
STCFormer & \textbf{6.291}{\fontsize{6}{7}\selectfont\textpm0.090} & \textbf{4.045}{\fontsize{6}{7}\selectfont\textpm0.028} & \textbf{4.883}{\fontsize{6}{7}\selectfont\textpm0.013} & \textbf{4.743}{\fontsize{6}{7}\selectfont\textpm0.149} & \textbf{1.832}{\fontsize{6}{7}\selectfont\textpm0.001} & \textbf{2.402}{\fontsize{6}{7}\selectfont\textpm0.003} & \textbf{6.917}{\fontsize{6}{7}\selectfont\textpm0.002} & \textbf{3.584}{\fontsize{6}{7}\selectfont\textpm0.002} \\
\bottomrule
\end{tabular*}
\endgroup
\end{table}

\subsection{Comparison with Time Series Foundation Models}
\label{app:foundation_models}
We evaluate Timer~\citep{liu2024timer}, Chronos-Bolt-Tiny~\citep{ansari2024chronos}, and Time-MoE~\citep{shi2025timemoe}. These models undergo full-parameter fine-tuning for five epochs on French and three epochs on Hunan and Global. STCFormer obtains lower mean MSE on all eight tasks, with reductions of 1.1\%--5.6\% relative to the strongest foundation-model baseline on each task. Time-MoE is the strongest of these baselines on six tasks, whereas Chronos-Bolt-Tiny performs best on Hunan temperature and U-wind. These results support the effectiveness of STCFormer under the evaluated fine-tuning budgets.

\subsection{Comparison with Grid-Based Predictive Models}
\label{app:grid_weather_models}
We additionally evaluate PhyDNet~\citep{DBLP:conf/cvpr/GuenT20a}, PredRNN~\citep{DBLP:journals/pami/WangWZGWYL23}, and SimVP~\citep{DBLP:conf/cvpr/Gao0WL22}, each trained for 30 epochs. Station observations are interpolated onto a grid, forecasts are generated on that grid, and predictions are mapped back to station locations for evaluation. STCFormer achieves the lowest mean MSE on all eight tasks. On Global temperature and wind, it reduces MSE from the strongest grid baseline's 20.139 and 6.574 to 6.917 and 3.584, respectively. Although interpolation can introduce station--grid conversion error, the consistent margin across all eight tasks highlights both the necessity and effectiveness of forecasting directly on irregular station networks.

\subsection{Additional Hyperparameter Studies}
\label{app:additional_hyperparameters}
We complement the InfoLoss-weight, cluster-ratio, and embedding-dimension studies in Section~\ref{sec:hyperparameter_sensitivity} with look-back length and spatio-temporal patch size. We report test MSE for temperature forecasting: the look-back study forecasts 72 hours on French temperature, while the patch-size study forecasts 24 hours. Each plotted value is the mean of three runs; error bars are omitted for compactness.

\subsubsection{Look-Back Window Length}
Figure~\ref{fig:lookback_sensitivity} compares look-back windows of 24, 48, 96, and 168 hours for 72-hour French temperature forecasting. STCFormer attains MSE values of 11.498, 11.245, 11.318, and 11.501, respectively, with its lowest error at 48 hours. It achieves the lowest reported MSE among the compared methods at 24- and 48-hour look-back lengths; CCM performs best at 96 and 168 hours. The results support the 48-hour reference window, while showing that longer histories do not consistently improve performance. CCM uses DLinear as its forecasting backbone, as in the main comparison.

\begin{figure}[htbp]
\centering
\includegraphics[width=0.70\linewidth]{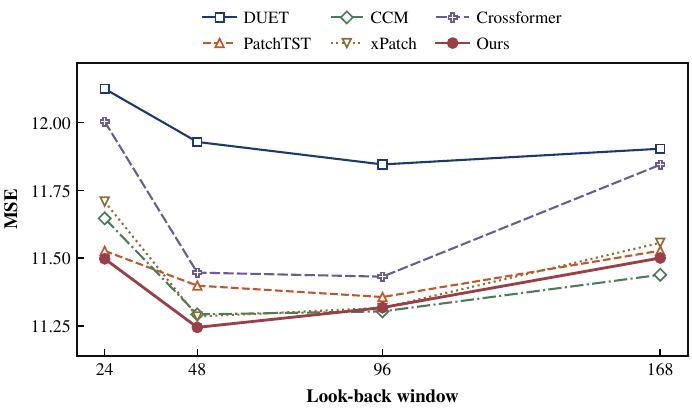}
\caption{Look-back sensitivity on French temperature with a fixed 72-hour forecasting horizon. The horizontal axis is the input-window length in hours.}
\label{fig:lookback_sensitivity}
\end{figure}

\subsubsection{Spatio-Temporal Patch Sizes}
We jointly vary the temporal patch length $L_t\in\{1,2,4,6\}$ and the spatial patch length $L_s$, using $L_s\in\{1,2,3,6\}$ on French and $L_s\in\{1,7,10,14\}$ on Hunan. Figure~\ref{fig:patch_size_sensitivity} presents the resulting 24-hour MSE. The French sweep favors $(L_t,L_s)=(6,2)$, while the main experiments use $(4,2)$. Hunan reaches 4.742 at $(2,10)$, almost identical to 4.743 at $(4,7)$. With $L_t=4$, replacing $L_s=1$ by 2 on French and 7 on Hunan reduces MSE from 6.553 to 6.291 and from 5.163 to 4.743, respectively. Joint spatial and temporal patching is thus useful in these settings, but the best combination is dataset-dependent and coarser patches do not uniformly improve performance.

\begin{figure}[htbp]
\centering
\begin{subfigure}[t]{0.49\linewidth}
\centering
\includegraphics[width=\linewidth]{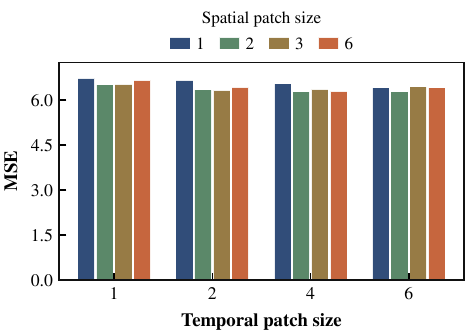}
\caption{French temperature.}
\end{subfigure}\hfill
\begin{subfigure}[t]{0.49\linewidth}
\centering
\includegraphics[width=\linewidth]{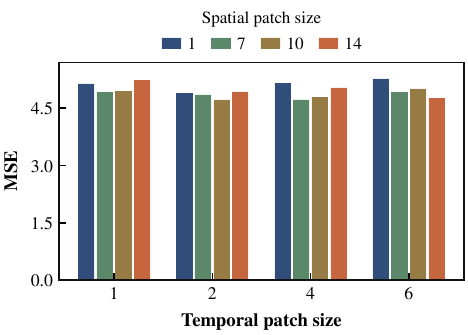}
\caption{Hunan temperature.}
\end{subfigure}
\caption{Joint sensitivity to temporal and spatial patch sizes (24-hour MSE). Each bar represents one $(L_t,L_s)$ configuration; colors distinguish spatial patch lengths.}
\label{fig:patch_size_sensitivity}
\end{figure}
\FloatBarrier

\subsection{Additional Efficiency Analysis}
\label{app:additional_efficiency}
We complement the Hunan comparison in Section~\ref{sec:efficiency_analysis} with French temperature forecasting, using a 48-hour input window, a 24-hour horizon, and batch size 8 for every method. Figure~\ref{fig:efficiency_french} plots mean training time per epoch against test MSE, with bubble area proportional to peak training GPU memory. The time axis is logarithmic to accommodate the wide runtime range.

Among the methods shown, STCFormer achieves the lowest reported MSE (6.291), taking $15.306\pm0.003$\,s/epoch and 0.378\,GiB. It is $18.1\times$ faster than DCRNN; compared with TimeFilter, it uses 28.3\% less training time and 86.8\% less memory. PatchTST and TimeXer train faster, whereas DUET uses less memory. These results demonstrate a favorable accuracy--efficiency trade-off across methods with substantially different computational profiles.

\begin{figure}[!htbp]
\centering
\includegraphics[width=0.75\linewidth]{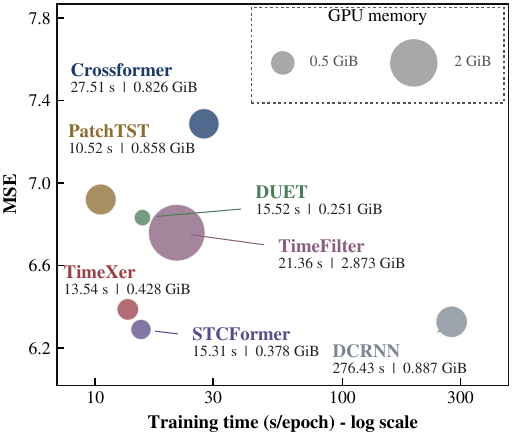}
\caption{Efficiency on French temperature (48-to-24-hour forecasting, batch size 8). Bubble area indicates peak training GPU memory; labels report time per epoch and memory. The time axis uses a logarithmic scale.}
\label{fig:efficiency_french}
\end{figure}
\FloatBarrier
\subsubsection{Comparison with Full Attention}
\label{app:full_attention_efficiency}
We compare CGAB with the Full Attention variant on French and Global temperature forecasting, using a 48-hour input window, a 24-hour horizon, and batch size 8.

\setlength{\columnsep}{18pt}
\setlength{\intextsep}{6pt}
\begin{wraptable}[11]{r}{0.50\textwidth}
\centering
\captionsetup{font=small,hypcap=false}
\caption{CGAB versus full attention. Time: s/epoch; memory: GiB. All values are three-run means; lower is better. Best values per dataset are bold.}
\label{tab:full_attention_efficiency}
\fontsize{8}{9.5}\selectfont
\setlength{\tabcolsep}{2pt}
\renewcommand{\arraystretch}{1.10}
\begin{tabular*}{\linewidth}{@{\extracolsep{\fill}}lrrr@{}}
\toprule
Method & Time & Memory & MSE \\
\midrule
\multicolumn{4}{l}{\textit{French temperature}} \\
Full Attention & \textbf{12.478} & \textbf{0.315} & 6.334 \\
CGAB & 15.314 & 0.378 & \textbf{6.291} \\
\midrule
\multicolumn{4}{l}{\textit{Global temperature}} \\
Full Attention & \textbf{63.194} & 10.023 & 6.976 \\
CGAB & 70.427 & \textbf{8.913} & \textbf{6.917} \\
\bottomrule
\end{tabular*}
\end{wraptable}
Table~\ref{tab:full_attention_efficiency} shows that CGAB lowers MSE on both datasets. Its clustering and regional aggregation increase epoch time by 22.7\% on French and 11.4\% on Global, while reducing peak memory on the larger Global network by 11.1\%, from 10.023 to 8.913\,GiB. The modest runtime cost therefore supports more accurate forecasts and, at larger spatial scale, lower memory use, making the additional structured interaction worthwhile.
\par

\subsection{Sensitivity to Spatial Patch Construction}
\label{app:spatial_patch_construction}
We compare four ways to organize stations into spatial patches: random grouping, Hilbert-curve ordering, capacity-constrained geographic clustering, and the serpentine ordering used in the main experiments. This comparison concerns patch construction before embedding, rather than the subsequent dynamic clustering in CGAB.

\begin{wraptable}[16]{r}{0.50\textwidth}
\centering
\captionsetup{font=small,hypcap=false}
\caption{Spatial patch construction (temperature MSE). Values are mean $\pm$ standard deviation over three runs; best means are bold.}
\label{tab:spatial_patch_construction}
\fontsize{8}{9.5}\selectfont
\setlength{\tabcolsep}{1.5pt}
\renewcommand{\arraystretch}{1.16}
\begin{tabular*}{\linewidth}{@{\extracolsep{\fill}}lcc@{}}
\toprule
Construction & French & Hunan \\
\midrule
Random & $6.381\!\pm\!0.043$ & $4.982\!\pm\!0.284$ \\
Hilbert curve & $\mathbf{6.152}\!\pm\!0.018$ & $4.808\!\pm\!0.044$ \\
\shortstack[l]{Capacity-constrained\\geographic clustering} & $6.259\!\pm\!0.140$ & $4.782\!\pm\!0.155$ \\
Serpentine (default) & $6.291\!\pm\!0.090$ & $\mathbf{4.743}\!\pm\!0.149$ \\
\bottomrule
\end{tabular*}
\end{wraptable}
Table~\ref{tab:spatial_patch_construction} shows lower mean MSE for all three geographic constructions than for random grouping on both datasets. Hilbert ordering performs best on French (6.152), whereas the default serpentine ordering performs best on Hunan (4.743). Relative to random grouping, serpentine ordering reduces mean MSE by 1.4\% and 4.8\%, respectively. These results validate geographically informed patch construction: every geographic strategy improves over random grouping, and the default serpentine ordering provides a strong, scalable choice across datasets.
\par

\subsection{Numerical Diagnostics for the Theoretical Analysis}
\label{app:theory_validation}
\captionsetup{hypcap=false}

We examine whether the cluster-size dependence in Theorem~\ref{thm:local-partition-lipschitz} and the assignment behavior encouraged by InfoLoss are visible on French Temperature. Two matched one-layer models differ only in the InfoLoss weight (Full-1L: $0.1$; NoInfo-1L: $0$). Both use a 48-hour input, 24-hour forecast, 6-hour temporal patches, two-station spatial patches, $n=117$ valid station tokens, and $K=10$ clusters. Bound and assignment statistics cover all 1,685 test windows and eight temporal patches per window (13,480 records per model).

\paragraph{Cluster size and the diagnostic bound.}
Figure~\ref{fig:theory_cluster_bound} evaluates the monotone $B(q)/B(n)$ dependence on the largest cluster size $q$; the curves follow the bound formula using constants reconstructed from the reported median and dense diagnostics, and markers show the observed medians. A smaller largest cluster yields a smaller fixed-partition diagnostic, while $q=n$ recovers the dense value. The mean diagnostic ratio is $0.4076$ for Full-1L and $0.5337$ for NoInfo-1L. Using the maximum observed local-attention input norm instantiates the theoretical dependence on the empirical test domain.

\paragraph{Effect of InfoLoss.}
Table~\ref{tab:theory_infoloss} summarizes the matched ablation. With InfoLoss, the mean station--cluster mutual information is higher, while the largest cluster, cluster-size variation, and empty-cluster rate are lower. The mean local-to-dense diagnostic ratio falls by $0.1261$. This is consistent with the theoretical chain from more informative, balanced assignments to smaller clusters and a tighter fixed-partition diagnostic. 

\begin{table}[!htbp]
  \centering
  \captionsetup{font=small}
  \caption{French Temperature one-layer assignment and bound diagnostics, averaged over 13,480 test-window/patch records per model. $\Delta$ is Full-1L minus NoInfo-1L.}
  \label{tab:theory_infoloss}
  \fontsize{8}{9.5}\selectfont
  \setlength{\tabcolsep}{6pt}
  \resizebox{0.82\textwidth}{!}{%
  \begin{tabular}{lcccc}
    \toprule
    Metric & Preferred & Full-1L & NoInfo-1L & $\Delta$ \\
    \midrule
    Mutual information $I$ & $\uparrow$ & \textbf{2.2878} & 2.0283 & $+0.2594$ \\
    Largest-cluster ratio $m(H)/n$ & $\downarrow$ & \textbf{0.1248} & 0.2454 & $-0.1206$ \\
    Cluster-size CV & $\downarrow$ & \textbf{0.1513} & 0.6951 & $-0.5438$ \\
    Empty-cluster rate & $\downarrow$ & \textbf{0.0000} & 0.0225 & $-0.0225$ \\
    Bound ratio $B(m(H))/B(n)$ & $\downarrow$ & \textbf{0.4076} & 0.5337 & $-0.1261$ \\
    \bottomrule
  \end{tabular}}
\end{table}

\paragraph{Fixed-partition perturbations.}
For the theorem-facing residual local-attention operator, we freeze each partition before perturbing its input and disable dropout. On the same 64 test windows (eight patches each), we evaluate 20 unit-Frobenius Gaussian directions at five relative perturbation magnitudes. Figure~\ref{fig:theory_max_amplification} compares only the \emph{largest sampled} amplification for each partition family. The maxima increase from Balanced (2.6266) through Learned (2.7020), size-matched Random (2.9542), and Imbalanced (3.1824) to Dense (3.8256), consistent with the diagnostic ordering at this aggregate level. Balanced gives the smallest amplification and Learned remains close, with both clearly below size-matched Random, Imbalanced, and Dense. The separation between Learned and size-matched Random further shows that the learned partition aligns the restricted interactions with the data. 

\begin{figure}[!htbp]
\centering
\begin{subfigure}[t]{0.49\linewidth}
  \centering
  \includegraphics[width=\linewidth]{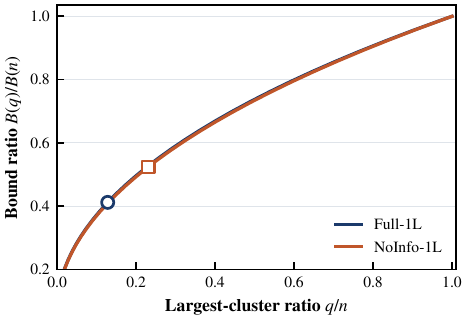}
  \caption{Cluster size and the fixed-partition diagnostic.}
  \label{fig:theory_cluster_bound}
\end{subfigure}\hfill
\begin{subfigure}[t]{0.49\linewidth}
  \centering
  \includegraphics[width=\linewidth]{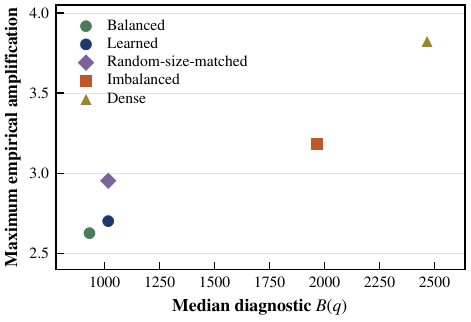}
  \caption{Largest sampled amplification.}
  \label{fig:theory_max_amplification}
\end{subfigure}
\caption{Numerical diagnostics on French temperature. (a) Normalized bound versus largest-cluster ratio; open markers denote test-record medians. (b) Maximum sampled fixed-partition amplification versus median $B(q)$ across five partition families.}
\label{fig:theory_diagnostics}
\end{figure}

\subsection{Prediction Visualizations}
\label{app:prediction_visualizations}
Figures~\ref{fig:prediction_french}--\ref{fig:prediction_global} compare STCFormer with DLinear, PatchTST, Crossformer, TimeXer, DUET, S$^2$Transformer, and EasyST on French temperature, Hunan temperature, and Global wind speed. Each figure shows the same station and forecast window across methods, using 48 hours of observations to predict the following 72 hours. Gray curves denote the observed history, black curves the future ground truth, and red dashed curves the predictions; the shaded region marks the forecast horizon.

For French temperature, STCFormer closely follows the low-temperature troughs and the declining level of successive peaks. On Hunan temperature, it captures both the cooler regime after the observed drop and the subsequent daily fluctuations, whereas several baselines retain a warm bias. On Global wind speed, STCFormer reproduces the sustained increase and elevated values near the end of the horizon more faithfully than the comparatively flat forecasts of DLinear, PatchTST, and TimeXer. Together, these cases visually confirm STCFormer's ability to track regime changes, periodic variation, and evolving trends across distinct station forecasting settings.

\begin{figure}[!ht]
\centering
\includegraphics[width=\textwidth,height=0.86\textheight,keepaspectratio]{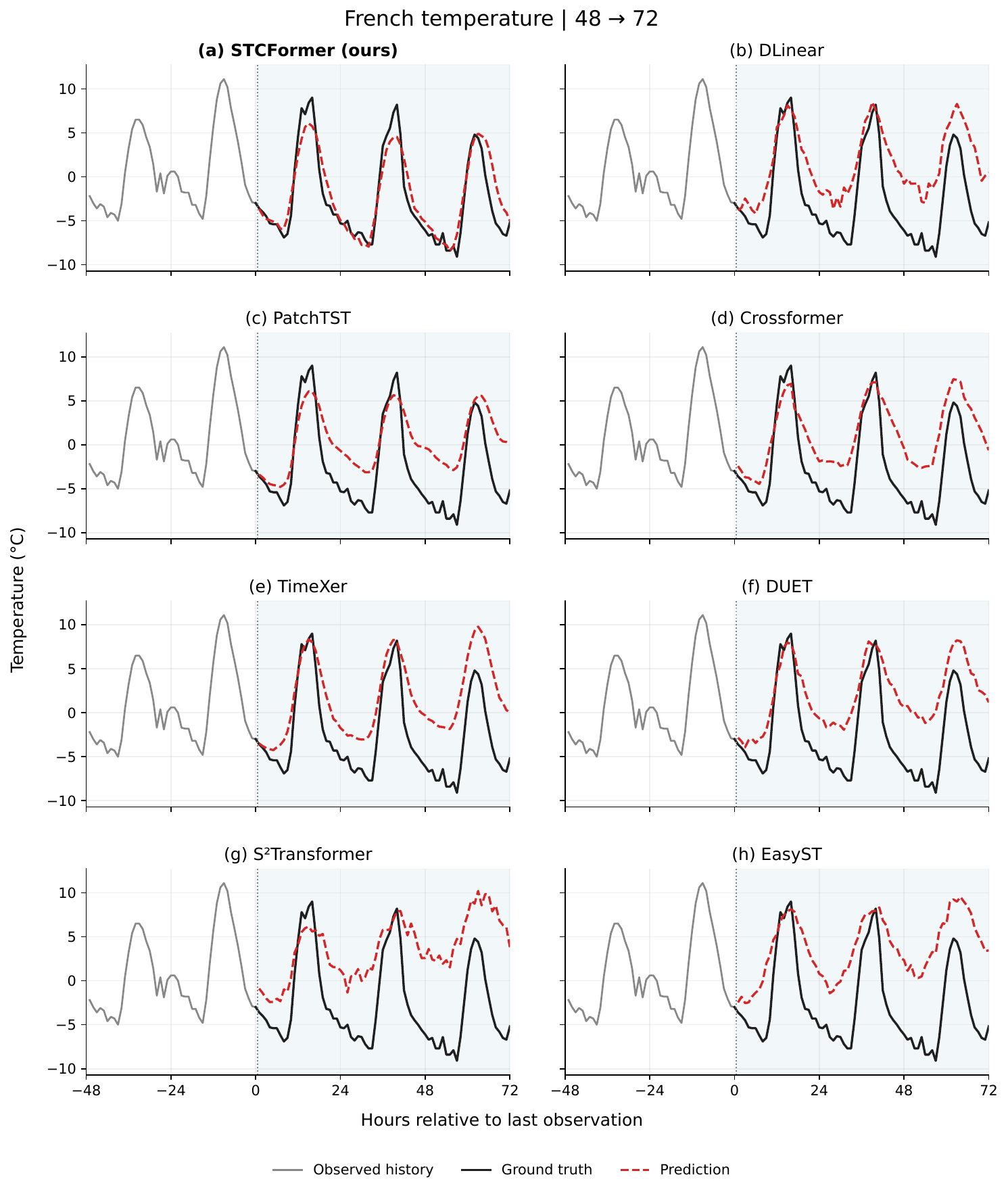}
\caption{Prediction visualization on French temperature with a 48-hour input and a 72-hour forecast horizon. All methods are shown for the same station and window. Gray: observed history; black: ground truth; red dashed: prediction. The shaded region denotes the forecast horizon.}
\label{fig:prediction_french}
\end{figure}

\begin{figure}[p]
\centering
\includegraphics[width=\textwidth,height=0.86\textheight,keepaspectratio]{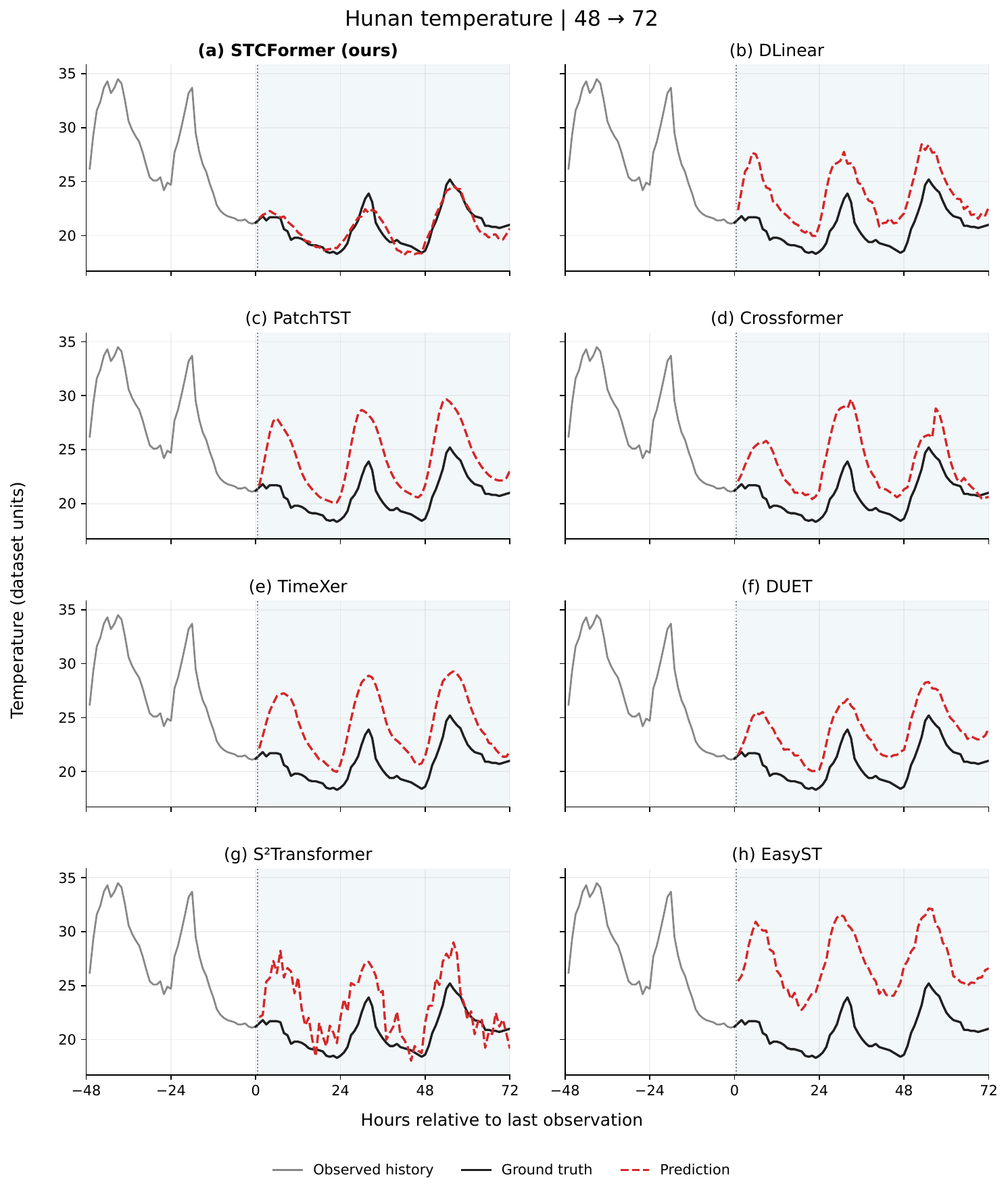}
\caption{Prediction visualization on Hunan temperature with a 48-hour input and a 72-hour forecast horizon. The plotting conventions follow Figure~\ref{fig:prediction_french}. Values are shown in dataset units.}
\label{fig:prediction_hunan}
\end{figure}

\begin{figure}[p]
\centering
\includegraphics[width=\textwidth,height=0.86\textheight,keepaspectratio]{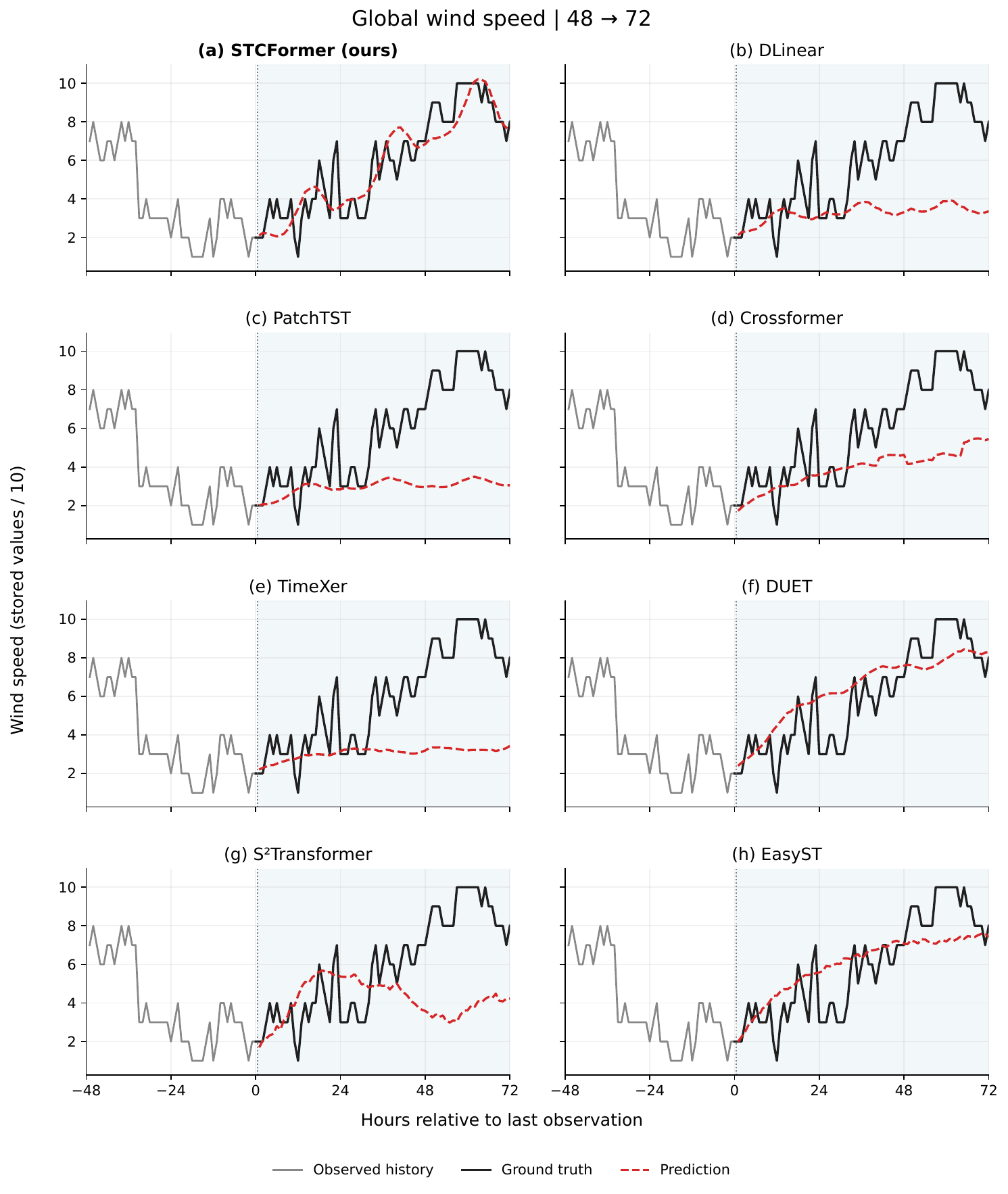}
\caption{Prediction visualization on Global wind speed with a 48-hour input and a 72-hour forecast horizon. The plotting conventions follow Figure~\ref{fig:prediction_french}. Stored values are divided by 10 for display, consistently for observations and predictions.}
\label{fig:prediction_global}
\end{figure}
\FloatBarrier

\end{document}